\documentclass[mnsc,nonblindrev]{informs4Preprint}

\OneAndAHalfSpacedXI

\TheoremsNumberedThrough
\theoremstyle{TH}
\newtheorem{condition}[assumption]{Condition}
\theoremstyle{EX}
\renewenvironment{proof}[1][Proof]{%
  \Trivlist\item[\hspace*{1em}\hskip\labelsep{\itshape #1.\enskip}]\ignorespaces
}{\endTrivlist\addvspace{0pt}}

\IfFileExists{1_intro.tex}{}{}

\usepackage[utf8]{inputenc}
\usepackage[T1]{fontenc}
\usepackage{amsmath,amssymb,amsfonts,mathtools}
\usepackage{algorithm}
\usepackage[noend]{algpseudocode}
\usepackage{array}
\usepackage{booktabs}
\usepackage{enumitem}
\usepackage{etoolbox,needspace,tabularx}
\usepackage{fancyvrb}
\usepackage{graphicx}
\usepackage{makecell}
\usepackage{microtype}
\usepackage{multirow}
\usepackage{nicefrac}
\usepackage{placeins}
\usepackage{xcolor}
\usepackage{hyperref}
\usepackage{url}
\allowdisplaybreaks

\usepackage{natbib}
\bibpunct[, ]{(}{)}{,}{a}{}{,}
\def\bibfont{\small}

\let\argmax\relax
\let\argmin\relax

\usepackage{amsmath,amsfonts,bm}

\def\eqref#1{equation~\ref{#1}}

\def\1{\bm{1}}

\def\vone{{\bm{1}}}

\DeclareMathAlphabet{\mathsfit}{\encodingdefault}{\sfdefault}{m}{sl}
\SetMathAlphabet{\mathsfit}{bold}{\encodingdefault}{\sfdefault}{bx}{n}

\def\sP{{\mathbb{P}}}

\newcommand{\E}{\mathbb{E}}

\newcommand{\R}{\mathbb{R}}

\DeclareMathOperator*{\argmax}{arg\,max}
\DeclareMathOperator*{\argmin}{arg\,min}

\renewcommand{\eqref}[1]{\textup{(\ref{#1})}}

\newcommand{\xhdr}[1]{\vspace{1mm}\noindent{\bfseries #1}}

\DeclareMathOperator{\Reg}{Reg}
\DeclareMathOperator{\Viol}{Viol}
\DeclareMathOperator{\clip}{clip}
\DeclareMathOperator{\supp}{supp}

\newcommand{\wh}{\widehat}
\newcommand{\wt}{\widetilde}

\newcommand{\inner}[2]{\left\langle #1,#2\right\rangle}

\newcommand{\Null}{\mathtt{null}}
\newcommand{\Aset}{\mathcal A}
\newcommand{\Asetzero}{\mathcal A_0}
\newcommand{\VUB}{V^{\mathrm{UB}}}

\newcommand{\DT}{\mathcal D_T}
\newcommand{\Rest}{\mathcal R_{\mathrm{est}}}
\newcommand{\Rshared}{\mathcal R_{\mathrm{sh}}}

\newcommand{\ROPD}{\textsc{ROPD}}
\newcommand{\AnonymousCodeURL}{}

\hypersetup{
  colorlinks=true, linkcolor=blue, citecolor=blue, urlcolor=blue,
  pdftitle={Corruption-Robust Sparse Linear Contextual Bandits with Knapsack Constraints},
  pdfauthor={Yige Wang, Hanyang Li, Yiming Zong, Wanteng Ma, Jiashuo Jiang}
}

\begin{document}

\RUNAUTHOR{Wang, Li, Zong, Ma and Jiang}
\RUNTITLE{Corruption-Robust Sparse Linear Contextual Bandits with Knapsack Constraints}

\TITLE{Corruption-Robust Sparse Linear Contextual Bandits with Knapsack Constraints}

\ARTICLEAUTHORS{%
\AUTHOR{$\text{Yige Wang}^\dag, \text{Hanyang Li}^\dag, \text{Yiming Zong}^\dag$, $\text{Wanteng Ma}^{\ddag*}$, $\text{Jiashuo Jiang}^{\dag*}$}
\AFF{$\dag$~Department of Industrial Engineering and Decision Analytics, The Hong Kong University of Science and Technology\\
$\ddag$~Department of Statistics and Data Science, The Wharton School, University of Pennsylvania}
}

\ABSTRACT{%
We study sparse linear contextual bandits with knapsack constraints under joint reward and consumption corruption. Consumption corruption creates a challenge beyond corrupted rewards: it affects not only statistical estimates, but also the recorded budget, resource prices, and stopping decisions that govern future allocation. We develop Robust Optimistic Primal--Dual (ROPD), an estimator-modular framework that combines corruption-aware confidence widths with online resource prices and a budget-safety rule. With concrete sparse implementation, ROPD achieves regret against a clean population-LP benchmark of $\widetilde O(T^{2/3}+\Gamma T^{1/3})$ under forced exploration and population-design coverage, and $\widetilde O(\sqrt T+\Gamma)$ under on-policy realized-design coverage, for a supplied valid corruption bound $\Gamma$ under the stated proportional-budget scaling and fixed model/design parameters. When the corruption level is unknown, Shared-Grid adapts confidence radii around common point estimates fitted to a single realized history, incurring explicit initialization and master-comparison costs; its sharper on-policy guarantee additionally requires recommendation coverage. Both methods preserve observed budgets on every realization and bound clean resource violation by cumulative consumption corruption. These results connect corruption-robust sparse estimation with resource accounting, pricing, and stopping in high-dimensional online allocation.}


\maketitle
\raggedbottom

\section{Introduction}
\label{sec:introduction}

Contextual bandits with knapsack constraints (CBwK) couple learning with resource allocation: after observing a context, a learner chooses an action whose reward must be learned while its consumption depletes finite budgets \citep{badanidiyuru2018bandits,agrawal2016linear,agrawal2016efficient}. In high-dimensional applications, only a small subset of features may govern reward and consumption, motivating sparse action-specific models \citep{bastani2020online,hao2020high,ma2023high}.

Corruption of consumption feedback creates a difficulty beyond corrupted rewards. A corrupted reward changes what the learner infers about an action, whereas corrupted consumption also changes the state used to allocate future resources---consumption estimate, recorded budget, resource prices, and potentially stopping time. For example, an advertisement that is truly expensive but repeatedly recorded as inexpensive can be favored while the observed budget account still appears feasible.

A robust estimator alone therefore does not resolve the allocation problem. The analysis must control how corrupted observations affect later predictions, relate observed consumption to the latent clean resource use, and account for reward lost when corrupted resource information changes stopping. In the sparse setting, these guarantees also depend on whether the contexts actually collected for each action are informative. When the corruption level is unknown, candidates calibrated to different levels may recommend different actions even though feedback is observed only for the actions that are actually played.

We address these issues with Robust Optimistic Primal--Dual (\ROPD), which combines corruption-aware prediction widths with online resource prices and a budget-safety rule. The allocation analysis is estimator-modular: it uses only pointwise confidence and cumulative uncertainty along the recommended actions. Norm-constrained Lasso provides one concrete sparse implementation. We study both forced exploration and on-policy learning, and for unknown corruption introduce Shared-Grid, which compares confidence calibrations around common point estimates fitted to one realized history.

\begin{table}[t]
\centering
\caption{Guarantee regimes and leading terms when corruption is absent.}
\label{tab:results-overview}
\footnotesize
\setlength{\tabcolsep}{3pt}
\begin{tabularx}{\linewidth}{@{}>{\raggedright\arraybackslash}p{0.18\linewidth}>{\raggedright\arraybackslash}p{0.12\linewidth}>{\raggedright\arraybackslash}X>{\raggedright\arraybackslash}p{0.28\linewidth}@{}}
\toprule
Corruption information & Sampling & Coverage & Corruption-free term \\
\midrule
Bound $\Gamma\ge C_Z$ & Forced & Population design & $\widetilde O(T^{2/3})$ \\
Bound $\Gamma\ge C_Z$ & On-policy & Realized design & $\widetilde O(\sqrt T)$ \\
Not supplied & Forced & Population design & $\widetilde O(T^{2/3})+\Reg_{\mathrm{mst}}(T)$ \\
Not supplied & On-policy & Realized design and recommendation coverage & $\widetilde O(\sqrt T)+\Reg_{\mathrm{mst}}(T)$ \\
\bottomrule
\end{tabularx}
\par\smallskip
\raggedright Orders assume proportional budgets, fixed model and coverage parameters, dominated initialization and burn-in, and negligible failure costs. $\Reg_{\mathrm{mst}}(T)$ is the master-comparison cost; Corollaries~\ref{cor:known-rates-main} and ~\ref{cor:unknown-rates-main} give the full bounds.
\end{table}

\subsection{Main Contributions}
\label{sec:contributions}

\xhdr{1. Joint statistical and resource effects of corruption.}
We measure joint reward--consumption corruption through the priced budget $C_Z$ and separate its two roles: contaminated samples affect future prediction, while corrupted consumption directly changes resource accounting, price updates, and stopping. This decomposition yields regret guarantees together with exact observed-budget feasibility and an explicit bound on clean resource violation.

\xhdr{2. Estimator-modular robust allocation with two data routes.}
Given a valid bound $\Gamma\ge C_Z$, \ROPD{} combines optimistic predictions with online resource prices. Its allocation theorem requires only pointwise confidence and cumulative recommendation-width control; norm-constrained Lasso is one concrete implementation. Forced exploration under population design gives the $\widetilde O(T^{2/3})$ corruption-free regime, while realized on-policy design gives $\widetilde O(\sqrt T)$ under the conditions summarized in Table~\ref{tab:results-overview}.

\xhdr{3. Adaptation to unknown corruption on shared data.}
Shared-Grid compares corruption calibrations around common point estimates fitted to the single realized history, rather than analyzing candidates on counterfactual estimator histories. Its guarantee retains initialization and master-comparison costs explicitly; the sharper on-policy result additionally requires coverage of candidate recommendations by realized plays.

\xhdr{Relation to closest work.}
Clean high-dimensional CBwK already combines sparse estimation with primal--dual resource control \citep{ma2023high}, while corruption-robust linear and contextual bandits primarily address corrupted reward feedback without knapsack state \citep{he2022nearly,liu2024corruption}. Thus the new difficulty here is not sparse estimation or dual pricing in isolation: corrupted consumption simultaneously affects prediction, resource-price path, recorded budget, and stopping. We analyze these effects against a clean population-LP benchmark and, when the corruption level is unknown, adapt on a single realized estimator history. Section~\ref{app:related-work} gives the broader comparison.

\subsection{Related Work}
\label{app:related-work}

\xhdr{Contextual bandits with knapsacks and high-dimensional allocation.} The CBwK framework joins bandit learning with finite resource constraints \citep{badanidiyuru2018bandits}. \citet{agrawal2016linear} study the closest clean linear contextual model: expected rewards and resource consumptions are linear in observed features, and confidence sets are combined with online resource prices to control the knapsacks. \citet{agrawal2016efficient} develop a more general oracle-based contextual framework, and primal--dual ideas also connect this literature to online packing and dynamic resource allocation \citep{agrawal2014bandits,devanur2009adwords,agrawal2014dynamic}. Our known-corruption method keeps the basic confidence-plus-price structure, but the feedback used by the estimator, the dual update, and the stopping rule can now differ from the clean process used by the benchmark. Our analysis accounts for the difference between clean and observed consumption in the price updates and the reward lost after stopping. Adversarial and non-stationary BwK study changing reward and consumption environments \citep{immorlica2022adversarial,liu2022non}, while instance-dependent BwK analyses go beyond worst-case guarantees under clean feedback \citep{sankararaman2021bandits}. Our model instead keeps a latent uncorrupted linear contextual model with i.i.d.\ contexts and studies corrupted observations of that process, so the goal is to recover performance against a clean population-LP benchmark.

\xhdr{Sparse contextual bandits and clean high-dimensional CBwK.} Sparse contextual bandit methods use regularization, forced sampling, Thompson sampling, or hard thresholding so regret depends on a small relevant support \citep{bastani2020online,hao2020high,ren2024dynamic,chakraborty2023thompson}. The closest paper to our sparse allocation setting is \citet{ma2023high}, which combines an online hard-thresholding estimator with a primal--dual CBwK method and obtains logarithmic dependence on the ambient feature dimension under clean feedback. Their sharper clean regime and our zero-corruption on-policy result both have $\widetilde O(\sqrt T)$ horizon dependence, with different action, sparsity, curvature, norm, and resource-price factors. This comparison concerns the horizon exponent, not identical dimension dependence. Our setting adds corrupted reward and consumption observations, a direct resource-accounting term, a sequential charge for the later effect of corrupted action-specific samples, and adaptation when the corruption level is unknown. The retained-design RSC condition in our sharper result plays the familiar statistical role of ensuring that action-specific data identify sparse parameters. The additional Shared-Grid recommendation-coverage condition is different: it controls how often candidate recommendations are supported by realized plays.

\xhdr{Corruption-robust bandits.} Corruption-robust bandit work first showed how regret can degrade with an adversarial corruption budget in stochastic multi-armed bandits \citep{lykouris2018stochastic,gupta2019better}. Later work studies linear and contextual models. \citet{ding2022robust} allow attacked rewards and contexts without requiring the attack budget. For a known played-action reward-corruption budget $C$, \citet{he2022nearly} obtain $\widetilde O(d\sqrt T+dC)$ and also study adaptation without a supplied bound. For fixed-parameter linear bandits, \citet{liu2024corruption} distinguish $C$ from the worst-action budget $C_\infty$ and establish $\widetilde O(d\sqrt T+\min\{dC,\sqrt d\,C_\infty\})$ with matching lower bounds under the corresponding corruption information. These results concern reward corruption without knapsack constraints and do not by themselves establish optimal dependence on our joint reward--consumption budget. Other work studies linear optimization, Lipschitz bandits, generalized linear models, and explicit poisoning attacks \citep{li2019stochastic,kang2023robust,yu2025corruption,liu2019data}. These papers develop robust confidence sets, weighting, elimination, and attack-adaptive exploration for bandit objectives centered on prediction and action selection. Our setting adds a resource signal that changes both statistical estimates and the state of the allocation system. The $Z$-scaled budget measures corruption in the action score, and the proof also controls observed budget account, dual-price path, and stopping time.

\xhdr{Robust sparse estimation and sequential prediction.} Robust sparse regression studies parameter recovery when some responses or samples are corrupted \citep{chen2013robust,bhatia2015robust,balakrishnan2017computationally,diakonikolas2019outlier,prasad2020robust,liu2020high}. Restricted-eigenvalue and sparse covariance conditions explain when high-dimensional parameters can be identified \citep{bickel2009simultaneous,raskutti2010restricted,negahban2012unified,wainwright2019high,buhlmann2011statistics}. These tools provide the statistical base of our estimator module, while an online allocation proof additionally requires the sum of reward and consumption widths along an adaptive action path. We prove this sequential bound by matching estimator sample set to its coverage event and by charging each corrupted observation through the later inverse action counts in which it still appears. Our concrete analysis uses norm-constrained Lasso with candidate-independent regularization. An enlarged-cone argument and an RSC tolerance yield explicit corruption-aware confidence, with a norm bound handling large corruption. The resulting allocation theorems transfer to any sparse estimator that supplies the same predictable pointwise and cumulative-width guarantees.

\xhdr{Unknown corruption and bandit model selection.} Corralling and model-selection methods combine candidates or policy classes under partial feedback \citep{agarwal2017corralling,foster2019model,pacchiano2020regret,pacchiano2022best,cutkosky2021dynamic,krishnamurthy2021optimal,ghosh2021model}. \citet{agarwal2017corralling} show that a master must avoid starving a candidate of feedback. \citet{foster2019model} adapt to an unknown contextual model class, and \citet{pacchiano2022best} combine model selection with guarantees for fixed-mean and adversarial rewards. Model selection has also been used for corruption-robust reinforcement learning \citep{wei2022model}. Unknown corruption in sparse CBwK adds a data-design problem. Two independently updated corruption candidates would choose different actions and would need different action-specific context histories, so the confidence set of an inactive candidate can depend on data that the learner never collected. The shared-estimator grid avoids reliance on these unobserved histories: every candidate uses the same realized point estimate, and the master compares only their current recommendations. The remaining candidate-coverage condition in the exploration-free result states when these recommendations receive enough realized plays for their shared confidence widths to be summed.

Our paper combines primal--dual resource control, corruption-robust learning, and sparse action-specific estimation in a setting where resource-consumption feedback is itself corrupted. This requires a relation between clean and observed resource use, stopping control through resource prices, sequential tracking of corrupted retained samples, and shared-estimator adaptation on one realized design.

\section{Model, Feedback, and Benchmark}
\label{sec:preliminaries}
\label{sec:model}
\label{app:model-details}

\subsection{Interaction and Feedback}
\label{sec:model-benchmark}

The learner interacts with the environment for $T$ rounds with $m$ resource budgets, $B=(B_1,\ldots,B_m)^\top$, with $B_i>0$ for every resource $i\in[m]=\{1,\ldots,m\}$. At round $t$, it observes a \emph{context} $x_t\in\R^d$, a vector of features available before choosing an action. It then selects an action $a_t$ from $K$ alternatives, indexed by $\Aset=[K]=\{1,\ldots,K\}$, or chooses the \emph{null action} $\Null$, which gives zero reward and uses no resources. We write $\Asetzero=\Aset\cup\{\Null\}$ for the full action set and call the actions in $\Aset$ \emph{non-null actions}. The learner receives reward and consumption feedback only for the action it selects; this feedback may be corrupted.

\xhdr{Clean outcomes and observed feedback.}
We distinguish three quantities: the expected outcome before corruption, its noisy realization, and the feedback observed after corruption. We use \emph{clean} to mean before adversarial corruption, not without stochastic noise. For action $a$ and context $x$, let $r^0(a,x)$ be the clean expected reward and $b_i^0(a,x)$ the clean expected consumption of resource $i$. These \emph{clean conditional means} are linear in the context, with unknown coefficient vectors $\mu_a^\star,w_{a,i}^\star\in\R^d$:
\begin{equation}
\label{eq:linear-means}
r^0(a,x)=\inner{\mu_a^\star}{x},
\quad
b_i^0(a,x)=\inner{w_{a,i}^\star}{x},
\quad i\in[m].
\end{equation}
Write $b^0(a,x)=(b_1^0(a,x),\ldots,b_m^0(a,x))^\top$ for the consumption-mean vector.

The clean realizations $r_t^{\mathrm{cl}}(a)$ and $b_{t,i}^{\mathrm{cl}}(a)$ equal these means plus ordinary stochastic noise, specified in Assumption~\ref{ass:base-model} below. The adversary adds raw perturbations $\xi_t^r(a)$ and $\xi_{t,i}^b(a)$ to these realizations. The resulting feedback is clipped to $[0,1]$:
\[
\begin{aligned}
r_t(a)=\clip(r_t^{\mathrm{cl}}(a)+\xi_t^r(a),0,1),\quad
b_{t,i}(a)=\clip(b_{t,i}^{\mathrm{cl}}(a)+\xi_{t,i}^b(a),0,1),
\end{aligned}
\]
where $\clip(z,0,1)=\min\{1,\max\{0,z\}\}$ truncates values outside this interval. The actual changes after clipping, rather than the raw perturbations, enter our corruption measures:
\begin{equation}
\label{eq:clipped-deviations}
c_t^r(a)=r_t(a)-r_t^{\mathrm{cl}}(a),
\quad
c_{t,i}^b(a)=b_{t,i}(a)-b_{t,i}^{\mathrm{cl}}(a).
\end{equation}
The vectors $b_t^{\mathrm{cl}}(a)$, $b_t(a)$, and $c_t^b(a)$ collect the corresponding $m$ resource coordinates. Because both clean and observed outcomes lie in $[0,1]$, $|c_t^r(a)|\le1$ and $\|c_t^b(a)\|_\infty\le1$, where $\|v\|_\infty$ is the largest absolute coordinate of $v$. For the null action, all clean outcomes, observed outcomes, and corruption terms are zero.

These variables describe each action's \emph{potential outcomes}: the outcomes it would produce if selected. They are defined for all actions to permit comparisons, but only $r_t(a_t)$ and $b_t(a_t)$ are observed (Table~\ref{tab:feedback-levels}). The adversary changes clean realizations, not their underlying means, and its raw perturbations have no probability model. Their total effect is measured by the corruption budget in Section~\ref{sec:benchmark-corruption}. Clipping bounds feedback; it does not identify corrupted observations.

\begin{table}[ht]
\centering
\caption{Three levels of reward and resource-consumption feedback.}
\label{tab:feedback-levels}
\begin{tabular}{lll}
\toprule
Quantity & Meaning & Observed by the learner?\\
\midrule
$r^0(a,x), b^0(a,x)$ & clean conditional means & No\\
$r_t^{\mathrm{cl}}(a), b_t^{\mathrm{cl}}(a)$ & clean noisy realizations before corruption & No\\
$r_t(a_t), b_t(a_t)$ & clipped feedback after corruption & Yes, for the chosen action\\
\bottomrule
\end{tabular}
\end{table}

\xhdr{Information and within-round timing.}
For analysis, let $\mathcal F_{t-1}$ be the full history before round $t$: all earlier contexts, actions, clean potential outcomes, corruption, and learner randomization. The learner's own history contains earlier contexts, selected actions, observed feedback, and its past randomization, but not latent clean outcomes or corruption. Its decisions use only this observed history, the current context, and fresh independent randomness.

Before observing $x_t$, the learner applies the budget-safety check in Section~\ref{sec:budget-feasibility}. A round is \emph{active} if this check passes; otherwise, it selects $\Null$ from that round onward. On an active round: 
(i) the learner observes $x_t$; 
(ii) conditional on $\mathcal F_{t-1}$ and $x_t$, the environment generates the clean potential outcomes and fixes the corruption for every action;
(iii) the learner draws $a_t$ using fresh randomness that, conditional on $\mathcal F_{t-1}$ and $x_t$, is independent of these potential outcomes and corruption, and observes only $r_t(a_t)$ and $b_t(a_t)$;
(iv) the learner records the observed resource use and updates its estimates and other algorithm variables.

Thus, corruption is \emph{non-anticipating}: it may depend on $\mathcal F_{t-1}$, the current context, and the generated clean potential outcomes, but not on the learner's fresh round-$t$ random seed. A sample is \emph{retained} when it is used to fit the estimates in Section~\ref{sec:concrete-sparse-estimator}. Both the action and the sample's eligibility for retention are determined before the current feedback is revealed.

\xhdr{Distributional and sparsity assumptions.}
The following conditions use the full pre-round history $\mathcal F_{t-1}$:
\begin{assumption}[Sparse linear contextual model]
\label{ass:base-model}
\leavevmode\par
\begin{enumerate}[label=(\roman*),leftmargin=2em]
\item \textbf{Contexts.} Given $\mathcal F_{t-1}$, $x_t$ has distribution $\mathcal D$ and satisfies $\|x_t\|_2\le1$.
\item \textbf{Clean outcomes.} For every action and resource, $r_t^{\mathrm{cl}}(a)=r^0(a,x_t)+\eta_t^r(a)$ and $b_{t,i}^{\mathrm{cl}}(a)=b_i^0(a,x_t)+\eta_{t,i}^b(a)$, where $\E[\eta_t^r(a)\mid\mathcal F_{t-1},x_t]=0$ and $\E[\eta_{t,i}^b(a)\mid\mathcal F_{t-1},x_t]=0$. The reward noise $\eta_t^r$ and consumption noise $\eta_{t,i}^b$ are conditionally sub-Gaussian with parameters at most $\sigma_r$ and $\sigma_b$, respectively. Clean rewards and every clean consumption coordinate lie in $[0,1]$.
\item \textbf{Sparse coefficients.} For every $a,i$, $\|\mu_a^\star\|_0,\|w_{a,i}^\star\|_0\le s_0$ and $\|\mu_a^\star\|_2,\|w_{a,i}^\star\|_2\le R$. Here $\|v\|_0$ counts the nonzero coordinates of $v$, $s_0$ is the true sparsity level, and $R>0$ bounds the Euclidean norm of each coefficient vector.
\end{enumerate}
\end{assumption}

Clean noise has mean zero given the full history and current context, and the sub-Gaussian parameters control its tail fluctuations. With the timing above, this supports concentration of retained clean-noise sums and unbiased importance-weighted comparisons in the unknown-corruption analysis. Conditioning only on the learner's observations would not suffice for the martingale arguments. The estimator uses a supplied sparsity upper bound $s\ge\max\{1,s_0\}$ for confidence and numerical-precision calibration; it does not constrain the Lasso fit to have at most $s$ nonzero coefficients.

\subsection{Benchmark, Resource Prices, and Effective Corruption}
\label{sec:benchmark-corruption}

\xhdr{Clean benchmark and regret.}
We evaluate the learner using clean rewards rather than corrupted feedback. The benchmark is a \emph{population linear program (LP)} based on the context distribution and clean conditional means. Let $y(a\mid x)$ be the probability of choosing non-null action $a$ given context $x$, with remaining probability assigned to $\Null$. This rule can depend on the context but is the same across rounds. Writing $\E_x$ for expectation over $x\sim\mathcal D$, the LP maximizes expected reward subject to expected resource budgets:
\begin{equation}
\label{eq:population-lp}
\begin{aligned}
\VUB:=\max_y\quad
&T\,\E_x\left[\sum_{a\in\Aset}r^0(a,x)y(a\mid x)\right]\\
\text{s.t.}\quad
&T\,\E_x\left[\sum_{a\in\Aset}b_i^0(a,x)y(a\mid x)\right]\le B_i,
\quad i\in[m],\\
&y(a\mid x)\ge0,
\quad
\sum_{a\in\Aset}y(a\mid x)\le1,
\quad \forall x.
\end{aligned}
\end{equation}

Under the i.i.d.\ model, this value upper-bounds the clean expected reward attainable under the population resource constraints \citep{agrawal2016linear,agrawal2016efficient}. For the analysis, fix a measurable optimizer $y^\star$ before interaction. The learner does not know or solve this LP. Its clean expected regret is
\begin{equation}
\label{eq:regret}
\Reg(T):=\VUB-\E\left[\sum_{t=1}^T r^0(a_t,x_t)\right],
\end{equation}
where the expectation includes contexts, clean noise, learner randomness, and any randomized corruption.

\xhdr{Resource prices.}
To compare reward with resource use, assign nonnegative weights $\lambda\in\Lambda:=\{\lambda\in\R_+^m:\|\lambda\|_1\le1\}$ and choose an overall scale $Z>0$. The weights have total mass at most one, and $Z\lambda_i$ is the price of one unit of resource $i$ in reward units. The \emph{clean Lagrangian score} is expected reward minus priced expected consumption:
\[
F^0(a,x;\lambda)=r^0(a,x)-Z\lambda^\top b^0(a,x).
\]
\begin{assumption}[Lagrangian scale]
\label{ass:Z}
Choose a positive scale $Z>0$ satisfying $Z\ge \VUB/B_{\min}$, where $B_{\min}=\min_iB_i$.
\end{assumption}

The choice $Z=T/B_{\min}$ is computable from the horizon and budgets and satisfies Assumption~\ref{ass:Z} because $\VUB\le T$. This scale ensures that, in the stopping argument, depletion of any resource carries enough dual value to offset the remaining benchmark reward. Under proportional budgets, $B_{\min}=\Theta(T)$ and $B_{\max}:=\max_i B_i=O(T)$, this choice gives $Z=O(1)$. Our formal bounds nevertheless keep the dependence on these quantities explicit.

\xhdr{Effective corruption.}
Reward corruption changes the score directly; consumption corruption changes its resource-cost term. For every $\lambda\in\Lambda$, the resulting score change satisfies
\begin{equation}
\label{eq:lagrangian-corruption-bound}
\left|c_t^r(a)-Z\lambda^\top c_t^b(a)\right|
\le |c_t^r(a)|+Z\|c_t^b(a)\|_\infty.
\end{equation}

We therefore measure both types of corruption in the same units through the \emph{effective corruption budget}
\begin{equation}
\label{eq:effective-corruption}
C_Z:=\sum_{t=1}^T\max_{a\in\Aset}\left\{|c_t^r(a)|+Z\|c_t^b(a)\|_\infty\right\}.
\end{equation}

The maximum covers all non-null actions before the fresh action draw on each round. In the \emph{known-corruption} setting, the learner receives a deterministic bound $\Gamma$ satisfying $C_Z\le\Gamma$ on every admissible realization; in the \emph{unknown-corruption} setting, no such bound is supplied. Knowing $\Gamma$ does not identify corrupted observations. For later analysis, we also measure corruption along the actions actually played:
\[
\begin{aligned}
C_Z(t):=\sum_{\tau=1}^t\bigl(|c_\tau^r(a_\tau)|+Z\|c_\tau^b(a_\tau)\|_\infty\bigr),\quad
C_b(t):=\sum_{\tau=1}^t\|c_\tau^b(a_\tau)\|_\infty.
\end{aligned}
\]
Here $C_Z(t)$ measures both types of corruption in score units, whereas $C_b(t)$ measures only consumption corruption. Both use the selected actions, not the worst actions as in $C_Z$; in particular, $C_Z(t)\le C_Z$.

For estimation, let $\mathcal S_a^{\mathrm{est}}(t)$ contain the earlier rounds $\tau<t$ when action $a$ was selected and its feedback retained for fitting (Section~\ref{sec:concrete-sparse-estimator}). The corresponding reward and resource-$i$ corruption totals are
\[
C_{a,\mathrm{est}}^r(t):=\sum_{\tau\in\mathcal S_a^{\mathrm{est}}(t)}|c_\tau^r(a)|,
\quad
C_{a,i,\mathrm{est}}^b(t):=\sum_{\tau\in\mathcal S_a^{\mathrm{est}}(t)}|c_{\tau,i}^b(a)|.
\]
All these corruption quantities are for analysis: the learner cannot compute them without the clean outcomes.

\subsection{Observed and Clean Budget Feasibility}
\label{sec:corruption-feasibility}
\label{sec:budget-feasibility}

Because the learner records observed consumption, its budget account may differ from clean resource use. Choose an \emph{operating budget} $0<B^{\mathrm{op}}\le B$, with coordinatewise inequalities. This is the resource limit used by the algorithm; a smaller operating budget reserves a buffer against consumption corruption. Let $U_t=\sum_{\tau<t}b_\tau(a_\tau)$ be the observed use before round $t$. A non-null action is allowed only if at least one unit remains in every operating budget, enough to cover any observed consumption on that round:
\begin{equation}
\label{eq:pre-action-safety}
U_{t,i}\le B_i^{\mathrm{op}}-1,
\quad \forall i\in[m].
\end{equation}

Otherwise the learner selects $\Null$ from that round onward. The one-unit safety margin is already included in \eqref{eq:pre-action-safety}, so it is not subtracted again from $B^{\mathrm{op}}$. For $[z]_+:=\max\{z,0\}$, define the largest observed and clean budget excesses relative to the original budget vector $B$:
\[
\begin{aligned}
\Viol_{\mathrm{obs}}(T):=\max_{i\in[m]}\Bigl[\sum_{t=1}^T b_{t,i}(a_t)-B_i\Bigr]_+,\quad
\Viol_{\mathrm{cl}}(T):=\max_{i\in[m]}\Bigl[\sum_{t=1}^T b_{t,i}^{\mathrm{cl}}(a_t)-B_i\Bigr]_+.
\end{aligned}
\]
Recall that $C_b(T)$ is total consumption corruption along the played actions. Below, $\Gamma_b$ is an upper bound on this quantity, when available, and $\vone$ is the all-ones vector. The guarantees hold on every realization.

\begin{proposition}[Three feasibility levels]
\label{prop:feasibility-levels}
Under \eqref{eq:pre-action-safety}, for every resource $i$:
\begin{enumerate}[label=(\roman*),leftmargin=2em,nosep]
\item $\sum_{t=1}^T b_{t,i}(a_t)\le B_i^{\mathrm{op}}$ (exact observed feasibility);
\item $\sum_{t=1}^T b_{t,i}^{\mathrm{cl}}(a_t)\le B_i^{\mathrm{op}}+C_b(T)$ (clean sample-path control);
\item if $C_b(T)\le\Gamma_b$, $B_i>\Gamma_b$ for every $i$, and $B^{\mathrm{op}}=B-\Gamma_b\vone$, then $\sum_{t=1}^T b_{t,i}^{\mathrm{cl}}(a_t)\le B_i$ (strict clean sample-path feasibility).
\end{enumerate}
\end{proposition}
\begin{proof}[Proof]
Before a non-null action, the rule in \eqref{eq:pre-action-safety} and $b_{t,i}(a_t)\le1$ imply that the updated observed use is at most $B_i^{\mathrm{op}}$. Null actions use no resources, which proves (i). Since $b_{t,i}^{\mathrm{cl}}(a_t)=b_{t,i}(a_t)-c_{t,i}^b(a_t)$,
\[
\sum_{t=1}^T b_{t,i}^{\mathrm{cl}}(a_t)
\le \sum_{t=1}^T b_{t,i}(a_t) +\sum_{t=1}^T|c_{t,i}^b(a_t)|
\le B_i^{\mathrm{op}}+C_b(T),
\]
which proves (ii). Substituting $B_i^{\mathrm{op}}=B_i-\Gamma_b$ proves (iii).
\end{proof}

For the default choice $B^{\mathrm{op}}=B$, Proposition~\ref{prop:feasibility-levels} gives
\[
\Viol_{\mathrm{obs}}(T)=0,
\quad
\Viol_{\mathrm{cl}}(T)\le C_b(T)\le C_Z/Z.
\]

The last inequality follows because the per-round maximum defining $C_Z$ dominates $Z\|c_t^b(a_t)\|_\infty$. Hence $C_Z\le\Gamma$ also certifies $\Viol_{\mathrm{cl}}(T)\le\Gamma/Z$. A corruption buffer protects clean feasibility but leaves fewer resources for reward. Let $V^{\mathrm{UB}}(B')$ denote the population-LP value with budget vector $B'$. To apply the regret guarantee with a reduced operating budget, first use that budget's benchmark and a scale $Z\ge V^{\mathrm{UB}}(B^{\mathrm{op}})/B_{\min}^{\mathrm{op}}$, where $B_{\min}^{\mathrm{op}}:=\min_i B_i^{\mathrm{op}}$; the computable choice $Z=T/B_{\min}^{\mathrm{op}}$ is sufficient. All quantities that depend on $Z$, including the corruption allowance and confidence radii, use this same scale. Comparison with the original budget then adds $V^{\mathrm{UB}}(B)-V^{\mathrm{UB}}(B^{\mathrm{op}})$; for the buffered choice, $B^{\mathrm{op}}=B-\Gamma_b\vone$.

\begin{remark}[Conditional-mean feasibility]
Proposition~\ref{prop:feasibility-levels} controls clean realized consumptions $b_{t,i}^{\mathrm{cl}}(a_t)$, not their conditional means. To bound $\sum_t b_i^0(a_t,x_t)$, one must also account for clean stochastic noise. Under Assumption~\ref{ass:base-model}, a martingale concentration argument gives an additional buffer of order $\sqrt{T\log(m/\delta)}$, where $\delta\in(0,1)$ is the failure probability. Adding this concentration buffer gives the corresponding conditional-mean feasibility statement.
\end{remark}

\subsection{Notation Summary}
Table~\ref{tab:model-notation} collects the main notation; its final three rows preview quantities defined in Sections~\ref{sec:overview}--\ref{sec:unknown}.
\begin{table}[H]
\centering
\caption{Model notation and selected quantities introduced later.}
\label{tab:model-notation}
\label{tab:main-notation}
\small
\begin{tabularx}{\linewidth}{@{}lX@{}}
\toprule
Symbol & Meaning\\
\midrule
$T,K,d,m$ & Number of rounds, non-null actions, context coordinates, and resources.\\
$x_t,a_t$; $B$ & Observed context and selected action at round $t$; resource-budget vector.\\
$r^0,b^0$ & Clean conditional reward and consumption means used by the benchmark.\\
$r^{\mathrm{cl}},b^{\mathrm{cl}}$ & Clean noisy realizations before corruption.\\
$r,b$ & Clipped corrupted feedback observed for the played action.\\
$c^r,c^b$ & Changes from clean to observed outcomes, measured after clipping.\\
$\VUB,\Reg(T)$ & Clean population-LP value and clean expected regret.\\
$Z,\lambda,\Lambda$ & Price scale, normalized resource-price vector, and its admissible set.\\
$C_Z,\Gamma$ & Global effective corruption with a worst-action maximum on each round; supplied upper bound in the known-corruption setting.\\
$C_Z(t)$ & Effective corruption accumulated along the selected actions through round $t$.\\
$C_b(t)$ & Consumption corruption accumulated along the selected actions through round $t$.\\
$B^{\mathrm{op}},U_t$ & Operating budget and observed resource use before round $t$.\\
$R$; $s_0,s$ & Coefficient-norm bound; true sparsity and supplied calibration upper bound.\\
\midrule
\multicolumn{2}{@{}l}{\emph{Selected notation introduced in the estimation and algorithm sections}}\\
$\kappa_{\mathrm{FE}},\kappa_{\mathrm{op}}$ & Curvature bounds for forced-exploration and on-policy designs (Appendix~\ref{app:concrete-design-conditions}).\\
$A_{\mathrm{FE}},A_{\mathrm{op}}$; $n_0$ & Coefficients in corruption-free confidence widths; minimum retained-sample count per action (Section~\ref{sec:estimator-contract}).\\
$\chi,n_{\mathrm{cand}}$; $t_{\mathrm{init}}$ & Constants comparing recommendation counts with actual plays; safe initialization length (Section~\ref{sec:unknown}).\\
\bottomrule
\end{tabularx}
\end{table}

\section{Estimation and Optimistic Primal--Dual Learning}
\label{sec:overview}

The method has two components: an estimator predicts reward and resource consumption, and an allocation rule uses these predictions and their uncertainty to choose actions and update resource prices. We first describe a norm-constrained Lasso estimator, then state the confidence and coverage requirements that connect estimation to regret, and finally give the allocation and price updates. The allocation guarantees depend on these requirements, not on Lasso itself. Calibration, proofs, and certified optimization details are in Appendix~\ref{app:prediction-guarantees}.

\subsection{Sparse Estimation and Prediction Confidence}
\label{sec:concrete-sparse-estimator}

For each non-null action, we fit one reward model and one model for each resource. A \emph{retained observation} is an observation used in these fits. Forced exploration (FE) retains observations from context-independent uniform exploration; Shared-Grid also retains its fixed initialization samples. On-policy learning (OP) instead retains every non-null observation. All models for the same action use the same retained contexts.

Fix an action $a$ and one response, called a \emph{channel}: reward with parameter $\theta^\star=\mu_a^\star$, or consumption of resource $i$ with $\theta^\star=w_{a,i}^\star$. Let $\mathcal S_a^{\mathrm{est}}(t)=\{\tau<t:a_\tau=a\text{ and the sample is retained}\}$ be the retained round indices before $t$, and let $n=N_a^{\mathrm{est}}(t):=|\mathcal S_a^{\mathrm{est}}(t)|$. The matrix $X\in\R^{n\times d}$ has the retained contexts as rows, and $y\in\R^n$ contains the corresponding observed responses.

For analysis, write $y=X\theta^\star+\eta+\zeta$, where $\eta$ is clean stochastic noise and $\zeta$ collects the post-clipping corruptions. The learner observes $X$ and $y$, not the separate noise and corruption terms. Whether to retain a sample is decided before its feedback is revealed.

For regularization strength $\rho_n>0$, we use the norm-constrained Lasso fit \citep{tibshirani1996lasso}:
\begin{equation}
\label{eq:lasso-estimator}
\widehat\theta\in\argmin_{\|\theta\|_2\le R}
\left\{\frac{\|y-X\theta\|_2^2}{2n}+\rho_n\|\theta\|_1\right\},\quad n\ge1.
\end{equation}
The $\ell_1$ penalty encourages sparse estimates, while the $\ell_2$ constraint enforces the known coefficient bound $R$. Neither fixes the support of the estimate. Set $\widehat\theta=0$ when $n=0$.

The prediction guarantee uses a supplied sparsity bound $s\ge\max\{1,s_0\}$, a retained-sample threshold $n_0$, and an estimation failure level $\delta_{\mathrm{est}}\in(0,1)$. The bound $s$ calibrates confidence and numerical precision; it is not a support constraint in the fit. The retained design is the matrix of contexts used by that fit. Its restricted strong convexity (RSC) condition requires those contexts to distinguish changes in the coefficient vector, with an $\ell_1$ tolerance for high-dimensional data. The constant $\kappa>0$ measures the curvature in this condition; the precise inequality is \eqref{eq:empirical-rsc}.

Appendix~\ref{app:concrete-estimator} specifies $\rho_n,\kappa,n_0$, and Appendix~\ref{app:tractable-options} specifies how accurately to solve the objective. For a fixed retained history, the fitted coefficients and regularization do not depend on a proposed corruption bound $\Gamma$. Such a bound changes the confidence radii, not the fitting objective.

\begin{proposition}[Prediction confidence for the concrete estimator]
\label{prop:lasso-prediction-main}
Under Assumption~\ref{ass:base-model} and the calibration specified in Appendix~\ref{app:concrete-estimator}, there is a noise event of probability at least $1-\delta_{\mathrm{est}}$, simultaneous over actions, response channels, and obtained sample prefixes. On this event, whenever a retained prefix $n\ge n_0$ satisfies the calibrated RSC condition, any feasible fit meeting the certified optimization tolerance of Appendix~\ref{app:tractable-options} satisfies
\begin{equation}
\label{eq:lasso-prediction-main}
\sup_{\|x\|_2\le1}|\langle\widehat\theta-\theta^\star,x\rangle| \le
\min\left\{2R,\,\frac{7\rho_n\sqrt{s}}{\kappa}+\frac{4C_n}{\kappa n}\right\},
\quad C_n:=\|\zeta\|_1.
\end{equation}
For an exact minimizer, the coefficient $7$ can be replaced by $6$.
\end{proposition}

Here $C_n$ measures corruption only in the retained responses for this action and channel. The term $7\rho_n\sqrt{s}/\kappa$ is the clean estimation contribution: $\rho_n$ decreases as $n^{-1/2}$ under the stated calibration, $\sqrt{s}$ reflects sparsity, and $1/\kappa$ reflects design conditioning. The additional term $4C_n/(\kappa n)$ measures the effect of retained corruption. The cap $2R$ follows from the norm constraint, and the coefficient $7$, rather than $6$, allows the certified numerical error. Appendix~\ref{app:robust-sparse-estimators} gives the proof, including the large-corruption case.

\subsection{Coverage and the Estimator Interface}
\label{sec:estimator-contract}

The allocation rule needs predictions and confidence radii from estimator. Write $\widehat\mu_{a,t-1}$ and $\widehat w_{a,i,t-1}$ for the reward and resource fits using observations before round $t$. At the current context, their clipped predictions are $\widehat r_{a,t}=\clip(\langle\widehat\mu_{a,t-1},x_t\rangle,0,1)$ and $\widehat b_{a,i,t}=\clip(\langle\widehat w_{a,i,t-1},x_t\rangle,0,1)$. The interface imposes two requirements: the radii must cover prediction errors, and the uncertainty charged over the run must be controlled.

\xhdr{Pointwise prediction confidence.}
A corruption candidate $\Gamma$ is a proposed upper bound on $C_Z$; it is valid when $\Gamma\ge C_Z$. Let $\beta_{a,t}^{r,\Gamma}$ and $\beta_{a,i,t}^{b,\Gamma}$ be nonnegative error radii for the reward and resource-$i$ predictions. These radii bound errors in \emph{clean conditional means}, not errors relative to the noisy observed feedback. Predictions and radii use the observed history and $x_t$ and are fixed before the fresh action draw. On a common retained history, candidates share point estimates and differ only in their radii.

Let $\delta_{\mathrm{cov}}\in(0,1)$ be the failure level for the design/coverage event. On one joint event of failure at most $\delta_{\mathrm{est}}+\delta_{\mathrm{cov}}$, require simultaneously for every valid $\Gamma\ge C_Z$,
\begin{equation}
\label{eq:general-pointwise-confidence}
|\widehat r_{a,t}-r^0(a,x_t)|\le\beta_{a,t}^{r,\Gamma},\quad
|\widehat b_{a,i,t}-b_i^0(a,x_t)|\le\beta_{a,i,t}^{b,\Gamma}
\quad(a\in\Aset,\ i\in[m],\ t\in[T]).
\end{equation}
For Lasso, Proposition~\ref{prop:lasso-prediction-main} and $C_Z\le\Gamma$ give the radii in Appendix~\ref{app:concrete-radii}; use unit radii before $n_0$ and zero predictions and radii for $\Null$. The allocation rule compares reward minus resource cost. Its combined radius therefore adds reward uncertainty to resource uncertainty in the same units. Since $\|\lambda\|_1\le1$, the largest resource radius suffices. Write $h$ for the length of the clean-score interval $[-Z,1]$:
\begin{equation}
\label{eq:score-radius-interface}
\beta_{a,t}^{Z,\Gamma}:=\beta_{a,t}^{r,\Gamma}+Z\max_i\beta_{a,i,t}^{b,\Gamma},
\quad h:=1+Z.
\end{equation}
\xhdr{Cumulative uncertainty.}
Pointwise confidence alone is not enough for a regret bound: the widths charged over the run must also remain small. Let $\widehat a_t$ denote the optimistic recommendation defined in Section~\ref{sec:core-primal-dual-rule}, with the fixed corruption input suppressed in this notation. Let $I_t\in\{0,1\}$ indicate whether the pre-action safety test \eqref{eq:pre-action-safety} passes; it remains one on an active round even if $\Null$ is subsequently chosen. Require on the same event
\begin{equation}
\label{eq:cumulative-width-interface-main}
2\sum_{t=1}^T I_t\beta_{\widehat a_t,t}^{Z,\Gamma}
\le\Rest^0+\psi_{\mathrm{est}}(\Gamma).
\end{equation}

Here $\Rest^0$ bounds cumulative uncertainty from clean estimation, and $\psi_{\mathrm{est}}(\Gamma)$ bounds the additional width caused by retained corruption. Both are analysis quantities, not algorithm inputs. The optimistic comparison charges uncertainty to the recommended action, so no separate cumulative-width bound is needed along the population-LP policy. Shared-Grid needs the corresponding bound uniformly over valid candidates on one shared history; see \eqref{eq:uniform-grid-width-interface}. Any estimator meeting these confidence and cumulative-width requirements can replace Lasso without changing the allocation analysis.

\xhdr{Design conditions for the concrete rates.}
The conditions below establish the interface for Lasso; the general allocation theorem does not impose them on another estimator that already meets the interface. FE selects its retained actions independently of the current context on active rounds, allowing a population condition to control the retained designs. OP uses adaptively selected observations and instead requires a condition on its actual retained designs. Restricted-curvature conditions are standard in sparse estimation \citep{negahban2012unified,raskutti2010restricted}, but adaptive action selection must also be accounted for \citep{bastani2020online}. Appendix~\ref{app:concrete-design-conditions} gives the calibrations and proofs.

\begin{assumption}[Population design for the concrete estimator]
\label{ass:sparse-context}
For a generic context $x\sim\mathcal D$, known constants $\kappa_x,L_x>0$ satisfy
\begin{equation}
\label{eq:population-rsc-calibration}
\E[xx^\top]\succeq\kappa_x I_d,\quad
\E\exp\left(\frac{\langle u,x\rangle^2}{L_x^2\|u\|_2^2}\right)\le2
\quad (u\ne0).
\end{equation}
\end{assumption}
Here $\kappa_x$ lower-bounds the eigenvalues of the population second-moment matrix $\E[xx^\top]$, while $L_x$ controls the directional sub-Gaussian scale.

For OP, the relevant information is in the contexts actually retained for each action, rather than in the population distribution alone.

\begin{condition}[On-policy retained-design coverage]
\label{ass:on-policy-re}
For supplied design parameters $\kappa_{\mathrm{op}},\upsilon_{\mathrm{op}}>0$ and threshold $n_0$, the action-specific realized retained designs satisfy the calibrated RSC event defined in Appendix~\ref{app:concrete-design-conditions}, simultaneously after $n_0$ samples per action, with probability at least $1-\delta_{\mathrm{cov}}$. Here $\kappa_{\mathrm{op}}$ is the RSC curvature and $\upsilon_{\mathrm{op}}$ controls its $\ell_1$ tolerance.
\end{condition}

This condition concerns the algorithm's actual adaptive, corrupted, and potentially stopped trajectory; population richness alone does not imply it. Shared-Grid OP also needs the separate recommendation-coverage condition in Section~\ref{sec:unknown}, because a candidate may recommend an action that is not played.

\xhdr{Rate constants.}
The concrete regret bounds use $\ell:=\log\bigl(16eK(m+1)dT/\min\{\delta_{\mathrm{est}},\delta_{\mathrm{cov}}\}\bigr)$ for concentration and $H_T:=1+\log(1+T)$ for the harmonic factor arising when widths are summed. The retained-design curvatures are $\kappa_{\mathrm{FE}}:=\kappa_x/2$ for FE and $\kappa_{\mathrm{op}}$ from Condition~\ref{ass:on-policy-re} for OP. The coefficients $A_{\mathrm{FE}}$ and $A_{\mathrm{op}}$ multiply the corresponding corruption-free $n^{-1/2}$ score widths. They include model and design parameters, not just universal constants; Appendix~\ref{app:concrete-design-conditions} gives their full dependence.

\subsection{Optimistic Primal--Dual Rule}
\label{sec:core-primal-dual-rule}

The \emph{primal} step recommends an action using estimated reward minus resource cost, with an uncertainty bonus. The \emph{dual} step adjusts resource prices using the consumption actually observed.

At prices $\lambda_t\in\Lambda$, substitute the fitted reward and consumption means into the clean score $F^0(a,x_t;\lambda_t)$. This gives the estimated score $\widehat F_{a,t}$; candidate $\Gamma$ recommends the action with the largest estimate plus its confidence radius:
\begingroup
\interdisplaylinepenalty=10000
\begin{align}
\label{eq:estimated-lagrangian-score}
\widehat F_{a,t}(\lambda_t)
&:=
\widehat r_{a,t}-Z\sum_{i=1}^m\lambda_{t,i}\widehat b_{a,i,t},\\
\label{eq:optimistic-recommendation}
\widehat a_t^\Gamma
&\in
\argmax_{a\in\Asetzero}\{\widehat F_{a,t}(\lambda_t)+\beta_{a,t}^{Z,\Gamma}\},
\end{align}
\endgroup
with zero score and radius for $\Null$ and the fixed order $\Null,1,\ldots,K$ for ties. $\widehat a_t^\Gamma$ is the optimistic recommendation for candidate $\Gamma$, computed on rounds that pass the pre-action safety test \eqref{eq:pre-action-safety}. The actually played action $a_t$ may differ because of forced exploration or the Shared-Grid master; if the safety test fails, the learner plays $\Null$ from that round onward.

For the price update, represent $\lambda_t$ by the first $m$ coordinates of a probability vector $q_t\in\R_+^{m+1}$: $\lambda_t=q_{t,1:m}$ and $q_{t,m+1}=1-\|\lambda_t\|_1$. The last coordinate holds unused price mass; it is not another physical resource. Compare observed consumption with the per-round operating-budget target $B^{\mathrm{op}}/T$. For stepsize $\eta_q>0$, set $\widetilde g_t=(b_t(a_t)-B^{\mathrm{op}}/T,0)$, with zero in the slack coordinate, and update
\begin{equation}
\label{eq:dual-mw-update}
q_{t+1,j}\propto q_{t,j}e^{\eta_q\widetilde g_{t,j}},\quad \lambda_{t+1}=q_{t+1,1:m}.
\end{equation}
Normalize $q_{t+1}$ to sum to one. The update increases the relative weight of resources with larger excess use. We denote its cumulative price-learning penalty by
\begin{equation}
\label{eq:dual-control-term}
\DT
:=
Z\left[\frac{\log(m+1)}{\eta_q}+\eta_q\left(1+\frac{B_{\max}}{T}\right)^2T\right].
\end{equation}
Choosing $\eta_q \asymp \frac{\sqrt{\log(m+1)/T}}{1+B_{\max}/T}$ gives $\DT = O\left(Z\left(1+\frac{B_{\max}}{T}\right)\sqrt{T\log(m+1)}\right)$.

\section{Robust Learning with Known Corruption}
\label{sec:general-known}
\label{app:known-details}

\subsection{Algorithm and Inputs}

We now assume that the learner receives a deterministic bound $\Gamma\ge C_Z$ that is valid on every admissible realization. On each active round, \ROPD{} forms an optimistic recommendation, selects an action, and updates its estimates, resource prices, and observed-budget account from the played action's feedback. Forced exploration may replace the recommendation by a uniform non-null action; on-policy learning plays the recommendation.

Using the fitted score and radius from Section~\ref{sec:overview}, write the optimistic score as
\begin{equation}
\label{eq:known-optimistic-score}
U_{a,t}^{\Gamma}:=\widehat F_{a,t}(\lambda_t)+\beta_{a,t}^{Z,\Gamma},
\end{equation}
with $U_{\Null,t}^{\Gamma}=0$. Here $U_{a,t}^{\Gamma}$ is a scalar optimistic score for action $a$, distinct from the cumulative observed-consumption vector $U_t$.

In FE mode, $\epsilon_t$ is the probability of a uniform forced action at round $t$, and only forced observations are retained for estimation. OP retains every non-null observation and sets $\epsilon_t=0$. For a fixed retained history, $\Gamma$ changes confidence radii but not the fitting objective; it can still affect future data through action selection.

Table~\ref{tab:known-inputs} lists the inputs, and Algorithm~\ref{alg:known} gives the complete procedure. In the algorithm, $\mathcal H_{a,t}^{\mathrm{est}}$ denotes action $a$'s retained data through round $t$, so $|\mathcal H_{a,t-1}^{\mathrm{est}}|=N_a^{\mathrm{est}}(t)$. The concrete fits are specified in Section~\ref{sec:concrete-sparse-estimator} and Appendix~\ref{app:concrete-estimator}. Another estimator meeting the interface can replace these fits without changing the allocation, price, or safety updates.

\begin{table}[H]
\centering
\caption{Inputs to the complete known-corruption method.}
\label{tab:known-inputs}
\small
\begin{tabularx}{\linewidth}{@{}lX@{}}
\toprule
Input & Definition and role\\
\midrule
$B^{\mathrm{op}}$ & Operating budget used by the observed-consumption safety account; the default is $B$.\\
$\Gamma$ & Deterministic supplied bound satisfying $C_Z\le\Gamma$ on every admissible realization.\\
$Z$ & For $B^{\mathrm{op}}=B$, a positive scale satisfying Assumption~\ref{ass:Z}; $Z=T/B_{\min}$ suffices. The reduced-budget calibration is described below.\\
$\eta_q$ & Multiplicative-weights stepsize for the resource-price distribution.\\
$\epsilon_t$ & Probability of a uniform non-null forced-exploration override; it is zero in on-policy mode.\\
$n_0$ & Action-specific retained-sample threshold before nontrivial sparse confidence is used.\\
$(\kappa,\upsilon)$ & Supplied RSC calibration justified by \eqref{eq:forced-rsc-calibration} or Condition~\ref{ass:on-policy-re}.\\
$s,R,\sigma_r,\sigma_b$ & Supplied sparsity, parameter-norm, and noise bounds for Lasso confidence.\\
$\delta_{\mathrm{est}},\delta_{\mathrm{cov}}$ & Failure levels used in $\ell$ and design calibration.\\
\bottomrule
\end{tabularx}
\end{table}

\begin{algorithm}[t]
\caption{Robust Optimistic Primal--Dual Learning with Known Corruption (\ROPD)}
\label{alg:known}
\label{alg:known-main}
\begingroup
\SingleSpacedXI
\small
\algrenewcommand\algorithmicindent{1.15em}
\begin{algorithmic}[1]
\Require Horizon $T$; actions $\Aset=[K]$ and $\Asetzero=\Aset\cup\{\Null\}$; budgets $0<B^{\mathrm{op}}\le B$; corruption bound $\Gamma\ge C_Z$; scale $Z$; dual stepsize $\eta_q$.
\Statex \textbf{Coverage inputs:} $\mathsf{mode}\in\{\mathsf{forced},\mathsf{on\mbox{-}policy}\}$ and exploration probabilities $\epsilon_t\in[0,1]$ ($\epsilon_t=0$ in on-policy mode).
\Statex \textbf{Estimation module:} predictions and radii satisfying \eqref{eq:general-pointwise-confidence} and \eqref{eq:cumulative-width-interface-main}, with the module's required inputs.
\Statex \textbf{Estimator inputs:} $s,R,\sigma_r,\sigma_b$, calibrated $(\kappa,\upsilon,n_0)$, failure levels $\delta_{\mathrm{est}},\delta_{\mathrm{cov}}$, and precision $s\rho_n^2/(8\kappa)$.
\State Set $\mathcal H_{a,0}^{\mathrm{est}}=\varnothing$ for all $a\in\Aset$, $U_1=0$, $q_{1,j}=1/(m+1)$ for $j\in[m+1]$, and $\lambda_1=q_{1,1:m}$.
\For{$t=1,\ldots,T$}
  \If{$U_{t,i}>B_i^{\mathrm{op}}-1$ for some $i\in[m]$}
    \State Play $a_t=\Null$; keep $(U_{t+1},q_{t+1},\lambda_{t+1})=(U_t,q_t,\lambda_t)$ and $\mathcal H_{a,t}^{\mathrm{est}}=\mathcal H_{a,t-1}^{\mathrm{est}}$ for all $a$; \textbf{continue}.
  \EndIf
  \State Observe $x_t$.
  \For{$a\in\Aset$}
    \State Set $n=|\mathcal H_{a,t-1}^{\mathrm{est}}|$. If $n=0$ or no justified calibration is supplied, use zero coefficients; otherwise obtain the channel fits with Appendix~\ref{app:tractable-options}'s certified solver, reusing unchanged fits.
    \State Use the supplied calibration and certified Lasso precision. Use unit channel radii if $n<n_0$ or calibration is unavailable; otherwise use \eqref{eq:concrete-reward-radius}--\eqref{eq:concrete-consumption-radius}.
    \State Compute $\beta_{a,t}^{Z,\Gamma}$ and $U_{a,t}^{\Gamma}$ from \eqref{eq:known-optimistic-score}.
  \EndFor
  \State Set $U_{\Null,t}^{\Gamma}=0$ and choose $\wh a_t\in\argmax_{a\in\Asetzero}U_{a,t}^{\Gamma}$ using the fixed tie-breaking rule.
  \State Draw $e_t\sim\operatorname{Bernoulli}(\epsilon_t)$. If $e_t=1$, draw $a_t$ uniformly from $\Aset$; otherwise set $a_t=\wh a_t$.
  \If{$a_t\ne\Null$}
    \State Observe $r_t(a_t)$ and $b_t(a_t)$.
  \Else
    \State Set $r_t(\Null)=0$ and $b_t(\Null)=0$.
  \EndIf
  \State Set $\mathcal H_{a,t}^{\mathrm{est}}=\mathcal H_{a,t-1}^{\mathrm{est}}$ for every $a\in\Aset$.
  \If{$a_t\ne\Null$ and $[\mathsf{mode}=\mathsf{on\mbox{-}policy}\text{ or }e_t=1]$}
    \State Append $(x_t,r_t(a_t),b_t(a_t))$ to $\mathcal H_{a_t,t}^{\mathrm{est}}$.
  \EndIf
  \State Set $g_t=b_t(a_t)-B^{\mathrm{op}}/T$ and $\wt g_t=(g_t,0)\in\R^{m+1}$.
  \State Set $\displaystyle q_{t+1,j}=\frac{q_{t,j}\exp(\eta_q\wt g_{t,j})}{\sum_{\ell=1}^{m+1}q_{t,\ell}\exp(\eta_q\wt g_{t,\ell})}, \forall j\in[m+1]$; then set $\lambda_{t+1}=q_{t+1,1:m}$ and $U_{t+1}=U_t+b_t(a_t)$.
\EndFor
\end{algorithmic}
\endgroup
\end{algorithm}

The pre-action check leaves room for one more observed consumption, since $b_{t,i}(a_t)\le1$. It therefore gives exact observed feasibility, $\sum_{t=1}^T b_{t,i}(a_t)\le B_i^{\mathrm{op}}$. If a valid bound $C_b(T)\le\Gamma_b$ is known and $B_i>\Gamma_b$ for every resource, using $B^{\mathrm{op}}=B-\Gamma_b\vone$ also gives strict clean sample-path feasibility by Proposition~\ref{prop:feasibility-levels}.

The regret guarantees below use the default budget $B^{\mathrm{op}}=B$. For a reduced operating budget, first apply the analysis to the reduced-budget benchmark, with $Z\ge V^{\mathrm{UB}}(B^{\mathrm{op}})/B_{\min}^{\mathrm{op}}$, where $B_{\min}^{\mathrm{op}}:=\min_i B_i^{\mathrm{op}}$; $Z=T/B_{\min}^{\mathrm{op}}$ is sufficient. Evaluate all scale-dependent quantities, including $C_Z$ and its supplied bound $\Gamma$, with this same scale. Returning to the original benchmark then adds $V^{\mathrm{UB}}(B)-V^{\mathrm{UB}}(B^{\mathrm{op}})$, as explained in Appendix~\ref{app:known-analysis}. This scale qualification is needed for the regret analysis, not for the pathwise feasibility statement. For the default-budget price update, the stepsize choice from Section~\ref{sec:core-primal-dual-rule} gives
\begin{equation}
\label{eq:dual-stepsize-order}
\DT=O\left(Z(1+\frac{B_{\max}}{T})\sqrt{T\log(m+1)}\right).
\end{equation}

\subsection{Regret and Resource Guarantees}
The first guarantee separates estimation uncertainty from exploration, price learning, direct corruption, and stopping. Set $\delta_{\mathrm{tot}}:=\delta_{\mathrm{est}}+\delta_{\mathrm{cov}}$, combining estimator and coverage failure probabilities.

\begin{theorem}[Known corruption: regret and resource control]
\label{thm:known-general}
Suppose Assumptions~\ref{ass:base-model} and~\ref{ass:Z} hold and the estimator/coverage pair satisfies the interface in \eqref{eq:general-pointwise-confidence}--\eqref{eq:cumulative-width-interface-main} for a deterministic supplied bound $\Gamma$ with $C_Z\le\Gamma$ on every admissible realization. Algorithm~\ref{alg:known-main} with $B^{\mathrm{op}}=B$ satisfies the following expected-regret bound:
\begin{equation}
\label{eq:known-modular-bound}
\Reg(T)\le O\left(\Rest^0+h\sum_{t=1}^T\epsilon_t+\DT+\Gamma+\psi_{\mathrm{est}}(\Gamma)+Z\right)+h\delta_{\mathrm{tot}}T.
\end{equation}
Moreover, on every realization, $\Viol_{\mathrm{obs}}(T)=0$ and $\Viol_{\mathrm{cl}}(T)\le C_b(T)\le C_Z/Z \le \Gamma/Z$.
\end{theorem}

\xhdr{Proof overview.}
In \eqref{eq:known-modular-bound}, $\Rest^0$ measures clean estimation uncertainty, $h\sum_t\epsilon_t$ the cost of forced exploration, and $\DT$ the cost of learning resource prices. Corruption enters in two distinct ways: $\Gamma$ accounts for transferring corrupted scores and resource use to their clean counterparts, while $\psi_{\mathrm{est}}(\Gamma)$ accounts for its accumulated effect on prediction widths. The term $Z$ is the stopping boundary, and $h\delta_{\mathrm{tot}}T$ is the failure-event cost. Confidence controls the clean-score gap, the price update controls observed resource imbalance, and $ZB_i\ge\VUB$ offsets benchmark reward lost after stopping. Appendix~\ref{app:known-analysis} gives the full argument.

\subsection{Concrete Rates}
The next corollary instantiates the theorem with Lasso. FE collects context-independent exploration samples and pays an exploration cost. OP uses all non-null observations but requires the realized-design condition; its sharper rate is not asserted under the FE assumptions alone.

\begin{corollary}[Known corruption: concrete rates]
\label{cor:known-rates-main}
\label{cor:known-on-policy-concrete}
Under Assumptions~\ref{ass:base-model}--\ref{ass:Z}, run Algorithm~\ref{alg:known-main} with $B^{\mathrm{op}}=B$, a deterministic valid bound $\Gamma\ge C_Z$, and the estimator and certified precision of Proposition~\ref{prop:lasso-prediction-main}. For a universal $C$, with the rate constants defined in Section~\ref{sec:estimator-contract}:
\begin{enumerate}[label=(\roman*),leftmargin=2em,nosep]
\item \textbf{Forced exploration.} Under the population-design Assumption~\ref{ass:sparse-context} (calibration in Appendix~\ref{app:concrete-design-conditions}), use \eqref{eq:forced-rsc-calibration} and a constant forced-exploration probability $\epsilon_t\equiv\epsilon$ chosen by \eqref{eq:lasso-exploration-choice}. Then
\begin{equation}
\label{eq:known-fe-lasso-rate}
\begin{aligned}
\Reg(T)\le C\biggl[&h^{1/3}A_{\mathrm{FE}}^{2/3}K^{1/3}T^{2/3}
+h\sqrt{K(n_0+\ell)T}\\
&+\left(1+\frac{KH_T}{\kappa_{\mathrm{FE}}\epsilon}\right)\Gamma
+\DT+Z\biggr]+h\delta_{\mathrm{tot}}T.
\end{aligned}
\end{equation}
\item \textbf{On-policy learning.} Under the retained-design Condition~\ref{ass:on-policy-re}, with $\epsilon_t=0$, use the supplied OP calibration and the same certified Lasso rule on all non-null samples. Then
\begin{equation}
\label{eq:known-op-lasso-rate}
\begin{aligned}
\Reg(T)\le C\biggl[&A_{\mathrm{op}}\sqrt{KT}+hKn_0 +\left(1+\frac{KH_T}{\kappa_{\mathrm{op}}}\right)\Gamma +\DT+Z\biggr]+h\delta_{\mathrm{tot}}T.
\end{aligned}
\end{equation}
\end{enumerate}
\end{corollary}

The terms involving $A_{\mathrm{FE}}$ or $A_{\mathrm{op}}$ measure clean learning, and the $n_0$ terms account for the initial period before informative confidence is available. The terms proportional to $\Gamma$ measure corruption's additional effect under the corresponding design conditioning. Price learning, stopping, and failure events contribute the same $\DT+Z+h\delta_{\mathrm{tot}}T$ terms as in the theorem.

With proportional budgets, fixed remaining parameters, dominated burn-in, the stepsize following \eqref{eq:dual-control-term}, and $\delta_{\mathrm{tot}}\le T^{-2}$, the FE and OP bounds simplify to $\widetilde O(T^{2/3}+\Gamma T^{1/3})$ and $\widetilde O(\sqrt T+\Gamma)$, respectively. Here $\widetilde O$ suppresses polylogarithmic factors in $d,K,m,T$ and inverse failure levels. Appendix~\ref{app:exploration-balance} gives full analysis.

\xhdr{Comparison of rates.}
Reward-corrupted linear bandits without resource constraints achieve near-optimal or minimax corruption dependence \citep{he2022nearly,liu2024corruption}, while clean sparse CBwK exhibits $T^{2/3}$-type learning costs and, under stronger coverage, $\sqrt T$ behavior \citep{ma2023high}. Our results extend these CBwK regimes to joint reward and consumption corruption, giving $\widetilde O(T^{2/3}+\Gamma T^{1/3})$ under FE and $\widetilde O(\sqrt T+\Gamma)$ under OP, together with observed-budget feasibility and control of clean resource violation. The rates are not directly comparable across settings and no matching lower bound establishes their optimality.

\subsection{Exploration Tradeoff and the Clean Case}
The preceding FE rate uses a calibrated exploration probability. The next result keeps $\epsilon$ explicit, showing the tradeoff between collecting informative samples and the reward cost of exploration.

\begin{corollary}[Known FE with a calibrated estimator]
\label{cor:known-forced-concrete}
Under Assumptions~\ref{ass:base-model},~\ref{ass:Z}, and~\ref{ass:sparse-context}, run Algorithm~\ref{alg:known} with $B^{\mathrm{op}}=B$, a deterministic pathwise bound $C_Z\le\Gamma$, constant $0<\epsilon\le1/2$, and \eqref{eq:forced-rsc-calibration}. Use the regularization, trivial early radii, and certified numerical precision of Appendix~\ref{app:concrete-estimator}. For a universal $C$ and $\delta_{\mathrm{tot}}=\delta_{\mathrm{est}}+\delta_{\mathrm{cov}}$,
\begin{equation}
\label{eq:known-fe-unoptimized-rate}
\begin{aligned}
\Reg(T)\le C\biggl[&\frac{hK(n_0+\ell)}{\epsilon}
+A_{\mathrm{FE}}\sqrt{\frac{KT}{\epsilon}}+h\epsilon T\\
&+\DT+\Gamma+\frac{K\Gamma H_T}{\kappa_{\mathrm{FE}}\epsilon}+Z\biggr]+h\delta_{\mathrm{tot}}T.
\end{aligned}
\end{equation}
The bound holds for every $T\ge1$, even if coverage is not reached. The choice \eqref{eq:lasso-exploration-choice} gives the main-text bound \eqref{eq:known-fe-lasso-rate}, by the calculation in Appendix~\ref{app:exploration-balance}.
\end{corollary}

In \eqref{eq:known-fe-unoptimized-rate}, smaller $\epsilon$ reduces the exploration cost $h\epsilon T$ but increases the initial coverage and estimation costs, including the corruption-dependent width. The choice \eqref{eq:lasso-exploration-choice} balances the clean terms and does not depend on $\Gamma$; the corruption term is evaluated at that choice. The bound remains valid at short horizons, but need not be informative before enough samples are collected.

OP has no forced-exploration parameter. Setting both the actual corruption and the supplied allowance to zero gives the following special case of its bound.

\begin{corollary}[Clean known-corruption case]
\label{cor:known-clean}
Under the on-policy conditions of Corollary~\ref{cor:known-rates-main} with $C_Z=\Gamma=0$, the $\Gamma$-dependent contribution in \eqref{eq:known-op-lasso-rate} vanishes. The burn-in, numerical/statistical width, dual term, stopping boundary, and $h\delta_{\mathrm{tot}}T$ remain.
\end{corollary}

These claims follow by substituting Lemma~\ref{lem:comp-width-forced} or Proposition~\ref{prop:lasso-widths-on-policy} into the modular known-corruption theorem; Appendix~\ref{app:known-analysis} proves the allocation step. In particular, no sparsity-constrained optimizer or empirical sparse-eigenvalue oracle is used.

\section{Adaptation to Unknown Corruption}
\label{sec:unknown}
\label{app:unknown-details}
\phantomsection
\label{sec:unknown-obstacle}
\label{sec:shared-grid}

Without a supplied corruption bound, Shared-Grid considers several possible bounds, called \emph{candidates}, and uses a \emph{master} to combine their recommendations into one action distribution. Candidates share the retained data, fitted coefficients, resource prices, and budget account; only their confidence radii and recommendations differ. They are not separate runs of \ROPD{}: each round produces one played action and at most one retained observation.

\subsection{Candidates and Shared Estimation}
\label{sec:unknown-candidates}
\label{app:unknown-algorithm}

\xhdr{Corruption candidates.}
A candidate $\Gamma$ uses the optimistic recommendation \eqref{eq:optimistic-recommendation} with the common fits and its own radius $\beta_{a,t}^{Z,\Gamma}$. The post-clipping bounds imply $C_Z\le C_{\max,Z}:=T(1+Z)$, so zero and successive powers of two cover all admissible corruption levels:

\begin{equation}
\label{eq:unknown-grid}
\mathcal G:=\{0\}\cup\{2^j:j=0,\ldots,\lceil\log_2 C_{\max,Z}\rceil\}.
\end{equation}

The learner uses every candidate without knowing which are valid. For analysis, define
\begin{equation}
\label{eq:valid-grid-point}
\Gamma^\star:=\min\{\Gamma\in\mathcal G:\Gamma\ge C_Z\}.
\end{equation}

It satisfies
\begin{equation}
\label{eq:appendix-grid-approximation}
\Gamma^\star=0\ \text{if }C_Z=0,
\qquad
C_Z\le\Gamma^\star\le\max\{1,2C_Z\}\ \text{otherwise}.
\end{equation}

The comparator $\Gamma^\star$ is determined by the realized path, not supplied to or computed by the learner. Its recommendations use the common realized history, not a counterfactual history from running that candidate alone.

\xhdr{Initialization and retained samples.}
Begin with a prescribed cyclic schedule over the $K$ non-null actions, of length $t_{\mathrm{init}}\le T$, a nonnegative multiple of $K$. The schedule is fixed before the current context and feedback, and every round must pass the safety test \eqref{eq:pre-action-safety}; failure stops non-null play and freezes all states. Initialization observations are retained and update prices and the budget account, but not master weights.

After initialization, forced-exploration (FE) mode overrides the master with probability $\epsilon_t$ and selects uniformly from the non-null actions. Let $e_t=1$ mark this override; in on-policy (OP) mode, $\epsilon_t=e_t=0$. Recall that $I_t=1$ means the safety test passes, even if $\Null$ is then chosen. The shared retained set for action $a$ before round $t$ is

\begin{equation}
\label{eq:unknown-estimator-sample-set}
\mathcal I_{a,t}^{\mathrm{est}}=
\begin{cases}
\{\tau<t:I_\tau=1,\ a_\tau=a,\ [\tau\le t_{\mathrm{init}}\ \text{or}\ e_\tau=1]\},
&\mathsf{mode}=\mathsf{forced},\\
\{\tau<t:I_\tau=1,\ a_\tau=a\},
&\mathsf{mode}=\mathsf{on\mbox{-}policy}.
\end{cases}
\end{equation}

Thus FE retains initialization and forced samples, whereas OP retains every non-null observation. Eligibility is decided before feedback. Each eligible observation is appended once, and all candidates use the same regularization, numerical precision, supplied calibration, and fits.

\subsection{Master and Complete Procedure}
\label{sec:master-specification}
\label{sec:unknown-procedure}

\xhdr{Advice and sampling.}
The EXP4.P master \citep{beygelzimer2011contextual} combines \emph{experts}, each supplying an action distribution called its advice. Set $K_0:=|\Asetzero|=K+1$ and $\mathcal E_{\mathrm M}:=\mathcal G\cup\{\mathrm{unif}\}$, with $M:=|\mathcal G|+1$. Candidate $\Gamma$ gives advice $\xi_{t,\Gamma}(a)=\1\{\widehat a_t^\Gamma=a\}$; the extra expert gives $\xi_{t,\mathrm{unif}}(a)=1/K_0$, as required by \citet[Theorem~2]{beygelzimer2011contextual}. Advice $\xi_{t,j}$ is distinct from corruption; the extra expert requires no fit.

Expert weights $w_{t,j}$ start at one. Their normalized values $w_{t,j}/W_t$, where $W_t:=\sum_jw_{t,j}$, weight the advice in the master's action distribution $p_t$. For failure level $\delta_{\mathrm{master}}$, use probability floor $p_{\min}$ and confidence-bonus coefficient $\alpha_m$:

\begin{equation}
\label{eq:exp4p-parameters}
p_{\min}:=\min\left\{\frac1{2K_0},\sqrt{\frac{\log M}{K_0T}}\right\},
\quad
\alpha_m:=\sqrt{\frac{\log(3M/\delta_{\mathrm{master}})}{K_0T}}.
\end{equation}

The uniform expert supplies advice, smoothing ensures $p_t(a)\ge p_{\min}$, and forced exploration supplies independent FE samples. The distribution $p_t$ applies only on the non-forced branch.

\xhdr{Payoffs and updates.}
The master evaluates reward net of current resource cost and rescales this score to $[0,1]$:

\begin{equation}
\label{eq:normalized-master-payoff}
F_t^{\mathrm{obs}}(a;\lambda_t):=r_t(a)-Z\lambda_t^\top b_t(a),
\quad
Y_t(a):=\frac{F_t^{\mathrm{obs}}(a;\lambda_t)+Z}{1+Z},
\quad a\in\Asetzero.
\end{equation}

Since $F_t^{\mathrm{obs}}\in[-Z,1]$, $Y_t\in[0,1]$. Only the played action's feedback is observed. On a non-forced active round after initialization, \emph{importance weighting} divides its payoff by $p_t(a_t)$ to correct for unequal sampling probabilities. Algorithm~\ref{alg:unknown} defines the action-payoff estimate $\widehat Y_t(a)$, its advice-weighted expert estimate $\widehat y_{t,j}$, and the inverse-probability bonus term $\widehat v_{t,j}$.

The complete procedure below uses the shared estimator, including trivial early radii. Expert weights $w_{t,j}$ combine advice; price weights $q_t$ govern resources. Initialization and overrides freeze expert weights, but prices and accounting update from observed consumption on every active round.

\begin{algorithm}[tp]
\caption{Complete Shared-Grid Procedure for Unknown Corruption}
\label{alg:unknown}
\label{alg:shared-grid-summary}
\begingroup
\SingleSpacedXI
\small
\algrenewcommand\algorithmicindent{1.05em}
\begin{algorithmic}[1]
\Require Horizon $T$; initialization $t_{\mathrm{init}}$; actions $\Asetzero$; operating budget $B^{\mathrm{op}}$; scale $Z=T/B_{\min}$; coverage $\mathsf{mode}\in\{\mathsf{forced},\mathsf{on\mbox{-}policy}\}$; rates $0\le\epsilon_t\le1/2$ with $\epsilon_t=0$ on-policy; stepsize $\eta_q$; estimator inputs $(s,R,\sigma_r,\sigma_b,\kappa,\upsilon,n_0)$; failure levels $\delta_{\mathrm{est}},\delta_{\mathrm{cov}},\delta_{\mathrm{master}}$ and the certified Lasso precision.
\Statex \textbf{Estimation module:} shared predictions and radii satisfying \eqref{eq:general-pointwise-confidence} and \eqref{eq:uniform-grid-width-interface}, with the module's required inputs.
\State Form $\mathcal G$ and $\mathcal E_{\mathrm M}$ as above; set $w_{1,j}=1$ for $j\in\mathcal E_{\mathrm M}$, $q_{1,j}=1/(m+1)$, $\lambda_1=q_{1,1:m}$, $U_1=0$, and initialize one shared estimator history per action.
\For{$t=1,\ldots,T$}
  \If{$U_{t,i}>B_i^{\mathrm{op}}-1$ for some $i\in[m]$}
    \State Play $a_t=\Null$; set $I_t=e_t=0$; freeze all states; \textbf{continue}.
  \EndIf
  \State Set $I_t=1$ and observe $x_t$.
  \If{$t\le t_{\mathrm{init}}$}
    \State Play the next non-null action in the fixed cyclic schedule; set $e_t=0$ and keep the master weights fixed.
  \Else
    \State For each action, set $n=|\mathcal I_{a,t}^{\mathrm{est}}|$. Use zero coefficients if $n=0$ or no justified calibration is supplied; otherwise fit or reuse the certified channel estimates. All candidates share these fits.
    \For{$\Gamma\in\mathcal G$}
      \State Compute $\beta_{a,t}^{Z,\Gamma}$ and $\widehat a_t^\Gamma\in\argmax_{a\in\Asetzero}\{\widehat F_{a,t}(\lambda_t)+\beta_{a,t}^{Z,\Gamma}\}$.
    \EndFor
    \State Set $W_t=\sum_{j\in\mathcal E_{\mathrm M}}w_{t,j}$ and $p_t(a)=(1-K_0p_{\min})\sum_{j\in\mathcal E_{\mathrm M}}(w_{t,j}/W_t)\xi_{t,j}(a)+p_{\min}$ for $a\in\Asetzero$.
    \State In forced mode draw $e_t\sim\operatorname{Bernoulli}(\epsilon_t)$; otherwise set $e_t=0$. If $e_t=1$, draw $a_t$ uniformly from $\Aset$; otherwise draw $a_t\sim p_t$.
  \EndIf
  \State Observe $(r_t(a_t),b_t(a_t))$ for $a_t\ne\Null$; for $a_t=\Null$, set both to zero.
  \If{$a_t\ne\Null$ and [$t\le t_{\mathrm{init}}$ or $\mathsf{mode}=\mathsf{on\mbox{-}policy}$ or $e_t=1$]}
    \State Append $(x_t,r_t(a_t),b_t(a_t))$ once to the shared history for $a_t$.
  \EndIf
  \If{$t>t_{\mathrm{init}}$ and $e_t=0$}
    \State Set $Y_t(a_t)=[r_t(a_t)-Z\lambda_t^\top b_t(a_t)+Z]/(1+Z)$ and $\widehat Y_t(a)=\1\{a=a_t\}Y_t(a_t)/p_t(a_t)$.
    \For{$j\in\mathcal E_{\mathrm M}$}
      \State Set $\widehat y_{t,j}=\sum_{a\in\Asetzero}\xi_{t,j}(a)\widehat Y_t(a)$ and $\widehat v_{t,j}=\sum_{a\in\Asetzero}\xi_{t,j}(a)/p_t(a)$.
      \State Update $w_{t+1,j}=w_{t,j}\exp\{(p_{\min}/2)(\widehat y_{t,j}+\alpha_m\widehat v_{t,j})\}$.
    \EndFor
  \Else
    \State Keep every master weight unchanged.
  \EndIf
  \State Update the observed budget account and resource prices: $U_{t+1}=U_t+b_t(a_t)$ and $(q_{t+1},\lambda_{t+1})$ by \eqref{eq:dual-mw-update}.
\EndFor
\end{algorithmic}
\endgroup
\end{algorithm}

\subsection{Confidence and Recommendation Coverage}
\label{sec:unknown-coverage}

\xhdr{A common confidence event.}
Because $\Gamma^\star$ depends on the realized path, the analysis needs one event covering every valid candidate. Shared fits, the time-uniform noise event, and calibrated retained designs give simultaneous pointwise confidence; the cumulative requirement is

\begin{equation}
\label{eq:uniform-grid-width-interface}
2\sum_{t=t_{\mathrm{init}}+1}^T I_t\beta_{\widehat a_t^\Gamma,t}^{Z,\Gamma}
\le\Rshared^0+\psi_{\mathrm{est}}(\Gamma).
\end{equation}

This inequality must hold for all valid $\Gamma\in\mathcal G$ on the same event. Here $\Rshared^0$ bounds corruption-free recommendation widths after initialization, and $\psi_{\mathrm{est}}(\Gamma)$ is their additional corruption cost. The common event and radii monotone in $\Gamma$ permit substitution of $\Gamma^\star$ after the run; Appendix~\ref{app:unknown-analysis} gives the argument.

\xhdr{Why OP needs recommendation coverage.}
A candidate may repeatedly recommend an action that the master rarely plays. Retained-design coverage controls the quality of that action's data, but does not ensure enough observations for its recommendations. Let $N_a(t):=\sum_{\tau<t}\1\{a_\tau=a\}$ count actual plays, including initialization, and let $N_a^{\mathrm{rec},\Gamma}(t):=\sum_{\tau=t_{\mathrm{init}}+1}^{t-1}I_\tau\1\{\widehat a_\tau^\Gamma=a\}$ count post-initialization active recommendations. The following condition relates the two counts.

\begin{condition}[Covered recommendations]
\label{ass:candidate-coverage}
For a recommendation-to-play ratio $\chi\ge1$, additive count slack $n_{\mathrm{cand}}\ge0$, and failure level $\delta_{\mathrm{cand}}\in(0,1)$,
\begin{equation}
\label{eq:event-candidate-coverage}
\sP\left(N_a^{\mathrm{rec},\Gamma}(t)\le n_{\mathrm{cand}}+\chi N_a(t)\ \text{for all $\Gamma\in\mathcal G$, $a\in\Aset$, and $t\le T+1$}\right)\ge1-\delta_{\mathrm{cand}}.
\end{equation}
\end{condition}

The ratio $\chi$ and slack $n_{\mathrm{cand}}$ ensure that frequent recommendations are supported by actual plays. This complements the OP retained-design Condition~\ref{ass:on-policy-re}; FE instead obtains candidate-independent samples through uniform exploration and does not need this extra condition. Both OP conditions concern the actual shared trajectory, not a post-run diagnostic. The next lemma gives sufficient sampling probabilities.

\begin{lemma}[Recommendation-aligned sampling implies candidate coverage]
\label{lem:candidate-coverage-sufficient}
Suppose that on every post-initialization active round in on-policy mode, a candidate-recommended action receives conditional sampling probability at least $q_{\mathrm{cov}}>0$: whenever $\widehat a_t^\Gamma=a$,
\[
\sP(a_t=a\mid\mathcal Q_t)\ge q_{\mathrm{cov}},
\]
where $\mathcal Q_t$ contains the history, current context, shared estimates, prices, and all candidate recommendations before the fresh action draw. Then, with probability at least $1-\delta_{\mathrm{cand}}$, simultaneously for every $\Gamma\in\mathcal G$, $a\in\Aset$, and $t\le T+1$,
\begin{equation}
\label{eq:candidate-coverage-sufficient}
N_a^{\mathrm{rec},\Gamma}(t)
\le
\frac{2}{q_{\mathrm{cov}}}N_a(t)+\frac{8}{q_{\mathrm{cov}}}\log\frac{2|\mathcal G|KT}{\delta_{\mathrm{cand}}}.
\end{equation}
Consequently, Condition~\ref{ass:candidate-coverage} holds with
\[
\chi=\frac{2}{q_{\mathrm{cov}}},
\quad
n_{\mathrm{cand}}
=\frac{8}{q_{\mathrm{cov}}}\log\frac{2|\mathcal G|KT}{\delta_{\mathrm{cand}}}.
\]
\end{lemma}

The proof is in Appendix~\ref{app:candidate-coverage-proof}. Default smoothing gives $q_{\mathrm{cov}}=p_{\min}$ and hence finite-horizon coverage, but $p_{\min}$ decreases with $T$: it does not establish the horizon-independent ratio used in the sharper OP summary. A horizon-independent lower bound on recommendation-aligned sampling instead gives constant $\chi$ and logarithmic slack. Appendix~\ref{app:primitive-coverage-example} illustrates nonempty design and count conditions, without asserting them for every adaptive OP trajectory. Table~\ref{tab:shared-grid-notation} summarizes the additional notation.

\begin{table}[H]
\centering
\caption{Additional notation for Shared-Grid and recommendation coverage.}
\label{tab:shared-grid-notation}
\small
\begin{tabularx}{\linewidth}{@{}lX@{}}
\toprule
Symbol & Definition and role\\
\midrule
$t_{\mathrm{init}}$ & Prescribed cyclic initialization length, divisible by $K$ and subject to budget safety.\\
$|\mathcal G|$ & Number of corruption candidates in $\mathcal G$.\\
$K_0$ & Number of actions including the null action, $K_0=|\Asetzero|=K+1$.\\
$p_{\min}$ & Probability floor for each action on the master-controlled branch.\\
$\alpha_m$ & EXP4.P confidence-bonus coefficient.\\
$N_a^{\mathrm{rec},\Gamma}(t)$ & Post-initialization active-round count of candidate $\Gamma$ recommending action $a$ before $t$.\\
$n_{\mathrm{cand}}$ & Additive slack when comparing recommendation counts with actual plays.\\
$\chi$ & Multiplicative bound on recommendations relative to actual plays.\\
\bottomrule
\end{tabularx}
\end{table}

\subsection{Regret and Resource Guarantees}
\label{sec:unknown-guarantee}
\label{sec:unknown-unoptimized}

The master-comparison cost includes override concentration and stochastic observed-to-clean score transfer:

\begin{equation}
\label{eq:unknown-master-complexity}
\Reg_{\mathrm{mst}}(T):=20h\sqrt{TK_0\log(3M/\delta_{\mathrm{master}})}.
\end{equation}

Here $h=1+Z$; direct corruption is charged separately. When the nontrivial-horizon condition for EXP4.P fails, the score-range bound gives the same order. The failure totals are

\begin{equation}
\label{eq:unknown-delta-total}
\delta_{\mathrm{tot}}^{\mathrm{U,FE}}:=\delta_{\mathrm{est}}+\delta_{\mathrm{cov}}+\delta_{\mathrm{master}},\quad \delta_{\mathrm{tot}}^{\mathrm{U,OP}}:=\delta_{\mathrm{tot}}^{\mathrm{U,FE}}+\delta_{\mathrm{cand}}.
\end{equation}

Below, $\delta_{\mathrm{tot}}$ is the FE or OP total, according to the mode; OP additionally includes the recommendation-coverage event.

\begin{theorem}[Unknown corruption: regret and resource control]
\label{thm:unknown-general}
Run Algorithm~\ref{alg:shared-grid-summary} with $B^{\mathrm{op}}=B$ and the master specification of Section~\ref{sec:master-specification}. Suppose Assumptions~\ref{ass:base-model} and~\ref{ass:Z} hold, and the pointwise and grid-uniform width interfaces \eqref{eq:general-pointwise-confidence} and \eqref{eq:uniform-grid-width-interface} hold on the corresponding estimator-and-coverage event. Then
\begin{equation}
\label{eq:unknown-modular-bound}
\begin{aligned}
\Reg(T)\le O\biggl(&\Rshared^0+ht_{\mathrm{init}}+h\sum_{t=1}^T\epsilon_t+\Reg_{\mathrm{mst}}(T)+\DT+Z\\[-1mm]
&+\E[3\Gamma^\star+\psi_{\mathrm{est}}(\Gamma^\star)]\biggr)+h\delta_{\mathrm{tot}}T.
\end{aligned}
\end{equation}
On every realization, $\Viol_{\mathrm{obs}}(T)=0$ and $\Viol_{\mathrm{cl}}(T)\le C_b(T)\le C_Z/Z\le\Gamma^\star/Z$. Consequently, $\E[\Viol_{\mathrm{cl}}(T)]\le\E[\Gamma^\star]/Z$.
\end{theorem}

\xhdr{Interpretation and proof overview.}
Relative to known corruption, the bound uses shared widths, adds initialization $ht_{\mathrm{init}}$ and master cost $\Reg_{\mathrm{mst}}(T)$, and compares with $\Gamma^\star$. The term $3\Gamma^\star$ accounts for two observed-to-clean score transfers and one resource transfer, separately from the prediction-width cost $\psi_{\mathrm{est}}(\Gamma^\star)$. The comparator is path-dependent, so these costs are averaged; feasibility remains pathwise. Price learning and stopping use the same argument as before. Appendix~\ref{app:unknown-analysis} gives the proof.

\xhdr{Concrete Lasso rates.}
Substituting the shared FE or OP widths gives the following rates, with $A_{\mathrm{FE}},A_{\mathrm{op}},n_0,H_T$ defined in Section~\ref{sec:estimator-contract}.

\begin{corollary}[Unknown corruption: concrete shared-data rates]
\label{cor:unknown-rates-main}
\label{cor:unknown-on-policy-concrete}
Under Assumptions~\ref{ass:base-model}--\ref{ass:Z}, run Algorithm~\ref{alg:shared-grid-summary} with $B^{\mathrm{op}}=B$, the shared estimator and precision of Proposition~\ref{prop:lasso-prediction-main}, and the EXP4.P specification of Section~\ref{sec:master-specification}. With a universal $C$ and the corresponding failure total in \eqref{eq:unknown-delta-total}:
\begin{enumerate}[label=(\roman*),leftmargin=2em,nosep]
\item \textbf{Forced exploration.} Under the population-design Assumption~\ref{ass:sparse-context} (calibration in Appendix~\ref{app:concrete-design-conditions}), use \eqref{eq:forced-rsc-calibration} and a constant forced-exploration probability $\epsilon_t\equiv\epsilon$ chosen by \eqref{eq:lasso-exploration-choice}. Then
\begin{equation}
\label{eq:unknown-fe-lasso-rate}
\begin{aligned}
\Reg(T) & \le C\biggl[h^{1/3}A_{\mathrm{FE}}^{2/3}K^{1/3}T^{2/3}
+h\sqrt{K(n_0+\ell)T}+ht_{\mathrm{init}}\\
&+\left(3+\frac{KH_T}{\kappa_{\mathrm{FE}}\epsilon}\right)\E[\Gamma^\star]
+\Reg_{\mathrm{mst}}(T)+\DT+Z\biggr]+h\delta_{\mathrm{tot}}T.
\end{aligned}
\end{equation}
\item \textbf{On-policy learning.} Under Conditions~\ref{ass:on-policy-re} and~\ref{ass:candidate-coverage}, with $\epsilon_t=0$, use the stated initialization, complete grid, and uniform-advice expert. Fit every realized non-null sample using the supplied OP calibration and the same certified Lasso rule. Then
\begin{equation}
\label{eq:unknown-op-lasso-rate}
\begin{aligned}
\Reg(T)&\le C\biggl[\chi A_{\mathrm{op}}\sqrt{KT}
+\Reg_{\mathrm{mst}}(T)+h\{t_{\mathrm{init}}+K(n_{\mathrm{cand}}+\chi n_0)\}\\
&+\left(3+\frac{K\{\chi H_T+n_{\mathrm{cand}}\}}{\kappa_{\mathrm{op}}}\right)\E[\Gamma^\star]+\DT+Z\biggr]
+h\delta_{\mathrm{tot}}T.
\end{aligned}
\end{equation}
\end{enumerate}
\end{corollary}

Compared with Corollary~\ref{cor:known-rates-main}, both modes add initialization and master costs; OP also pays the recommendation-coverage factors $\chi,n_{\mathrm{cand}}$. Under the same proportional-budget and stepsize regime, fixed remaining parameters, dominated initialization and burn-in, and $\delta_{\mathrm{tot}}\le T^{-2}$, the FE and OP summaries are $\widetilde O(T^{2/3}+\E[\Gamma^\star]T^{1/3})$ and $\widetilde O(\sqrt T+\E[\Gamma^\star])$. For OP, default smoothing alone gives horizon-dependent coverage parameters; those factors must remain in the explicit bound.

\xhdr{Arbitrary exploration rate.}
As in Section~\ref{sec:general-known}, the unoptimized FE bound exposes the dependence on a constant exploration probability:

\begin{corollary}[Unknown FE with shared estimation]
\label{cor:unknown-forced-concrete}
Under Assumptions~\ref{ass:base-model},~\ref{ass:Z}, and~\ref{ass:sparse-context}, run Algorithm~\ref{alg:unknown} with $B^{\mathrm{op}}=B$, fixed safe cyclic initialization length $t_{\mathrm{init}}$, constant $0<\epsilon\le1/2$, and \eqref{eq:forced-rsc-calibration}. Use the complete grid, uniform-advice master expert, and Appendix~\ref{app:concrete-estimator}'s common regularization and certified precision on the initialization and later forced samples. For a universal $C$ and $\delta_{\mathrm{tot}}=\delta_{\mathrm{tot}}^{\mathrm{U,FE}}$,
\begin{equation}
\label{eq:unknown-fe-unoptimized-rate}
\begin{aligned}
\Reg(T)\le C\biggl[&ht_{\mathrm{init}}+\frac{hK(n_0+\ell)}{\epsilon}
+A_{\mathrm{FE}}\sqrt{\frac{KT}{\epsilon}}+h\epsilon T\\
&+\Reg_{\mathrm{mst}}(T)+\DT+Z
+\left(3+\frac{KH_T}{\kappa_{\mathrm{FE}}\epsilon}\right)\E[\Gamma^\star]\biggr]
+h\delta_{\mathrm{tot}}T.
\end{aligned}
\end{equation}
Choosing \eqref{eq:lasso-exploration-choice} gives \eqref{eq:unknown-fe-lasso-rate}, including its capped regime, by the calculation in Appendix~\ref{app:exploration-balance}.
\end{corollary}

The calibration \eqref{eq:lasso-exploration-choice} balances clean learning, burn-in, and exploration without using the unknown corruption level. Appendix~\ref{app:exploration-balance} gives the calculation, including the capped regime.

\FloatBarrier

\section{Numerical Experiments}
\label{sec:experiments}
\label{app:numerical}
\label{app:coverage-diagnostic-details}

This section gives finite-sample illustrations of the four mechanisms emphasized by the theory: robustness with a supplied corruption bound, adaptation when the corruption level is unknown, the distinct roles of reward and consumption corruption, and the pathwise resource-accounting guarantee. The experiments use the same clean population-LP benchmark as the analysis and evaluate reward with the clean conditional means, so corrupted feedback affects the learner but not the benchmark. We focus on robustness as the corruption level varies rather than fitting an empirical asymptotic exponent; the latter would mix finite-horizon initialization, master, design, and logarithmic terms that are kept explicit in Corollaries~\ref{cor:known-rates-main} and~\ref{cor:unknown-rates-main}.

\subsection{Common environment and evaluation protocol}
\label{app:exp-setup}

\xhdr{Three-specialist instance.}
The main experiments use $K=3$ non-null actions, $m=2$ resources, dimension $d=20$, true and calibration sparsity $s_0=s=3$, and parameter-radius bound $R=1.2$. Budgets are $B=(0.5T,0.5T)$, hence $Z=2$. Draw $q_1,q_2$ independently and uniformly from $\{-1,-0.5,0.5,1\}$ and $q_3,z_5,\ldots,z_{20}$ independently and uniformly from $\{-1,+1\}$, and set
\begin{equation}
\label{eq:exp-context}
x=(0.50,\,0.40q_1,\,0.40q_2,\,0.20q_3,\,0.15z_5,\ldots,0.15z_{20}).
\end{equation}
Then $\|x\|_2^2\le0.97$. The clean reward means are
\begin{equation}
\label{eq:exp-reward-means}
\begin{aligned}
r_1^0(x)&=0.55+0.08q_1+0.02q_3,\\
r_2^0(x)&=0.55-0.08q_1+0.02q_3,\\
r_3^0(x)&=0.57+0.08q_2+0.02q_3,
\end{aligned}
\end{equation}
and the clean mean consumptions are
\begin{equation}
\label{eq:exp-consumption-means}
b_1^0=(0.60,0.40)^\top,\qquad b_2^0=(0.40,0.60)^\top,\qquad b_3^0=(0.50,0.50)^\top.
\end{equation}

Independent reward and consumption noises are uniform on $[-0.02,0.02]$. The exact finite-support population LP has value $\VUB=0.6225T$, action masses $(5/16,5/16,6/16)$, and resource use $(0.5,0.5)$; thus all three actions are clean-LP relevant and both resources bind in expectation. The minimum clean reward gap between the best and second-best action over the supported $(q_1,q_2)$ states is $0.02$.

\xhdr{Algorithms and practical calibration.}
All headline policy comparisons use the on-policy mode from Sections~\ref{sec:core-primal-dual-rule} and~\ref{sec:unknown}. We use the norm-constrained Lasso implementation in Section~\ref{sec:concrete-sparse-estimator} with fixed practical coefficients $c_\rho=0.25$, $c_\beta=0.75$, $c_C=1$, unit dual multiplier, start count $20$, and certified optimization gap $10^{-8}$. Shared-Grid uses the complete dyadic grid, the explicit uniform-advice expert, the EXP4.P specification of Section~\ref{app:unknown-details}, and $201$ safe cyclic initialization rounds (67 per non-null action). These settings are held fixed across the reported comparisons.

We compare \ROPD{} with a valid supplied bound, a Reward-Only Robust ablation that protects reward confidence but not the consumption channel, and Non-Robust PD with zero corruption allowance. Shared-Grid receives no corruption bound. Attack maps are functions of the potential contexts and are fixed before the learner's fresh action draw; compared methods receive the same contexts, clean potential outcomes, and exogenous attack map within each paired seed.

For a policy $\pi$, define the reported percentage Relative Regret by
\begin{equation}
\label{eq:exp-relative-regret}
\operatorname{RelReg}(\pi)
:=100\,\frac{\VUB-\sum_{t=1}^T r^0(a_t,x_t)}{\VUB},
\end{equation}
where the null action contributes zero, including after stopping.  Each reported cell uses 20 paired replications. Error bars and bands are 95\% percentile bootstrap intervals from 2,000 whole-replication resamples; paired differences and corrupted-minus-clean increments are formed within seed before resampling, and rounds are never treated as independent replicates. The final runs used in this section have no excluded seeds, no clipping, and no solver failures.

\begin{table}[t]
\centering
\caption{Theory-to-experiment map. The experiments are finite-sample illustrations of the mechanisms in the corresponding guarantees; empirical design and master diagnostics do not replace the stated probabilistic coverage conditions.}
\label{tab:exp-theory-map}
\footnotesize
\setlength{\tabcolsep}{4pt}
\begin{tabularx}{\linewidth}{@{}p{0.12\linewidth}p{0.3\linewidth}X@{}}
\toprule
Experiment & Theory connection & Finite-sample question \\
\midrule
Figure~\ref{fig:exp-known} & Theorem~\ref{thm:known-general}, Corollary~\ref{cor:known-rates-main} & How absolute regret changes as a valid known corruption level increases. \\
Figure~\ref{fig:exp-unknown} & Theorem~\ref{thm:unknown-general}, Corollary~\ref{cor:unknown-rates-main} & How much additional regret unknown corruption causes after accounting for each method's clean finite-horizon cost. \\
Figure~\ref{fig:exp-channels} & Joint reward--consumption model and the Reward-Only ablation & How the location of a fixed corruption budget across reward and resource channels changes performance. \\
Figure~\ref{fig:exp-accounting} & Proposition~\ref{prop:feasibility-levels} and the pathwise parts of Theorems~\ref{thm:known-general}--\ref{thm:unknown-general} & How over- and under-reported consumption affect stopping and clean resource use. \\
\bottomrule
\end{tabularx}
\end{table}

\subsection{Known corruption: performance as \texorpdfstring{$C_Z$}{C-Z} increases}
\label{app:exp-known}

We first study the setting of Theorem~\ref{thm:known-general}, where the learner is given a deterministic valid bound $\Gamma\ge C_Z$.  Fix $T=20{,}000$. Corruption targets Action~1 only on contexts where it is uniquely clean-preferred and applies the joint shift
\[
(\Delta r_1,\Delta b_{1,1},\Delta b_{1,2})=(-0.20,+0.10,+0.10).
\]

With $Z=2$, each attacked opportunity contributes $0.40$ units to $C_Z$. We vary $C_Z\in\{0,100,200,300,400\}$ by attacking $0,250,500,750,1000$ spread eligible opportunities.

Figure~\ref{fig:exp-known} reports absolute Relative Regret. All three methods coincide at $C_Z=0$ (mean $0.756\%$). ROPD has the lowest mean regret at every positive tested corruption level. At $C_Z=400$, mean Relative Regret is $1.333\%$ for ROPD, $1.383\%$ for Reward-Only Robust, and $2.008\%$ for Non-Robust PD. Thus the separation from the non-robust baseline becomes substantial as corruption grows, while the Reward-Only ablation remains between the full robust method and the non-robust policy throughout the main sweep. This is consistent with the corruption-dependent terms in Corollary~\ref{cor:known-rates-main}: the corruption-aware policies trade wider confidence sets for reduced sensitivity to contaminated feedback, and protecting both channels is useful when resource feedback can also be corrupted.

\begin{figure}[t]
\centering
\includegraphics[width=0.8\linewidth]{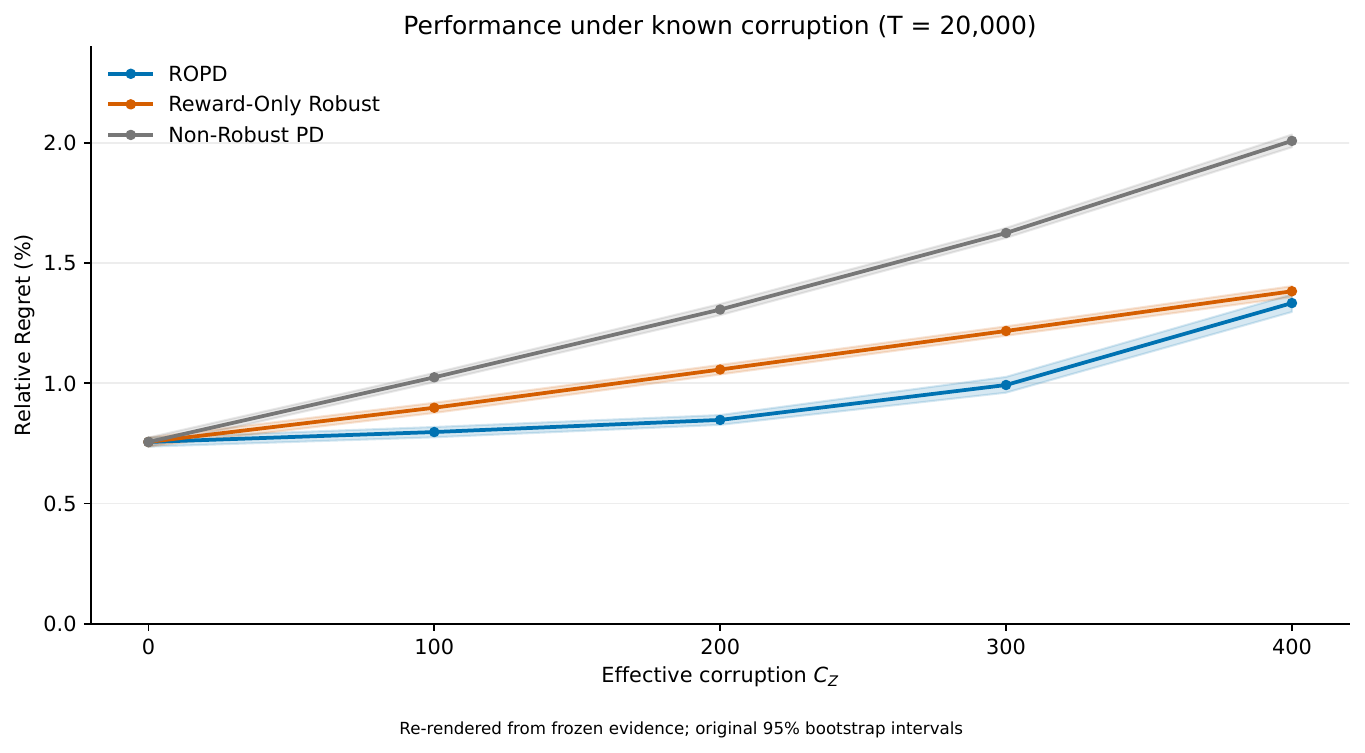}
\caption{\textbf{Known-corruption robustness.} At fixed $T=20{,}000$, the effective corruption level $C_Z$ increases along the horizontal axis. The vertical axis is absolute Relative Regret against the clean population-LP benchmark. Bands are the frozen 95\% whole-replication bootstrap intervals.}
\label{fig:exp-known}
\end{figure}

\subsection{Unknown corruption: adaptation on one shared history}
\label{app:exp-unknown}

We next remove the corruption-bound input and evaluate Shared-Grid from Algorithm~\ref{alg:shared-grid-summary}. Because Shared-Grid pays its initialization and master-comparison costs even on clean data, the primary quantity here is the \emph{corruption-induced increase}
\begin{equation}
\label{eq:exp-corruption-increment}
\Delta\operatorname{RelReg}_M(C_Z)
:=\operatorname{RelReg}_M(C_Z)-\operatorname{RelReg}_M(0)
\end{equation}
for method $M$, formed within paired seeds. This isolates sensitivity to corruption from the method's clean finite-horizon cost. Shared-Grid may still have larger absolute finite-horizon regret because it also pays initialization and master-comparison costs.

Fix $T=80{,}000$ and vary the unknown corruption budget over $C_Z\in\{0,1000,2000,3000,4000\}$. The same joint shift $(-0.20,+0.10,+0.10)$ targets clean-preferred Action~1 opportunities; Shared-Grid is never given $C_Z$.  Figure~\ref{fig:exp-unknown} shows that Shared-Grid incurs a smaller corruption-induced increase at every positive severity. At $C_Z=4000$, the increase is $1.313$ percentage points for Shared-Grid versus $2.574$ for Non-Robust PD. The paired Shared-minus-Non-Robust differences are $-0.262$, $-0.770$, $-0.916$, and $-1.261$ percentage points at $C_Z=1000,2000,3000,4000$, respectively; each pointwise paired 95\% interval lies below zero.

These results illustrate the role of the adaptive grid in Theorem~\ref{thm:unknown-general}: even without a supplied corruption level, the shared-data master is markedly less sensitive to increasing corruption on this instance. The result should be read together with the explicit master term in Corollary~\ref{cor:unknown-rates-main}. In absolute terms Shared-Grid still pays a visible clean finite-horizon master cost; the experiment therefore supports adaptation to corruption rather than an absolute dominance claim over the non-robust policy.

\begin{figure}[t]
\centering
\includegraphics[width=0.8\linewidth]{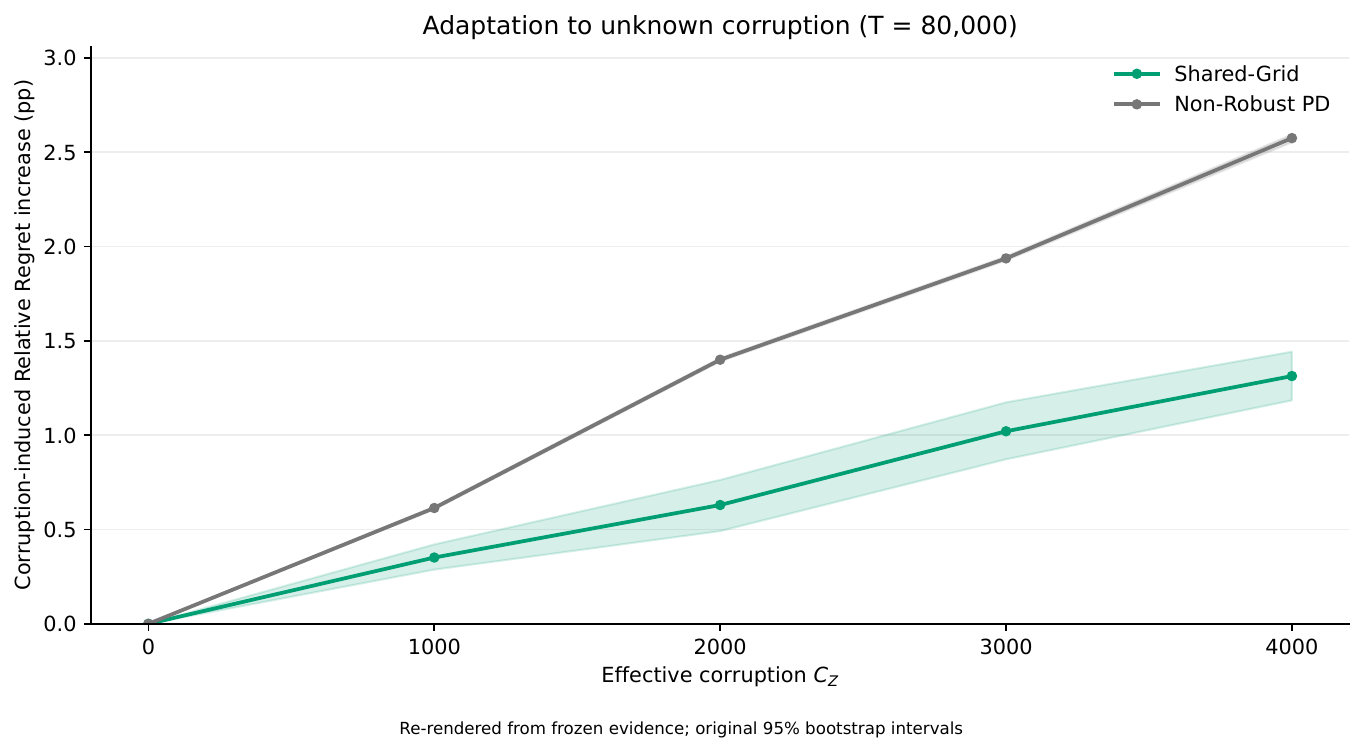}
\caption{\textbf{Unknown-corruption adaptation.} At fixed $T=80{,}000$, the horizontal axis is the unknown effective corruption level $C_Z$, and the vertical axis is the paired corruption-induced increase in Relative Regret. Bands are 95\% paired whole-replication bootstrap intervals.}
\label{fig:exp-unknown}
\end{figure}

\subsection{Reward, consumption, and joint corruption}
\label{app:exp-channels}

The preceding experiments vary the amount of corruption. We now fix its effective budget and vary \emph{where} it enters the feedback.  Set $T=40{,}000$ and attack the same 1,200 clean-preferred Action~1 opportunities in every corrupted condition. The shifts are
\[
\begin{array}{ll}
\text{Reward:}&(-0.20,0,0),\\
\text{Consumption:}&(0,+0.10,+0.10),\\
\text{Joint:}&(-0.10,+0.05,+0.05).
\end{array}
\]
Each attacked opportunity contributes $0.20$ to $C_Z$, so all three conditions have exactly $C_Z=240$, with the same target, attacked contexts, count, and timing within a paired seed.

Figure~\ref{fig:exp-channels} reports absolute Relative Regret under the corrupted environments. ROPD has the lowest mean regret in all three channels: $0.416\%$ under reward corruption, $0.907\%$ under consumption corruption, and $0.655\%$ under joint corruption. Reward-Only Robust attains $0.501\%$, $0.989\%$, and $0.738\%$, respectively, while Non-Robust PD attains $0.617\%$, $1.086\%$, and $0.840\%$. The ordering is consistent across the three equal-$C_Z$ corrupted environments and illustrates why consumption corruption is a separate modeling issue: it changes not only a regression target but also the resource-price and budget-accounting path. ROPD has the lowest absolute finite-horizon Relative Regret in each tested channel; the matched-clean incremental comparison is reported separately below and is more nuanced.

\begin{figure}[t]
\centering
\includegraphics[width=0.8\linewidth]{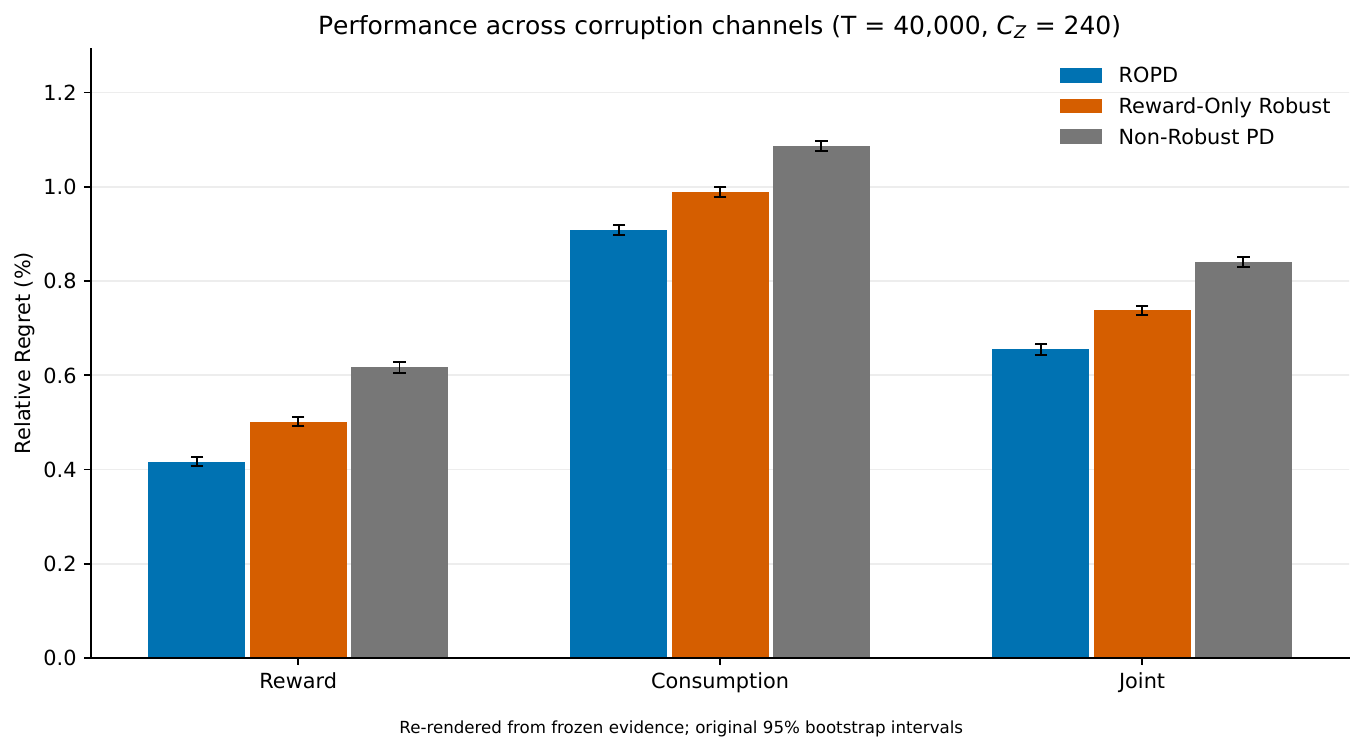}
\caption{\textbf{Corruption channels.} Reward, consumption, and joint attacks use the same 1,200 targeted opportunities and the same effective corruption budget $C_Z=240$ at $T=40{,}000$. Bars show absolute Relative Regret. Error bars are the frozen 95\% whole-replication bootstrap intervals.}
\label{fig:exp-channels}
\end{figure}

\subsection{Pathwise resource accounting under corrupted consumption}
\label{app:exp-accounting}

Finally, we isolate the resource-accounting statement of Proposition~\ref{prop:feasibility-levels}. This controlled stress test uses the same problem scale ($K=3$, $m=2$, $d=20$, $s=3$) and perturbs one frequently used action--resource pair by $\pm0.10$. The goal is not a policy ranking: it is to make the two pathwise effects of consumption corruption visible.

When consumption is over-reported, the observed ledger is depleted too quickly and the safety rule can stop allocation early. Across increasing severities, mean lost rounds rise from zero to $341$ at the largest tested perturbation. When consumption is under-reported, the observed ledger remains feasible while clean resource use can exceed the nominal budget. At the largest severity, mean clean resource-1 excess is $361.243$ and mean played consumption corruption is $C_b=363.665$. More importantly, every trajectory satisfies
\begin{equation}
\label{eq:exp-resource-bound}
\Viol_{\mathrm{obs}}(T)=0,
\qquad
\Viol_{\mathrm{cl}}(T)\le C_b(T)\le C_Z/Z,
\end{equation}
with no additional additive unit. Thus Figure~\ref{fig:exp-accounting} directly illustrates the distinction between observed feasibility and clean feasibility that motivates the resource component of the theory.

\begin{figure}[t]
\centering
\includegraphics[width=0.96\linewidth]{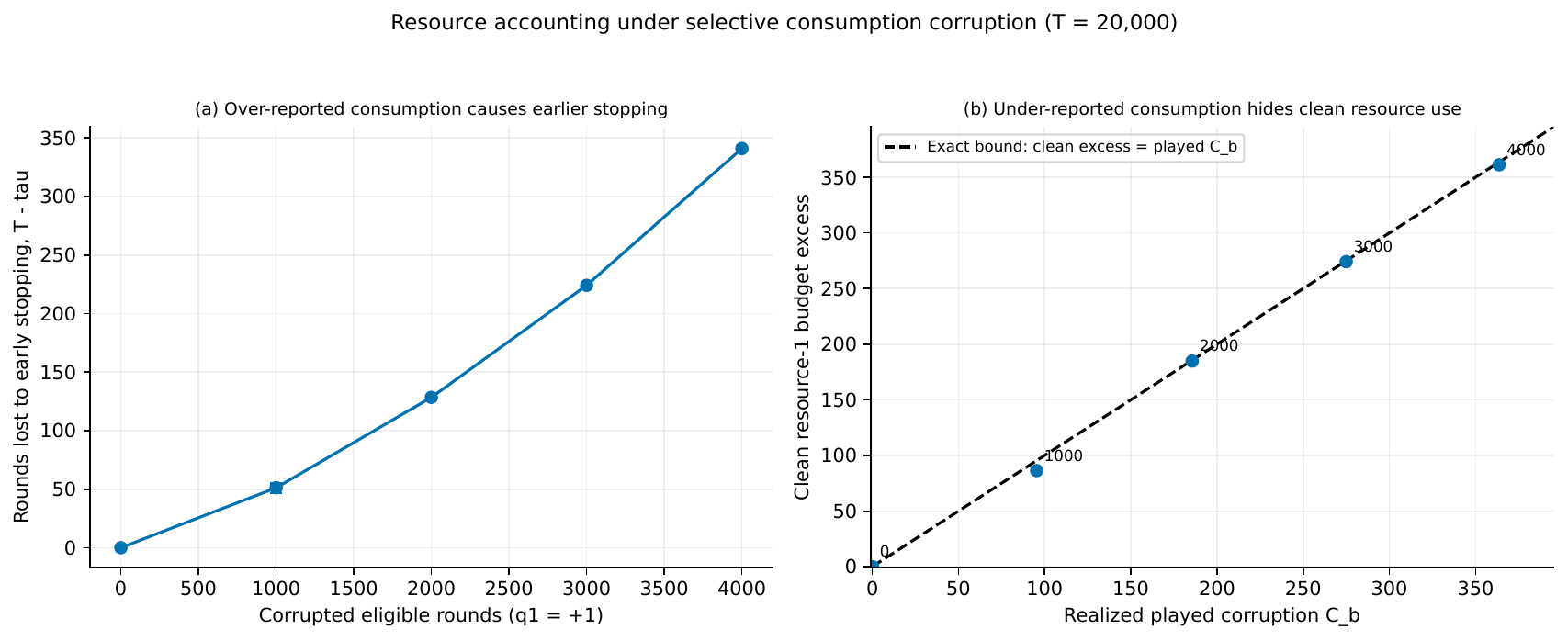}
\caption{\textbf{Resource accounting under consumption corruption.} Over-reported consumption advances stopping (left), while under-reported consumption can hide clean resource use (right). The diagonal in the right panel is the pathwise upper bound $\Viol_{\mathrm{cl}}(T)\le C_b(T)$.  Observed-budget violation is zero on every reported trajectory.}
\label{fig:exp-accounting}
\end{figure}

\subsection{Additional checks and interpretation}
\label{app:exp-controls}

The main experiments above measure finite-horizon performance under known corruption, unknown corruption, different feedback channels, and corrupted resource accounting.  We report three additional checks to clarify the interpretation of these results.

\xhdr{Matched clean controls.}
For the known-corruption algorithms, supplying a positive bound $\Gamma$ changes the confidence widths even when the observed feedback is clean. We therefore also compare each corrupted run with a clean run using the same supplied bound.  This separates the effect of corrupted feedback from the finite-sample effect of the corruption-aware confidence allowance.

For the sweep in Figure~\ref{fig:exp-known}, both robust methods have smaller matched corruption increments than Non-Robust PD at the two largest tested levels, $C_Z=300$ and $400$.  At $C_Z=400$, the increments are $1.153$ percentage points for ROPD, $1.113$ for Reward-Only Robust, and $1.252$ for Non-Robust PD.  The difference between ROPD and Reward-Only is small in this matched comparison, so Figure~\ref{fig:exp-known} is interpreted primarily as an absolute finite-horizon performance comparison.

The same control helps interpret Figure~\ref{fig:exp-channels}.  With clean feedback and the same supplied $\Gamma=240$, Relative Regret is $0.353\%$ for ROPD and $0.434\%$ for Reward-Only Robust, compared with $0.546\%$ in the ordinary $\Gamma=0$ clean setting.  Table~\ref{tab:exp-matched-channel} reports the resulting channel-specific increments.

\begin{table}[t]
\centering
\caption{Increase in Relative Regret relative to matched clean controls (percentage points).  ROPD and Reward-Only Robust use clean controls with the same supplied $\Gamma=240$; Non-Robust PD has no corruption-bound input.}
\label{tab:exp-matched-channel}
\small
\setlength{\tabcolsep}{6pt}
\begin{tabular}{@{}lccc@{}}
\toprule
Channel & ROPD & Reward-Only Robust & Non-Robust PD \\
\midrule
Reward      & 0.0631 & 0.0673 & 0.0709 \\
Consumption & 0.5543 & 0.5551 & 0.5403 \\
Joint       & 0.3019 & 0.3034 & 0.2939 \\
\bottomrule
\end{tabular}
\end{table}

The matched-control comparison gives a more refined view of the channel effects. In the consumption and joint conditions, the three incremental regrets are numerically close; Non-Robust PD has slightly smaller increments than ROPD by $0.0140$ and $0.0080$ percentage points, respectively, in the paired comparisons.  This does not conflict with the absolute comparison in Figure~\ref{fig:exp-channels}.  The robust methods also have lower matched clean regret when supplied with $\Gamma=240$, and under the corrupted environments themselves ROPD attains the lowest Relative Regret in all three channels. Thus, Figure~\ref{fig:exp-channels} summarizes finite-horizon performance under corruption, while the matched controls isolate the additional effect of corrupted feedback relative to each method's corresponding clean baseline.

\xhdr{Horizon and design checks.}
A separate frozen horizon experiment uses $T\in\{40{,}000,80{,}000,160{,}000\}$ with corruption of order $T^{2/3}$. The corruption-induced increase for Shared-Grid decreases from $1.612$ to $0.920$ percentage points over these horizons, while the corresponding increase for Non-Robust PD decreases from $3.530$ to $1.851$. Clean and corrupted Relative Regret also decrease with $T$. These results provide a finite-horizon consistency check for the sublinear behavior in Corollaries~\ref{cor:known-rates-main} and \ref{cor:unknown-rates-main}, rather than an empirical estimate of their asymptotic exponents.

The experiment also has nondegenerate action-specific designs: all three specialist actions have positive clean LP mass, and their population second moments on the true reward supports are strictly positive. For Shared-Grid, candidate recommendations become meaningfully different under corruption; at $C_Z=4000$, grid candidates disagree on about $61\%$ of active post-initialization rounds.  These diagnostics confirm that the finite-sample experiment exercises the learning and adaptation mechanisms studied in the theory.

\xhdr{Summary.}
The experiments illustrate four distinct parts of the analysis.  With a supplied corruption bound, ROPD has lower absolute Relative Regret than the compared baselines throughout the tested corruption sweep.  Without such a bound, Shared-Grid incurs substantially less corruption-induced degradation than Non-Robust PD across the tested unknown-corruption levels, while retaining its finite-horizon master cost.  Across reward, consumption, and joint corruption, ROPD has the lowest absolute Relative Regret in the tested corrupted environments; the matched-clean incremental comparisons are more nuanced and are interpreted separately above.  Finally, the accounting experiment directly illustrates exact observed feasibility and the corruption-controlled clean resource bound.


\section{Conclusion and Limitations}
\label{sec:conclusion}
We develop corruption-robust contextual allocation methods that account for both contaminated predictions and errors in resource accounting, pricing, and stopping. \ROPD{} uses a supplied corruption bound, while Shared-Grid adapts confidence radii on one realized history. Both preserve observed budgets pathwise and control clean resource violation by consumption corruption. The estimator interface separates these allocation guarantees from the choice of sparse fitting method. Forced exploration provides one coverage route; the sharper on-policy rates require informative realized designs, with the additional recommendation-coverage condition for Shared-Grid. Obtaining less conservative calibration and weaker trajectory conditions remains a direction for further work.

\bibliographystyle{plainnat}
\bibliography{references}

\clearpage
\begin{APPENDICES}
\renewcommand*{\theHsection}{appendix.\arabic{section}}

\section{Lasso Confidence, Design Calibration, and Cumulative Widths}
\label{app:prediction-guarantees}

\subsection{Retained Histories and the Confidence Interface}
\label{app:sparse-prediction}
\label{app:estimator-interface}
The regression sample set, not the recommendation count, determines each fit. With $e_\tau$ the independent forced-exploration indicator, the three retained histories are
\begin{equation}
\label{eq:estimator-sample-sets}
\begin{aligned}
\mathcal S_a^{\mathrm{FE}}(t)&=\{\tau<t:I_\tau e_\tau=1,\ a_\tau=a\},\\
\mathcal S_a^{\mathrm{OP}}(t)&=\{\tau<t:a_\tau=a\},\\
\mathcal S_a^{\mathrm{sh,FE}}(t)&=\{\tau\le t_{\mathrm{init}}:\tau<t,\ I_\tau=1,\ a_\tau=a\}
\cup\{t_{\mathrm{init}}<\tau<t:I_\tau e_\tau=1,\ a_\tau=a\}.
\end{aligned}
\end{equation}
These sets are defined for $a\in\Aset$; let $N_a^{\mathrm{FE}}(t):=|\mathcal S_a^{\mathrm{FE}}(t)|$ and $N_a^{\mathrm{est}}(t):=|\mathcal S_a^{\mathrm{est}}(t)|$ for the selected mode. In OP it equals the realized non-null action count $N_a(t)$; in FE it generally does not. The candidate-recommendation count $N_a^{\mathrm{rec},\Gamma}(t)$ is a different object from either retained count.

\subsection{Regularization and Prediction Calibration}
\label{app:concrete-estimator}

Section~\ref{sec:concrete-sparse-estimator} defines the norm-constrained objective and states Proposition~\ref{prop:lasso-prediction-main}. Here we specify its inputs and calibration. Supply $R>0$, an integer $\max\{1,s_0\}\le s\le d$, and noise bounds $\sigma_r,\sigma_b\ge0$. Bounded clean outcomes permit $\sigma_r=\sigma_b=1/2$ by Hoeffding's lemma. The fit does not require a support constraint; $s$ is used for confidence and numerical-precision calibration. For one action and channel, $n$, $X$, $y$, and $\theta^\star$ have the local meanings given in Section~\ref{sec:concrete-sparse-estimator}; the corresponding estimate $\widehat\theta$ is $\widehat\theta_{a,t-1}$. At zero samples set it to zero. Exact fits select the minimum-Euclidean-norm minimizer, while approximate fits use the deterministic certified solver in Appendix~\ref{app:tractable-options}.

A design calibration consists of supplied $\kappa,\upsilon>0$ and an integer $n_0$. For failure levels $\delta_{\mathrm{est}},\delta_{\mathrm{cov}}\in(0,1)$ and $n\ge1$, set
\begin{equation}
\label{eq:lasso-calibration}
\ell=\log\frac{16eK(m+1)dT}{\min\{\delta_{\mathrm{est}},\delta_{\mathrm{cov}}\}},\quad
\omega_n=\frac{\upsilon\ell}{n},\quad
\rho_n=4(\sigma+R\sqrt{\kappa\upsilon})\sqrt{\frac{2\ell}{n}},\quad
n_0\ge\left\lceil\frac{256s\upsilon\ell}{\kappa}\right\rceil.
\end{equation}
Use $\sigma=\sigma_r$ or $\sigma_b$ for the respective channel, writing $\rho_n^r$ or $\rho_n^b$ accordingly. The required restricted strong convexity (RSC) event is
\begin{equation}
\label{eq:empirical-rsc}
\frac{1}{n}\|Xu\|_2^2\ge\kappa\|u\|_2^2-\omega_n\|u\|_1^2
\quad\text{for every }u\in\R^d\text{ and every retained prefix }n\ge n_0.
\end{equation}
The tolerance is essential when $n<d$. The positive regularization floor in \eqref{eq:lasso-calibration} controls corruption even if clean noise is zero. It depends on the design calibration, not on a candidate corruption bound.

Under this calibration, Proposition~\ref{prop:lasso-prediction-main} gives simultaneous prediction on every unit-ball context. Its bound holds for all $C_n\ge0$, with no corruption-dependent sample threshold, and permits an objective gap of $s\rho_n^2/(8\kappa)$. Exact optimization replaces the coefficient $7$ by $6$.

Appendix~\ref{app:robust-sparse-estimators} proves the sequential noise bound, enlarged-cone argument, and large-corruption case. Appendix~\ref{app:concrete-radii} gives the resulting reward and resource radii. All candidates share the retained data, regularization, numerical-precision rule, and fitted coefficients.

\xhdr{Computation and information.}
Appendix~\ref{app:tractable-options} gives a constrained proximal solver with a computable optimality-gap certificate for exactly \eqref{eq:lasso-estimator}. Calibration is a separate requirement: the learner uses model-supplied bounds justified in Appendix~\ref{app:concrete-design-conditions}, not an oracle for an unknown empirical RE constant. The on-policy guarantee remains conditional on supplied trajectory bounds. Convex fitting alone does not certify informative confidence for an arbitrary design. At most $(m+1)T$ changed fits are needed; the grid adds candidate scoring, not candidate-specific fits.

\subsection{Population and On-Policy Design Conditions}
\label{app:concrete-design-conditions}

Assumption~\ref{ass:sparse-context} and Condition~\ref{ass:on-policy-re} instantiate the main-text interface for the estimator in Appendix~\ref{app:concrete-estimator}; they are not imposed on another estimator that already meets that interface. Recall $h=1+Z$. For a calibrated pair $(\kappa,\upsilon)$, define the clean-width coefficient
\begin{equation}
\label{eq:lasso-clean-coefficient}
A(\kappa,\upsilon):=\frac{28\sqrt{2s\ell}}{\kappa}
\left[\sigma_r+Z\sigma_b+hR\sqrt{\kappa\upsilon}\right],\qquad H_T:=1+\log(1+T).
\end{equation}
This coefficient includes noise, sparsity, norm, design, and numerical-accuracy costs. Horizon summaries use the price stepsize specified after \eqref{eq:dual-control-term}.

Assumption~\ref{ass:sparse-context} is stated in Section~\ref{sec:estimator-contract}.

The condition imposes a lower bound in all population directions, rather than only in $2s$-sparse directions. Under context-independent forced selection, Lemma~\ref{lem:forced-re} proves \eqref{eq:empirical-rsc} simultaneously for retained prefixes with the conservative calibration
\begin{equation}
\label{eq:forced-rsc-calibration}
D_x=4096L_x^4/\kappa_x^2,\quad
\kappa_{\mathrm{FE}}=\kappa_x/2,\quad
\upsilon_{\mathrm{FE}}=16\kappa_{\mathrm{FE}}D_x,\quad
n_0=\lceil4096sD_x\ell\rceil.
\end{equation}
Write $A_{\mathrm{FE}}=A(\kappa_{\mathrm{FE}},\upsilon_{\mathrm{FE}})$. These formulas do not require $n\ge d$, but their constants can be conservative. Since $\|x\|_2\le1$, necessarily $\kappa_x\le1/d$; dimension dependence through curvature must not be suppressed. Appendix~\ref{app:primitive-coverage-example} gives a bounded Rademacher-feature instance with $D_x$ independent of $d$ and explains its compatible nonnegative feedback model.

\xhdr{On-policy coverage.}
Retain all non-null observations and write $N_a(t)=\sum_{\tau<t}\1\{a_\tau=a\}$. Let $\mathcal E_{\mathrm{RSC}}(\kappa_{\mathrm{op}},\upsilon_{\mathrm{op}},n_0)$ denote \eqref{eq:empirical-rsc} for every action-specific realized design with $N_a(t)\ge n_0$, using $\omega_n=\upsilon_{\mathrm{op}}\ell/n$. Condition~\ref{ass:on-policy-re} in Section~\ref{sec:estimator-contract} states the required probability bound for this event.

Set $A_{\mathrm{op}}=A(\kappa_{\mathrm{op}},\upsilon_{\mathrm{op}})$. Population calibration does not imply this condition for adaptive actions: selection, budget stopping, and corruption can change each action's realized design. A post-run diagnostic is not a learner-supplied lower bound, and a selected-support eigenvalue does not certify \eqref{eq:empirical-rsc}. Context-independent uniform actions with sufficient budget give a nonempty illustrative trajectory class, not a guarantee that our on-policy algorithm behaves that way (Appendix~\ref{app:primitive-coverage-example}). Candidate-recommendation coverage remains the separate Condition~\ref{ass:candidate-coverage} for Shared-Grid OP.

\xhdr{Exploration calibration.}
For the concrete forced-exploration implementation, choose a constant probability
\begin{equation}
\label{eq:lasso-exploration-choice}
\epsilon=\min\left\{\frac12,\max\left\{
\left(\frac{A_{\mathrm{FE}}^2 K}{h^2T}\right)^{1/3},
\sqrt{\frac{K(n_0+\ell)}{T}}\right\}\right\}.
\end{equation}
This balances clean statistical, burn-in, and exploration costs using supplied calibration only; it is independent of the corruption bound and is also used by Shared-Grid. The full bounds for arbitrary constant $0<\epsilon\le1/2$ appear in Sections~\ref{app:known-details} and~\ref{app:unknown-details}.

\subsection{Channel Radii and Score Confidence}
\label{app:concrete-radii}

All point estimates use \eqref{eq:lasso-estimator} and \eqref{eq:lasso-calibration}. At zero samples the estimate is zero and both channel radii are one. For a supplied, justified calibration and $n=N_a^{\mathrm{est}}(t)\ge n_0$, the concrete radii are
\begin{align}
\label{eq:concrete-reward-radius}
\beta_{a,t}^{r,\Gamma}&=\min\{1,7\rho_n^r\sqrt{s}/\kappa+4\Gamma/(\kappa n)\},\\
\label{eq:concrete-consumption-radius}
\beta_{a,i,t}^{b,\Gamma}&=\min\{1,7\rho_n^b\sqrt{s}/\kappa+4\Gamma/(Z\kappa n)\}.
\end{align}
Otherwise they are one. Null-action estimates and radii are zero. These prescriptions are defined even when a calibration event fails; confidence is asserted on that event, not by a data-dependent claim that an uncomputed RSC constant is positive.

Let $\beta^{r,0}$ and $\beta^{b,0}$ be the same formulas with $\Gamma=0$. The clean combined width and its corruption envelope obey
\begin{equation}
\label{eq:radius-decomposition}
\begin{aligned}
\beta_{a,t}^{Z,\Gamma}&\le f(n)+\frac{8\Gamma}{\kappa(n\vee1)}\1\{n\ge n_0\},\\
f(n)&:=\begin{cases}h,&n<n_0,\\ \min\{h,A(\kappa,\upsilon)/\sqrt n\},&n\ge n_0.\end{cases}
\end{aligned}
\end{equation}
Here $f$ is nonincreasing, $h=1+Z$, and $A$ is defined in \eqref{eq:lasso-clean-coefficient}. The bound follows from $\min\{1,u+v\}\le\min\{1,u\}+v$ and holds on any realized history, independently of confidence validity. Before burn-in, the entire radius is charged to the clean part. The nonnegative envelope $8\Gamma/(\kappa n)$ need not itself be clipped.

\begin{lemma}[From channel prediction to score confidence]
\label{lem:regression-to-score}
On the noise-and-design event of Proposition~\ref{prop:lasso-prediction-main}, \eqref{eq:general-pointwise-confidence} holds simultaneously for all valid $\Gamma$, and for every $\lambda\in\Lambda$,
\begin{equation}
\label{eq:appendix-score-confidence}
|\widehat F_{a,t}(\lambda)-F^0(a,x_t;\lambda)|\le\beta_{a,t}^{Z,\Gamma}.
\end{equation}
\end{lemma}
\begin{proof}
The clean means are in $[0,1]$ because the clean outcomes are in that interval and their noise is conditionally mean-zero. Projection onto $[0,1]$ cannot increase distance to a clean mean. On a retained history,
\[
\sum_{\tau\in\mathcal S_a^{\mathrm{est}}(t)}|c_\tau^r(a)|\le C_Z\le\Gamma,
\qquad
\sum_{\tau\in\mathcal S_a^{\mathrm{est}}(t)}|c_{\tau,i}^b(a)|\le C_Z/Z\le\Gamma/Z.
\]
These inequalities concern observed-minus-clean deviations after clipping, not the raw perturbations. Apply Proposition~\ref{prop:lasso-prediction-main} to each channel, including the trivial early radii, and use $\|\lambda\|_1\le1$. The same fits and noise event work for every valid $\Gamma$, including a path-dependent grid choice. No union bound over corruption values is needed.
\end{proof}

\subsection{Sequential Noise and the Corrupted Lasso Basic Inequality}
\label{app:robust-sparse-estimators}

\begin{lemma}[Coordinate noise control for retained prefixes]
\label{lem:lasso-noise-score}
For the pre-feedback retention rules in \eqref{eq:estimator-sample-sets}, with probability at least $1-\delta_{\mathrm{est}}$, every action, response channel, and obtained prefix of size $1\le n\le T$ satisfies
\begin{equation}
\label{eq:lasso-noise-event}
\|X^\top\eta/n\|_\infty\le\sigma\sqrt{2\ell/n}\le\rho_n/2.
\end{equation}
The claim remains valid for adaptively selected and budget-stopped OP histories.
\end{lemma}
\begin{proof}
Fix an action, channel, and coordinate $j$. Let $Q_t$ indicate that this action's observation is retained. Conditional on the past, current context, and the learner's fresh action/retention seed, $Q_t x_{t,j}$ is determined and the clean noise still has its stated conditional sub-Gaussian bound: the fresh seed is conditionally independent of the potential outcomes and corruption. The retention rule cannot inspect the current feedback. Thus, for every real $u$,
\[
\E\left[\exp\{uQ_t x_{t,j}\eta_t-u^2\sigma^2Q_t/2\}\mid\text{past, context, seed}\right]\le1.
\]
Iterated conditioning makes the product an exponential supermartingale in calendar time. Stop it at the $n$th retained observation or at $T$, whichever comes first. On the event that the $n$th observation occurs and its score exceeds $b$, its value is at least $\exp(ub-u^2\sigma^2n/2)$. Optional stopping and Markov's inequality, with $u=b/(n\sigma^2)$, give a tail at most $\exp[-b^2/(2n\sigma^2)]$ on this event. Apply both signs and take $b=\sigma\sqrt{2n\ell}$. Union over $dK(m+1)T$ choices gives failure at most $2dK(m+1)T e^{-\ell}\le\delta_{\mathrm{est}}$. When $\sigma=0$, the conditional sub-Gaussian property implies $\eta_t=0$ almost surely and the assertion is immediate. The supermartingale argument retains adaptive selection and stopping throughout.
\end{proof}

\begin{lemma}[Shared-regularization Lasso under response corruption]
\label{lem:lasso-corruption}
Consider $y=X\theta^\star+\eta+\zeta$, with row norms at most one, $\|\theta^\star\|_2\le R$, and $|J|\le s$ for $J=\supp(\theta^\star)$. Let $\widehat\theta$ be feasible for \eqref{eq:lasso-estimator} and have objective suboptimality at most $\varepsilon_{\mathrm{opt}}\ge0$. Suppose
\[
\begin{gathered}
Q(u):=\|Xu\|_2^2/n\ge\kappa\|u\|_2^2-\omega\|u\|_1^2\quad\text{for every }u,\\
\|X^\top\eta/n\|_\infty\le\rho/2,\qquad
0<\omega\le\kappa/(256s),\qquad \rho\ge4R\sqrt{\kappa\omega}.
\end{gathered}
\]
Writing $C_n=\|\zeta\|_1$, one has
\begin{equation}
\label{eq:lasso-general-gap-bound}
\|\widehat\theta-\theta^\star\|_2\le
\min\left\{2R,
\frac{6\rho\sqrt{s}}{\kappa}+\frac{4C_n}{\kappa n}
+2\sqrt{\frac{\varepsilon_{\mathrm{opt}}+4\omega\varepsilon_{\mathrm{opt}}^2/\rho^2}{\kappa}}\right\}.
\end{equation}
In particular, $\varepsilon_{\mathrm{opt}}\le s\rho^2/(8\kappa)$ gives \eqref{eq:lasso-prediction-main}. The error vector may be dense; the proof controls its $\ell_1$ norm through the penalized basic inequality.
\end{lemma}
\begin{proof}
Put $\Delta=\widehat\theta-\theta^\star$, $r=\|\Delta\|_2\le2R$, and $\gamma=C_n/n$. Feasibility of $\theta^\star$ and the objective-gap inequality yield the basic inequality
\begin{equation}
\label{eq:lasso-basic-inequality}
\frac12 Q(\Delta)
\le\frac{\eta^\top X\Delta}{n}+\frac{\zeta^\top X\Delta}{n}
+\rho(\|\theta^\star\|_1-\|\theta^\star+\Delta\|_1)+\varepsilon_{\mathrm{opt}}
\le\frac{3\rho}{2}\|\Delta_J\|_1-\frac{\rho}{2}\|\Delta_{J^c}\|_1+\gamma r+\varepsilon_{\mathrm{opt}}.
\end{equation}
The corruption bound uses $|\langle x_j,\Delta\rangle|\le r$, and the penalty difference uses sparsity of $\theta^\star$. Optimality compares the penalized objectives, so the penalty difference remains in the basic inequality. Since $Q\ge0$,
\begin{equation}
\label{eq:lasso-enlarged-cone}
\|\Delta_{J^c}\|_1\le3\|\Delta_J\|_1+\frac{2\gamma}{\rho}r+\frac{2\varepsilon_{\mathrm{opt}}}{\rho},\qquad
\|\Delta\|_1\le(4\sqrt{s}+2\gamma/\rho)r+2\varepsilon_{\mathrm{opt}}/\rho.
\end{equation}
Squaring with $(a+b)^2\le2a^2+2b^2$ gives
\[
\|\Delta\|_1^2\le(64s+16\gamma^2/\rho^2)r^2+8\varepsilon_{\mathrm{opt}}^2/\rho^2.
\]
First suppose $\gamma\le\rho\sqrt{\kappa/(64\omega)}$. The two tolerance coefficients are each at most $\kappa/4$, so
\[
Q(\Delta)\ge\kappa r^2/2-8\omega\varepsilon_{\mathrm{opt}}^2/\rho^2.
\]
Combining with \eqref{eq:lasso-basic-inequality} and dropping its negative term gives
\[
\kappa r^2/4\le(3\rho\sqrt{s}/2+\gamma)r+
\varepsilon_{\mathrm{opt}}+4\omega\varepsilon_{\mathrm{opt}}^2/\rho^2.
\]
For $ar^2\le br+c$ with $a>0,b,c\ge0$, $r\le b/a+\sqrt{c/a}$, proving the second bound in \eqref{eq:lasso-general-gap-bound} in this case.
If $\gamma>\rho\sqrt{\kappa/(64\omega)}$, then the regularization floor implies $\gamma>R\kappa/2$. Consequently $r\le2R<4\gamma/\kappa$, which proves the same bound without absorbing the enlarged cone. This case distinction uses the unknown true corruption only in analysis; neither the algorithm nor burn-in tests it.
Finally, when $\varepsilon_{\mathrm{opt}}\le s\rho^2/(8\kappa)$, one has $4\omega\varepsilon_{\mathrm{opt}}/\rho^2\le1$. The square-root term is at most $2\sqrt{2\varepsilon_{\mathrm{opt}}/\kappa}\le\rho\sqrt{s}/\kappa$. Cauchy--Schwarz gives simultaneous prediction on every new unit-ball context. With exact optimization that term is zero.
\end{proof}

\begin{proof}[Proof of Proposition~\ref{prop:lasso-prediction-main}]
Use Lemma~\ref{lem:lasso-noise-score}. At every calibrated prefix, \eqref{eq:lasso-calibration} implies $\omega_n\le\kappa/(256s)$ and $\rho_n\ge4R\sqrt{\kappa\omega_n}$, even when $\sigma=0$. Apply Lemma~\ref{lem:lasso-corruption} on the actual retained design. The lemma holds for arbitrary retained corruption vectors and therefore simultaneously for all admissible corruption sums on that history.
\end{proof}

\subsection{A Primitive RSC Calibration for Forced Samples}
\label{app:sparse-design}

\begin{lemma}[Bounded sub-Gaussian design calibration]
\label{lem:forced-re}
Under Assumptions~\ref{ass:base-model} and~\ref{ass:sparse-context}, the FE retained histories, including a safe context-independent cyclic initialization, satisfy \eqref{eq:empirical-rsc} with \eqref{eq:forced-rsc-calibration}, simultaneously for all obtained prefixes $n\ge n_0$, with failure probability at most $\delta_{\mathrm{cov}}/2$.
\end{lemma}
\begin{proof}
Write $\Sigma=\E[xx^\top]$. For a unit vector $v$, put $W=\langle v,x\rangle^2$. The exponential-moment assumption gives $\E W^p\le2p!L_x^{2p}$ and
$\E|W-\E W|^p\le2^{p+1}p!L_x^{2p}$ for $p\ge2$. Expanding the exponential power series yields Bernstein's bound
\begin{equation}
\label{eq:design-bernstein}
\sP\left(\left|\frac1n\sum_{j=1}^n W_j-\E W\right|>z,\ n\text{ obtained}\right)
\le2\exp\left(-\frac{nz^2}{32L_x^4+4L_x^2 z}\right).
\end{equation}
For independent samples this is the ordinary Chernoff calculation. For the actual retained samples the identical calculation uses the exponential supermartingale with predictable selection and stopping at the $n$th retained observation, as in Lemma~\ref{lem:lasso-noise-score}. FE selection and cyclic initialization are independent of the current context; safety is determined before that context. Hence the conditional moment bound is unchanged by selection. No conditioning on the final count or on survival is needed.

Jensen's inequality implies $\kappa_x\le L_x^2\log2\le L_x^2$. With $z=\kappa_x/8$, the exponent in \eqref{eq:design-bernstein} is at least $n/D_x$ for $D_x=4096L_x^4/\kappa_x^2$. For each support of size $q$, a $1/4$-net of its unit sphere has size at most $9^q$. Union over supports and the net implies
\[
|v^\top(\widehat\Sigma_n-\Sigma)v|\le(\kappa_x/4)\|v\|_2^2
\quad\text{for all }q\text{-sparse }v
\]
except with probability at most $2\exp[-n/D_x+q\log(9ed/q)]$. To see the factor two from the net explicitly, for the symmetric restriction $M$ to any support, its operator norm is at most the maximum net quadratic form plus $\frac12\|M\|_{\mathrm{op}}$.

For each $n\ge8D_x\ell$, let $k=\lfloor n/(8D_x\ell)\rfloor$ and $q=\min\{2k,d\}$. Since $\log(9ed/q)\le\ell$ and $q\le n/(4D_x\ell)$, the preceding failure bound is at most $2e^{-3n/(4D_x)}\le2e^{-6\ell}$. If $q=d$, the net result gives $\widehat\Sigma_n\succeq3\kappa_x I_d/4$ and the desired RSC follows. Otherwise partition any $u$ into disjoint blocks of $k$ coordinates, in decreasing absolute order. On the union of any two blocks the symmetric operator norm of $\widehat\Sigma_n-\Sigma$ is at most $\kappa_x/4$. Also
\[
\sum_j\|u_{J_j}\|_2\le\|u\|_2+\|u\|_1/\sqrt{k}.
\]
Indeed, the norm of each block after the first is at most the preceding block's $\ell_1$ norm divided by $\sqrt{k}$. Bilinearity and the two-block restriction therefore give
\[
|u^\top(\widehat\Sigma_n-\Sigma)u|
\le\frac{\kappa_x}{4}\left(\sum_j\|u_{J_j}\|_2\right)^2
\le\frac{\kappa_x}{2}\|u\|_2^2+\frac{\kappa_x}{2k}\|u\|_1^2.
\]
It follows that $u^\top\widehat\Sigma_n u\ge\kappa_{\mathrm{FE}}\|u\|_2^2-\kappa_{\mathrm{FE}}\|u\|_1^2/k$. The bound $k\ge n/(16D_x\ell)$ proves the asserted tolerance. Union over at most $KT$ obtained prefixes costs at most $2KT e^{-6\ell}\le\delta_{\mathrm{cov}}/2$. The threshold $n_0=\lceil4096sD_x\ell\rceil$ is above $8D_x\ell$ and satisfies the absorption requirement in \eqref{eq:lasso-calibration}.
\end{proof}

The full population lower bound and all-direction sub-Gaussian envelope are used in the block-extension step above; positivity only on $2s$-sparse directions would not justify that step. RSC with a tolerance is a standard framework for high-dimensional regularized estimation \citep{negahban2012unified}; the selected-sample concentration and corruption calculation above supply the specific steps used here.

\subsection{A Nonempty Bounded-Context Instance}
\label{app:primitive-coverage-example}
Let $x=(1,\epsilon_2,\ldots,\epsilon_d)/\sqrt d$, where the $\epsilon_j$ are independent Rademacher signs. Then $\|x\|_2=1$ and $\E xx^\top=I_d/d$. One may take $\kappa_x=1/d$ and $L_x=4/\sqrt d$. To verify the exponential envelope, write $\sqrt d\langle u,x\rangle=u_1+V$, where $V=\sum_{j\ge2}u_j\epsilon_j$ is sub-Gaussian with scale at most $\|u\|_2$. For $\alpha>0$ and an independent standard normal $g$, the identity $\E_{V,g} e^{\sqrt{2\alpha}gV}=\E_V e^{\alpha V^2}$ implies $\E e^{V^2/(8\|u\|_2^2)}\le(1-1/4)^{-1/2}$. Since $(u_1+V)^2\le2u_1^2+2V^2$, the required expectation is at most $e^{1/8}(1-1/4)^{-1/2}<2$. Thus $D_x$ is independent of $d$ and the sufficient sample threshold is of order $s\ell$, albeit with conservative numerical constants. It can be smaller than $d$ when $d/(s\ell)$ is sufficiently large; no claim of a small practical burn-in follows from these constants.

For compatible sparse means, take $s_0=s$ and choose $\theta^\star=\sqrt d(c e_1+\sum_{j=2}^{s}b_j e_j)$ with $0<c-\sum|b_j|\le c+\sum|b_j|<1$ and $d(c^2+\sum b_j^2)\le R^2$. Such choices exist for every $R>0$ by scaling $c,b_j$ down together. Add independent bounded mean-zero noise of amplitude less than the distance of the means to the endpoints of $[0,1]$, separately for each channel. This gives sparse parameters with bounded, potentially context-dependent reward and consumption means. In this normalized class curvature is $1/d$; the coefficient $A$ must retain that dependence.

For illustration, context-independent uniform non-null sampling and budgets $B_i\ge T+1$ prevent safety stopping and give the same design event on all action histories. A time-uniform Chernoff bound gives $N_a(t)\ge(t-1)/(2K)-c\log(KT/\delta)$. Moreover, $N_a^{\mathrm{rec},\Gamma}(t)\le t-1\le2K N_a(t)+cK\log(KT/\delta)$ for any recommendations. Thus this sampling regime provides a nonempty illustrative class with informative action-specific designs and recommendation coverage. It is not a guarantee that the proposed OP algorithms generate the same trajectories under corruption and budget stopping.

\subsection{A Reproducible Solver and Certified Numerical Error}
\label{app:tractable-options}

For $n\ge1$, the following solver minimizes the constrained objective \eqref{eq:lasso-estimator}. Inputs are the retained matrix $X$, clipped response vector $y\in[0,1]^n$, $R$, and the channel-specific $\rho_n$. There is no unpenalized intercept and no automatic centering, feature standardization, or response rescaling. A constant feature, when used, is part of $x$ and subject to both the penalty and norm bound. Any deliberate preprocessing must first be reflected in $R$, noise bounds, and design calibration.

Let $g(\theta)=\|y-X\theta\|_2^2/(2n)$ and $F(\theta)=g(\theta)+\rho_n\|\theta\|_1$. Its gradient is 1-Lipschitz because $\|X^\top X/n\|_{\mathrm{op}}\le1$. Start at $\theta^{(0)}=0$ and use step size one. For iteration $j$ set
\begin{equation}
\label{eq:lasso-prox-step}
 v=\theta^{(j)}-\nabla g(\theta^{(j)}),\qquad
 z=\operatorname{sign}(v)(|v|-\rho_n)_+,\qquad
 \theta^{(j+1)}=z\min\{1,R/\|z\|_2\},
\end{equation}
with the last expression interpreted as zero when $z=0$. The soft threshold and scaling together are the proximal map of $\rho_n\|\cdot\|_1+\iota_{\{\|\cdot\|_2\le R\}}$, where the last term is the convex indicator (zero on the ball, $+\infty$ outside). To check this, set $a=(\|z\|_2/R-1)_+$, and choose $w_j=\operatorname{sign}(z_j)$ if $z_j\ne0$, and $w_j=v_j/\rho_n$ otherwise. Then $w\in\partial\|\theta^{(j+1)}\|_1$, $v-\theta^{(j+1)}=\rho_n w+a\theta^{(j+1)}$, and $a(\|\theta^{(j+1)}\|_2-R)=0$, which are sufficient convex optimality conditions for the proximal problem.

At the new iterate form $v_F=\nabla g(\theta^{(j+1)})+\rho_n w$ and the computable certificate
\begin{equation}
\label{eq:lasso-gap-certificate}
G_j=\langle v_F,\theta^{(j+1)}\rangle+R\|v_F\|_2\ge F(\theta^{(j+1)})-\min_{\|\theta\|_2\le R}F(\theta)\ge0.
\end{equation}
The first inequality follows by minimizing the supporting affine minorant of $F$ over the ball. Stop at the first certified $G_j\le\varepsilon_{\mathrm{opt},n}:=s\rho_n^2/(8\kappa)$. This is a deterministic selection convention for approximate fits; it need not return the minimum-norm exact minimizer. The latter convention is well defined because the minimizer set is nonempty, compact, and convex and the squared Euclidean norm is strictly convex. Measurability of the exact convention follows, for example, by taking limits of the unique minimizers of $F(\theta)+\epsilon\|\theta\|_2^2$ on the ball as $\epsilon\downarrow0$. The finite algorithm uses measurable arithmetic and a first-passage stopping rule.

\begin{proposition}[Solver termination and precision accounting]
\label{prop:lasso-solver}
In exact arithmetic the stopping rule above terminates. A sufficient upper bound on the iterations is $1+\lceil16R^2/\varepsilon_{\mathrm{opt},n}^2\rceil$. Each iteration costs $O(nd)$ arithmetic operations. The returned fit satisfies Proposition~\ref{prop:lasso-prediction-main}, including its numerical-error contribution.
\end{proposition}
\begin{proof}
Proximal descent and the 1-Lipschitz gradient give $F(\theta^{(j)})-F(\theta^{(j+1)})\ge\frac12\|\theta^{(j+1)}-\theta^{(j)}\|_2^2$. Since $F(0)\le1/2$ and $F\ge0$, the sum of these squared steps is at most one. The proximal optimality equation implies
\[
\|v_F+a\theta^{(j+1)}\|_2
\le\|\theta^{(j)}-\theta^{(j+1)}\|_2+
\|\nabla g(\theta^{(j+1)})-\nabla g(\theta^{(j)})\|_2
\le2\|\theta^{(j+1)}-\theta^{(j)}\|_2.
\]
Using complementarity and $\|\theta^{(j+1)}\|_2\le R$ in \eqref{eq:lasso-gap-certificate} gives $G_j\le4R\|\theta^{(j+1)}-\theta^{(j)}\|_2$. If all of the first $J$ certificates exceeded $\varepsilon_{\mathrm{opt},n}$, the squared-step sum would exceed $J\varepsilon_{\mathrm{opt},n}^2/(16R^2)$. This is impossible for the stated $J$. The certificate implies the objective-gap condition in Lemma~\ref{lem:lasso-corruption}.
\end{proof}

The iteration bound is conservative; it establishes certified termination rather than practical speed. For floating-point implementation, verify feasibility and evaluate an upper bound on \eqref{eq:lasso-gap-certificate} using directed rounding or explicitly bounded arithmetic residuals. An approximately valid subgradient with a certified Euclidean discrepancy $e$ adds at most $2Re$ to the certificate. Accept an iterate only when a certified upper bound on its Euclidean norm is at most $R$. A fixed iteration limit or small successive-step difference without these error bounds is not the certificate used in our guarantee. If the tolerance is not met, use trivial radii until a certified fit is available; the concrete rates assume the stated certificate is obtained.

For a different certified objective gap $\varepsilon_{\mathrm{opt},a,t}$, the general additional prediction term is
\[
E_{a,t}=2\sqrt{\bigl(\varepsilon_{\mathrm{opt},a,t}+4\omega_n\varepsilon_{\mathrm{opt},a,t}^2/\rho_n^2\bigr)/\kappa}.
\]
Apply this separately by channel and add twice the recommendation-weighted sum of $E^r+Z\max_iE_i^b$ to the cumulative width interface. Under the prescribed tolerance these sums are already included in the coefficient $7$ and hence in $A$; there is no omitted fixed numerical error accumulated over $T$ rounds.

\xhdr{Full online cost and calibration.}
For a fixed channel, all corruption candidates use the same $X,y,\rho_n$ and solver stopping criterion. Only an action receiving a newly retained observation requires a new fit; cached fits for other actions retain their certified objective gaps for the unchanged objectives, since $\rho_n$ depends on $n$, not calendar time. There are at most $(m+1)T$ changed fits. A conservative total arithmetic cost is
\[
O\left((m+1)d\sum_{a=1}^K\sum_{n=1}^{N_a^{\mathrm{est}}(T+1)}
n\left[1+\frac{R^2}{\varepsilon_{\mathrm{opt},n}^2}\right]
+T\{K(m+1)d+K|\mathcal G|+m\}\right),
\]
using the largest channel-specific iteration bound in the first term. Computing all fitted scores costs $O(K(m+1)d)$ per round; scoring the grid costs $O(K|\mathcal G|)$, and the master and price updates add $O(|\mathcal G|+K+m)$. Storing retained raw samples costs $O(T(d+m))$. No delayed refitting or unanalysed stale estimator is used. In FE, $\kappa_x,L_x$ and the model bounds are supplied and \eqref{eq:forced-rsc-calibration} is computed directly. In OP, the supplied pair must be justified for the actual trajectory by Condition~\ref{ass:on-policy-re}; the method does not compute the smallest RE constant. Without justified bounds, trivial radii remain valid but these informative-width rates do not follow.

\subsection{Counting and Summing Statistical Widths}
\label{app:forced-exploration}
\label{app:covariate-diversity}

\begin{lemma}[Sequential scalar charging]
\label{lem:sequential-charging}
For $z_t\in[0,1]$, $Z_t=\sum_{u<t}z_u$, and $S=\sum_tz_t$, one has
\begin{equation}
\label{eq:sequential-charging}
\sum_t\frac{z_t}{Z_t\vee1}\le2[1+\log(1+S)],\qquad
\sum_t\frac{z_t}{\sqrt{Z_t\vee1}}\le4\sqrt S.
\end{equation}
If nonnegative integer counts $A_n$ have cumulative sums $\sum_{j=0}^n A_j\le b+\chi(n+1)$, then every nonincreasing nonnegative sequence $f$ satisfies
\begin{equation}
\label{eq:recommendation-summation}
\sum_{n=0}^N A_n f(n)\le b f(0)+\chi\sum_{n=0}^N f(n).
\end{equation}
\end{lemma}
\begin{proof}
For the first sum, the terms with $Z_t<1$ have total mass less than two and denominator one. For the remaining terms $Z_t\ge1$ and $z_t/Z_t\le1$, so $z_t/Z_t\le2\log(1+z_t/Z_t)$. These logarithms telescope from a starting cumulative mass at least one, giving a sum at most $2\log(1+S)$; the asserted bound follows. For the square-root sum,
\[
\sqrt{1+v+z}-\sqrt{1+v}=\frac{z}{\sqrt{1+v+z}+\sqrt{1+v}},
\]
whose denominator is at most $4\sqrt{v\vee1}$. Telescoping gives at most $4(\sqrt{1+S}-1)\le4\sqrt S$. Finally apply discrete summation by parts to the cumulative sums of $A_n$; all differences $f(n)-f(n+1)$ are nonnegative. Substitution of $b+\chi(n+1)$ gives \eqref{eq:recommendation-summation}.
\end{proof}

\begin{lemma}[Forced counts and cumulative widths]
\label{lem:comp-width-forced}
\label{lem:sparse-event-forced}
Let $0<\epsilon\le1/2$ be constant. On an event of failure at most $\delta_{\mathrm{cov}}/2$, every active round with $\epsilon(t-1)\ge8K\ell$ satisfies
\begin{equation}
\label{eq:forced-count-lower-bound}
N_a^{\mathrm{FE}}(t)\ge\epsilon(t-1)/(2K)\quad(a\in\Aset).
\end{equation}
Together with Lemma~\ref{lem:forced-re}, this gives the estimator/coverage event of failure at most $\delta_{\mathrm{est}}+\delta_{\mathrm{cov}}$. For an arbitrary sequence of recommendations, the FE cumulative interface holds with a universal constant $C$ and
\begin{equation}
\label{eq:forced-width-constant}
\Rest^0\le C\left[\frac{hK(n_0+\ell)}{\epsilon}+A_{\mathrm{FE}}\sqrt{\frac{KT}{\epsilon}}\right],\qquad
\psi_{\mathrm{est}}(\Gamma)\le\frac{CK\Gamma H_T}{\kappa_{\mathrm{FE}}\epsilon}.
\end{equation}
For shared FE after initialization the same bounds hold over subsequent rounds, with the separate initialization cost charged in the allocation theorem.
\end{lemma}
\begin{proof}
Generate the independent forced coins and uniform draws for all calendar times, including unused draws after stopping. On the event that round $t$ is active, all earlier rounds were active, so the actual pre-$t$ forced count equals this latent count. A multiplicative Chernoff bound at mean $\epsilon(t-1)/K$ gives failure at most $\exp[-\epsilon(t-1)/(8K)]$; union over $a,t$ and the condition $\epsilon(t-1)\ge8K\ell$ prove the first assertion. No distributional claim is made conditional on being active.

All actions have both the count bound and $n\ge n_0$ by the deterministic calendar threshold
\begin{equation}
\label{eq:coverage-time}
t_{\mathrm{cov}}=\min\left\{T+1,1+\left\lceil\frac{2K(n_0+4\ell)}{\epsilon}\right\rceil\right\}.
\end{equation}
Before this time charge at most $2h$ per active round. Afterwards \eqref{eq:radius-decomposition} and \eqref{eq:forced-count-lower-bound} bound the recommended action's statistical width by $A_{\mathrm{FE}}\sqrt{2K/[\epsilon(t-1)]}$ and its corruption envelope by $16K\Gamma/[\kappa_{\mathrm{FE}}\epsilon(t-1)]$. The interface contains this recommendation width with coefficient two. Summing $t^{-1/2}$ and $t^{-1}$ gives the conservative bound \eqref{eq:forced-width-constant}. If the threshold exceeds $T$, the displayed burn-in term already bounds the whole horizon. For shared FE use elapsed post-initialization time in the count argument. Initialization samples can only increase that count, and Lemma~\ref{lem:forced-re} applies to the combined retained prefix; no comparison of two different fitted estimators is used.
\end{proof}

\begin{lemma}[On-policy played-action widths]
\label{lem:learner-width-adaptive}
\label{lem:sparse-event-adaptive}
For the concrete on-policy module with the stated calibration, every $\Gamma\ge0$ satisfies
\[
\sum_{t=1}^T I_t\beta_{a_t,t}^{Z,\Gamma}
\le C\left\{hKn_0+A_{\mathrm{op}}\sqrt{KT}
              +\frac{K\Gamma H_T}{\kappa_{\mathrm{op}}}\right\}.
\]
\end{lemma}
\begin{proof}
Null actions have zero radius. For a fixed non-null action, its successive plays have pre-round retained counts $0,1,\ldots,N_a(T+1)-1$. Apply the envelope \eqref{eq:radius-decomposition}. The first $n_0$ counts cost at most $hn_0$. Subsequently, $\sum_{n=1}^{N}n^{-1/2}\le2\sqrt N$ and $\sum_{n=1}^{N}n^{-1}\le H_T$. Summing over actions and using $\sum_a\sqrt{N_a(T+1)}\le\sqrt{KT}$ proves the claim. The RSC and noise events justify the confidence radii; no comparison with an LP-policy sample count is needed.
\end{proof}

\begin{proposition}[Known and shared on-policy cumulative interfaces]
\label{prop:lasso-widths-on-policy}
For known-corruption OP, the recommendation equals the played action, and the cumulative interface holds with
\begin{equation}
\label{eq:known-op-widths}
\Rest^0\le C\{hKn_0+A_{\mathrm{op}}\sqrt{KT}\},
\qquad
\psi_{\mathrm{est}}(\Gamma)\le\frac{CK\Gamma H_T}{\kappa_{\mathrm{op}}}.
\end{equation}
For shared OP, on the coverage event of Condition~\ref{ass:candidate-coverage}, uniformly over candidates the post-initialization cumulative interface holds with
\begin{equation}
\label{eq:shared-op-widths}
\begin{aligned}
\Rshared^0&\le C\{hK(n_{\mathrm{cand}}+\chi n_0)
                   +\chi A_{\mathrm{op}}\sqrt{KT}\},\\
\psi_{\mathrm{est}}(\Gamma)&\le
 \frac{CK\Gamma}{\kappa_{\mathrm{op}}}
       \{\chi H_T+n_{\mathrm{cand}}\}.
\end{aligned}
\end{equation}
The noise and retained-design events additionally ensure that these radii provide the pointwise confidence interface.
\end{proposition}
\begin{proof}
The known case follows from twice Lemma~\ref{lem:learner-width-adaptive}, with constant factors absorbed into $C$.

For the shared case fix $a,\Gamma$, and let $A_n$ count active post-initialization recommendations of $a$ made when its pre-round realized count is $n$. At the time immediately after the last such recommendation with count at most $n$, Condition~\ref{ass:candidate-coverage} gives
\[
\sum_{j=0}^{n}A_j\le n_{\mathrm{cand}}+\chi(n+1),
\]
because that recommendation round can increase the realized count by at most one. Apply \eqref{eq:recommendation-summation} to the nonincreasing clean-width envelope $f$. Its initial value is at most $h$ and
\[
\sum_{n=0}^{N}f(n)\le h(n_0+1)+2A_{\mathrm{op}}\sqrt N.
\]
Sum through the final, possibly unplayed count level for each action, then use $\sum_a\sqrt{N_a(T+1)}\le\sqrt{KT}$ and $n_0\ge1$. This yields the stated clean-width bound after absorbing universal constants.

For the corruption part, dominate $\1\{n\ge n_0\}/n$ by the nonincreasing sequence $1/(n\vee1)$, whose partial sum is at most $2H_T$. The same summation inequality gives a charge of order $n_{\mathrm{cand}}+\chi H_T$ per action. Multiplying by the corruption coefficient in \eqref{eq:radius-decomposition}, summing over actions, and absorbing universal constants proves the second line. The count event holds simultaneously for all candidates, so the conclusion is uniform on the shared history.
\end{proof}

\subsection{Failure Probabilities and Concrete Substitution}
\label{app:failure-bookkeeping}
\label{app:intermediate-clean-widths}
For known FE combine the coordinate-noise event, the forced-design event, and the forced-count event; their failure is at most $\delta_{\mathrm{est}}+\delta_{\mathrm{cov}}$. For known OP combine noise with Condition~\ref{ass:on-policy-re}. For shared OP add the candidate event. None of these events is obtained by conditioning the clean martingale increments on successful estimation. Instead, bound the score gap pathwise on the good event and by $h$ on its complement. This produces the explicit $h\delta_{\mathrm{tot}}T$ expected-regret contribution. Mean-zero clean-feedback terms are removed using the original filtration before this event split, as in Appendix~\ref{app:known-analysis}.

Equations~\eqref{eq:forced-width-constant}, \eqref{eq:known-op-widths}, and \eqref{eq:shared-op-widths} provide every input to the modular allocation theorems. They include the chosen solver's numerical error through $A$, retain the actual curvature and tolerance dependence, and have no corruption-dependent burn-in. If $\kappa$ is small or the calibrated threshold exceeds available data, the bound can be larger than $T$; the unconditional range bound $\Reg(T)\le T$ still applies. Neither convexity nor an unsuccessful calibration implies sublinear regret.

\section{Proof of the Known-Corruption Guarantee}
\label{app:known-analysis}

This appendix expands the main-text proof overview into five blocks: score confidence and comparison with the population LP; cumulative widths and overrides; dual control and insertion of the benchmark; the relation between clean and observed resource use and stopping cancellation; and final assembly with concrete substitutions. We use $B^{\mathrm{op}}=B$, predictable estimates satisfying the main-text interface, clipped observed feedback, and the clean conditional-mean score $F^0(a,x_t;\lambda_t)$. The estimation objective enters only in the concrete substitutions at the end.

Let $I_t$ indicate that the pre-action safety condition holds at the start of round $t$. Active rounds form a prefix because the algorithm plays $\Null$ and freezes its state after the first failure. Let $\tau:=\sum_{t=1}^T I_t$. For the null action, set $\beta_{\Null,t}^{Z,\Gamma}=0$, $\widehat F_{\Null,t}(\lambda)=0$, and $F^0(\Null,x;\lambda)=0$.

\begin{table}[H]
\centering
\caption{Notation used in the known-corruption proof.}
\label{tab:known-proof-notation}
\small
\begin{tabularx}{\linewidth}{@{}lX@{}}
\toprule
Symbol & Meaning\\
\midrule
$I_t$ & Indicator that round $t$ is active under the pre-action safety rule.\\
$\tau$ & Number of active rounds, $\tau=\sum_t I_t$.\\
$R_D$ & Dual regret against the best fixed resource-price vector in $\Lambda$.\\
$\mathcal E_{\mathrm{conf}}$ & Common event supplying the pointwise confidence interface.\\
$\Rest^0$ & Clean cumulative recommendation-width bound supplied by the estimator interface.\\
$\psi_{\mathrm{est}}(\Gamma)$ & Additional cumulative width caused by retained response corruption.\\
\bottomrule
\end{tabularx}
\end{table}

\subsection{Score Confidence and Population-LP Comparison}
\label{app:known-score-confidence}

\begin{lemma}[Score confidence]
\label{lem:score-confidence}
On the event $\mathcal E_{\mathrm{conf}}$, suppose the reward and consumption estimates satisfy the componentwise interface \eqref{eq:general-pointwise-confidence}. Then every active round $t$, action $a\in\Asetzero$, and price vector $\lambda_t\in\Lambda$ satisfy
\begin{equation}
\label{eq:known-score-confidence}
\left|\widehat F_{a,t}(\lambda_t)-F^0(a,x_t;\lambda_t)\right|
\le \beta_{a,t}^{Z,\Gamma}.
\end{equation}
\end{lemma}

\begin{proof}
For $a=\Null$, both scores are zero. For $a\in\Aset$, apply the componentwise confidence interface directly. The reward error is at most $\beta_{a,t}^{r,\Gamma}$, and the priced consumption error is at most
\[
Z\sum_{i=1}^m\lambda_{t,i}\beta_{a,i,t}^{b,\Gamma}
\le Z\|\lambda_t\|_1\max_i\beta_{a,i,t}^{b,\Gamma}
\le Z\max_i\beta_{a,i,t}^{b,\Gamma}.
\]
Adding the two terms gives $\beta_{a,t}^{Z,\Gamma}$, as stated.
\end{proof}

\subsubsection{Optimistic comparison with the population-LP policy}
\label{app:known-optimistic-comparison}

\begin{lemma}[Optimistic comparison]
\label{lem:optimistic-comparison}
On $\mathcal E_{\mathrm{conf}}$, every active round satisfies
\begin{equation}
\label{eq:optimistic-comparison}
F^0(\widehat a_t,x_t;\lambda_t)
\ge\sum_{a\in\Aset}y^\star(a\mid x_t)F^0(a,x_t;\lambda_t) -2\beta_{\widehat a_t,t}^{Z,\Gamma}.
\end{equation}
If $e_t$ is the forced-exploration indicator in Algorithm~\ref{alg:known}, the played action satisfies
\begin{equation}
\label{eq:optimistic-comparison-exploration}
F^0(a_t,x_t;\lambda_t)
\ge\sum_{a\in\Aset}y^\star(a\mid x_t)F^0(a,x_t;\lambda_t)
-2\beta_{\widehat a_t,t}^{Z,\Gamma}-(1+Z)e_t.
\end{equation}
\end{lemma}

\begin{proof}
Score confidence and maximization imply, for every $a\in\Asetzero$,
\[
F^0(a,x_t;\lambda_t)
\le U_{a,t}^{\Gamma}
\le U_{\widehat a_t,t}^{\Gamma}
\le F^0(\widehat a_t,x_t;\lambda_t)
+2\beta_{\widehat a_t,t}^{Z,\Gamma}.
\]
The last expression is nonnegative because it is at least $U_{\widehat a_t,t}^{\Gamma}\ge U_{\Null,t}^{\Gamma}=0$. Extend $y^\star(\cdot\mid x_t)$ to a distribution on $\Asetzero$ by assigning its remaining mass to $\Null$. Multiplying the action-wise inequality by this distribution and summing proves \eqref{eq:optimistic-comparison}, since the null clean score is zero.

A forced action can reduce the clean score by at most $1+Z$, because all clean scores lie in $[-Z,1]$; when $e_t=0$, the played action equals the recommendation. This proves \eqref{eq:optimistic-comparison-exploration} pathwise.
\end{proof}

\subsection{Cumulative Widths and Overrides}
\label{app:known-cumulative-widths}

Define the cumulative clean-score gap on active rounds
\begin{equation}
\label{eq:known-opt-error}
\mathcal E_{\mathrm{opt}}
:=
\sum_{t=1}^T I_t\left[
\sum_{a\in\Aset}y^\star(a\mid x_t)F^0(a,x_t;\lambda_t)
-F^0(a_t,x_t;\lambda_t)
\right].
\end{equation}

\begin{lemma}[Cumulative clean-score gap]
\label{lem:optimism-decomposition}
Under the cumulative estimator interface, with joint estimator and coverage failure probability at most $\delta_{\mathrm{tot}}$,
\begin{equation}
\label{eq:known-opt-error-bound}
\E[\mathcal E_{\mathrm{opt}}]
\le
\Rest^0
+\psi_{\mathrm{est}}(\Gamma)
+(1+Z)\sum_{t=1}^T\epsilon_t
+(1+Z)\delta_{\mathrm{tot}}T.
\end{equation}
Here $\Rest^0$ is the clean cumulative-width term, while $\psi_{\mathrm{est}}(\Gamma)$ is the extra width caused by corrupted retained samples.
\end{lemma}

\begin{proof}
Sum \eqref{eq:optimistic-comparison-exploration} over active rounds and apply the cumulative estimator interface \eqref{eq:cumulative-width-interface-main} directly. It bounds the width sum by $\Rest^0+\psi_{\mathrm{est}}(\Gamma)$; no estimator-specific support property is needed here. For the concrete module, \eqref{eq:radius-decomposition} charges pre-certification radii to clean burn-in. Lemma~\ref{lem:comp-width-forced} in FE and Proposition~\ref{prop:lasso-widths-on-policy} in OP provide the concrete interface bounds. Finally, $I_t$ is fixed before the fresh exploration coin is drawn, so
\[
\E\!\left[\sum_{t=1}^T I_t e_t\right]
\le\sum_{t=1}^T\epsilon_t.
\]
On the complement of the joint estimator--coverage event, each clean Lagrangian comparison is at most $1+Z$. Its contribution is at most $(1+Z)\delta_{\mathrm{tot}}T$. The clean-to-observed resource discrepancy enters later through Lemma~\ref{lem:known-resource-bridge}, separate from the estimator term.
\end{proof}

\subsection{Dual Control and Insertion of the LP Benchmark}
\label{app:known-dual}

On an active round, let
\[
g_t:=b_t(a_t)-\frac BT,
\qquad
\wt g_t:=(g_t,0)\in\R^{m+1}.
\]
Define
\begin{equation}
\label{eq:dual-regret-definition}
R_D
:=
\max_{\lambda\in\Lambda}\sum_{t=1}^T I_t\lambda^\top g_t
-\sum_{t=1}^T I_t\lambda_t^\top g_t.
\end{equation}

\begin{lemma}[Resource-price update bound]
\label{lem:dual-regret}
The multiplicative-weights update in Algorithm~\ref{alg:known} satisfies
\begin{equation}
\label{eq:dual-regret-bound}
R_D
\le
\frac{\log(m+1)}{\eta_q}
+\eta_q\left(1+\frac{B_{\max}}T\right)^2T.
\end{equation}
So $ZR_D\le\DT$ for the definition in \eqref{eq:dual-control-term}.
\end{lemma}

\begin{proof}
The first $m$ coordinates of $q_t$ equal $\lambda_t$, and the last coordinate is the unused mass $1-\|\lambda_t\|_1$. Every $\lambda\in\Lambda$ corresponds to $q=(\lambda,1-\|\lambda\|_1)$ in the $(m+1)$-simplex. Also,
\[
\|\wt g_t\|_\infty
\le 1+\frac{B_{\max}}T.
\]
Hoeffding's lemma for the finite distribution $q_t$ bounds its log-moment-generating function by its mean plus $\eta_q^2\|\widetilde g_t\|_\infty^2/2$. Summing the log-potential increments yields, for every such $q$ and any $\eta_q>0$,
\[
\sum_{t=1}^T I_t(q-q_t)^\top\wt g_t
\le
\frac{\mathrm{KL}(q\|q_1)}{\eta_q}
+\eta_q\sum_{t=1}^T I_t\|\wt g_t\|_\infty^2.
\]
The algorithm freezes $q_t$ after stopping, so this is exactly the update on the active prefix. Since $q_1$ is uniform, $\mathrm{KL}(q\|q_1)\le\log(m+1)$. Taking the maximum over $\lambda$ proves \eqref{eq:dual-regret-bound}.
\end{proof}

\subsubsection{Insertion of the population-LP benchmark}
\label{app:known-population-conversion}

\begin{lemma}[Population-LP conversion]
\label{lem:known-population-conversion}
The cumulative clean-score gap on active rounds satisfies
\begin{equation}
\label{eq:conversion-clean-consumption}
\begin{aligned}
\Reg(T)
\le{}&
\E[\mathcal E_{\mathrm{opt}}]
+\VUB\,\E\!\left[1-\frac\tau T\right]\\
&+Z\E\!\left[
\sum_{t=1}^T I_t\lambda_t^\top
\left(\frac BT-b^0(a_t,x_t)\right)
\right].
\end{aligned}
\end{equation}
\end{lemma}

\begin{proof}
Because active rounds form a prefix, $I_t$ is measurable before $x_t$ is drawn. Optimality of $y^\star$ gives
\begin{equation}
\label{eq:inactive-benchmark-value}
\VUB
=
\E\!\left[
\sum_{t=1}^T I_t\sum_{a\in\Aset}
y^\star(a\mid x_t)r^0(a,x_t)
\right]
+\VUB\,\E\!\left[1-\frac\tau T\right].
\end{equation}
Expanding \eqref{eq:known-opt-error}, using \eqref{eq:inactive-benchmark-value}, and then applying population-LP feasibility yield \eqref{eq:conversion-clean-consumption}. The key measurability point is that $I_t\lambda_t$ is predictable before $x_t$. For every resource coordinate,
\[
\E_x\!\left[\sum_a y^\star(a\mid x)b_i^0(a,x)\right]
\le \frac{B_i}{T}
\]
can be multiplied by $I_t\lambda_{t,i}$ and averaged without changing its direction. This step is where population-LP feasibility enters the proof.
\end{proof}

\subsection{Relating Resource Use and Controlling Stopping}
\label{app:known-resource-bridge}

The population LP uses clean conditional-mean consumption, while the algorithm updates prices and stops using observed consumption. This step connects the two resource processes.

\begin{lemma}[Relating clean and observed resource use]
\label{lem:known-resource-bridge}
The cumulative clean-score gap on active rounds satisfies
\begin{equation}
\label{eq:conversion-observed-consumption}
\begin{aligned}
\Reg(T)
\le{}&
\E[\mathcal E_{\mathrm{opt}}]
+\VUB\,\E\!\left[1-\frac\tau T\right]\\
&+Z\E\!\left[
\sum_{t=1}^T I_t\lambda_t^\top
\left(\frac BT-b_t(a_t)\right)
\right]
+\E[C_Z].
\end{aligned}
\end{equation}
\end{lemma}

\begin{proof}
Because $I_t\lambda_t$ is predictable before $x_t$ and $a_t$ is selected before the current feedback using randomness independent of the clean potential outcomes, the conditional mean-zero clean noise has zero expectation. Together with $b_t(a_t)=b_t^{\mathrm{cl}}(a_t)+c_t^b(a_t)$, this implies
\begin{align}
&Z\E\!\left[
\sum_{t=1}^T I_t\lambda_t^\top
\left(b_t(a_t)-b^0(a_t,x_t)\right)
\right]
\notag\\
&\qquad=
Z\E\!\left[
\sum_{t=1}^T I_t\lambda_t^\top c_t^b(a_t)
\right]
\le
Z\E\!\left[
\sum_{t=1}^T I_t\|c_t^b(a_t)\|_\infty
\right]
\le \E[C_Z].
\label{eq:consumption-corruption-transfer}
\end{align}
Equivalently, pathwise the corruption part is controlled by
\[
Z\sum_{t=1}^T I_t
\left|\lambda_t^\top
\bigl(b_t(a_t)-b_t^{\mathrm{cl}}(a_t)\bigr)\right|
\le C_Z.
\]
Substituting the expected inequality into \eqref{eq:conversion-clean-consumption} proves \eqref{eq:conversion-observed-consumption}. This produces one direct $C_Z$ resource-accounting term, separate from the estimator widths.
\end{proof}

\subsubsection{Stopping cancellation}
\label{app:known-stopping}

\begin{lemma}[Stopping cancellation and expected BwK conversion]
\label{lem:bwk-conversion}
Under Assumption~\ref{ass:Z},
\begin{equation}
\label{eq:bwk-conversion-bound}
\Reg(T)
\le
\E[\mathcal E_{\mathrm{opt}}]+Z\E[R_D]+\E[C_Z]+Z.
\end{equation}
\end{lemma}

\begin{proof}
Let
\[
G:=\sum_{t=1}^T I_t\left(b_t(a_t)-\frac BT\right).
\]
The definition of $R_D$ keeps the following maximum:
\begin{equation}
\label{eq:dual-with-stopping-certificate}
Z\sum_{t=1}^T I_t\lambda_t^\top
\left(\frac BT-b_t(a_t)\right)
\le
ZR_D-Z\max_{\lambda\in\Lambda}\lambda^\top G.
\end{equation}
The maximum selects the most depleted resource direction and supplies the charge used for the benchmark value after stopping.

If $\tau=T$, there is no inactive benchmark value and the nonnegative maximum may be dropped. If $\tau<T$, then the pre-action test first fails after $\tau$ active rounds. For some resource $i$,
\[
U_{\tau+1,i}=\sum_{t=1}^T I_t b_{t,i}(a_t)>B_i-1,
\]
and so
\begin{equation}
\label{eq:stopped-coordinate}
G_i
=U_{\tau+1,i}-\frac{\tau B_i}{T}
>B_i\left(1-\frac\tau T\right)-1.
\end{equation}
Since both $0$ and the coordinate vector $e_i$ belong to $\Lambda$,
\[
\max_{\lambda\in\Lambda}\lambda^\top G
\ge\max\{0,G_i\}.
\]
Assumption~\ref{ass:Z} gives $\VUB\le ZB_{\min}\le ZB_i$. If $G_i\ge0$, then \eqref{eq:stopped-coordinate} implies
\[
\VUB\left(1-\frac\tau T\right)
-Z\max_{\lambda\in\Lambda}\lambda^\top G
<Z.
\]
If $G_i<0$, the same display follows because \eqref{eq:stopped-coordinate} implies $B_i(1-\tau/T)<1$, while the maximum is nonnegative. In both cases,
\begin{equation}
\label{eq:stopping-certificate}
\VUB\left(1-\frac\tau T\right)
-Z\max_{\lambda\in\Lambda}\lambda^\top G
\le Z.
\end{equation}

Taking expectations in the two pathwise inequalities and combining \eqref{eq:conversion-observed-consumption}, \eqref{eq:dual-with-stopping-certificate}, and \eqref{eq:stopping-certificate} proves the stopping-aware conversion. The remaining $Z$ is the stopping boundary. The pre-action rule already gives exact observed budget feasibility.
\end{proof}

\subsection{Assembly and Concrete Substitutions}
\label{app:proof-known}

\begin{proof}[Proof of Theorem~\ref{thm:known-general}]
Lemma~\ref{lem:optimism-decomposition} gives
\[
\E[\mathcal E_{\mathrm{opt}}]
\le
\Rest^0
+\psi_{\mathrm{est}}(\Gamma)
+(1+Z)\sum_{t=1}^T\epsilon_t
+(1+Z)\delta_{\mathrm{tot}}T.
\]
Lemma~\ref{lem:dual-regret} gives $Z\E[R_D]\le\DT$, and Lemma~\ref{lem:bwk-conversion} gives the direct resource transfer and the stopping boundary. Thus, the assembled bound is
\begin{equation}
\label{eq:known-assembly-before-gamma}
\boxed{\begin{aligned}
\Reg(T)\le{}&O\!\left(\Rest^0+(1+Z)\sum_{t=1}^T\epsilon_t+\DT+\E[C_Z]+\psi_{\mathrm{est}}(\Gamma)+Z\right)\\
&+(1+Z)\delta_{\mathrm{tot}}T.
\end{aligned}}
\end{equation}
Finally, the supplied bound is pathwise, $C_Z\le\Gamma$, and therefore $\E[C_Z]\le\Gamma$. Replacing the explicit $\E[C_Z]$ in \eqref{eq:known-assembly-before-gamma} by $\Gamma$ proves \eqref{eq:known-modular-bound}. The final display keeps the direct transfer $\Gamma$ separate from the estimator term $\psi_{\mathrm{est}}(\Gamma)$.

Lemma~\ref{lem:comp-width-forced} gives \eqref{eq:forced-width-constant}. Substituting it into the assembled bound proves \eqref{eq:known-fe-unoptimized-rate} and Corollary~\ref{cor:known-forced-concrete}; the term $h\epsilon T$ is already present in that bound. The exploration balance in Appendix~\ref{app:exploration-balance}, including its range-bound argument for the capped case, then gives \eqref{eq:known-fe-lasso-rate} and Corollary~\ref{cor:known-rates-main}(i).

For OP, the recommendation is the actual played action and the regression count is $N_a(t)$. Substituting \eqref{eq:known-op-widths} from Proposition~\ref{prop:lasso-widths-on-policy} proves \eqref{eq:known-op-lasso-rate} and Corollary~\ref{cor:known-rates-main}(ii). These substitutions retain the explicit $A$, $n_0$, curvature, and harmonic factors. Setting $\Gamma=0$ proves Corollary~\ref{cor:known-clean}.
\end{proof}

If the algorithm instead uses a strictly smaller operating budget, the same five proof blocks apply after replacing $B$ by $B^{\mathrm{op}}$, using a scale valid for the reduced-budget benchmark, and measuring regret first against $V^{\mathrm{UB}}(B^{\mathrm{op}})$. Returning to the original benchmark adds exactly $V^{\mathrm{UB}}(B)-V^{\mathrm{UB}}(B^{\mathrm{op}})$.

\subsection{Exploration Calibration for the Concrete FE Bounds}
\label{app:exploration-balance}
\begin{proof}[Exploration calibration]
For clarity about the FE choice, put $a=A_{\mathrm{FE}}\sqrt{KT}$ and $b=hK(n_0+\ell)$. The three clean terms are $b/\epsilon+a/\sqrt\epsilon+hT\epsilon$. If the cap in \eqref{eq:lasso-exploration-choice} is inactive, its two lower bounds imply
\[
\frac{b}{\epsilon}\le\sqrt{bhT},\qquad
\frac{a}{\sqrt\epsilon}\le a^{2/3}(hT)^{1/3},\qquad
hT\epsilon\le\sqrt{bhT}+a^{2/3}(hT)^{1/3}.
\]
This proves the $T^{2/3}$ term together with the square-root burn-in contribution in \eqref{eq:known-fe-lasso-rate}. For the capped case, write
\[
r_1=\left(\frac{A_{\mathrm{FE}}^2K}{h^2T}\right)^{1/3},\qquad
r_2=\sqrt{\frac{K(n_0+\ell)}{T}}.
\]
If $\max\{r_1,r_2\}\ge1/2$, then the sum of the two clean terms in the main bound is $hT(r_1+r_2)\ge hT/2\ge T/2$. The unconditional range bound $\Reg(T)\le T$ therefore proves the same main-text inequality after choosing $C\ge2$. Thus it is valid for all horizons, not only the uncapped regime; no informative rate is asserted when these terms are of order $T$. The argument also proves \eqref{eq:unknown-fe-lasso-rate}, since its other terms are nonnegative. The rule balances clean costs only and is implementable with either known or unknown corruption; the corruption term is evaluated at the chosen $\epsilon$.
\end{proof}

\section{Proof of the Unknown-Corruption Guarantee}
\label{app:unknown-analysis}

The proof compares recommendations computed from one common realized estimator history. It does not compare counterfactual histories that different candidates would have generated. Direct corruption enters twice when transferring the master comparison to clean scores, and once in the resource conversion. The cumulative confidence term is separate.

\subsection{One Common Confidence and Width Event}
\label{app:shared-valid-grid}

\begin{lemma}[Shared estimates and all valid grid points]
\label{lem:shared-lasso-prediction}
For the concrete module of Appendices~\ref{app:concrete-estimator}--\ref{app:concrete-design-conditions}, under the FE calibration or the OP trajectory condition, one event of failure at most $\delta_{\mathrm{est}}+\delta_{\mathrm{cov}}$ gives pointwise confidence for every valid $\Gamma\in\mathcal G$. On this event the FE grid-uniform width interface holds with \eqref{eq:forced-width-constant}. For OP, intersect also with Condition~\ref{ass:candidate-coverage}'s event; then \eqref{eq:shared-op-widths} supplies that interface, with total failure at most $\delta_{\mathrm{est}}+\delta_{\mathrm{cov}}+\delta_{\mathrm{cand}}$.
\end{lemma}
\begin{proof}
The estimator sample set \eqref{eq:unknown-estimator-sample-set} is exactly the initialization-augmented FE or all-realized OP set in \eqref{eq:estimator-sample-sets}. Lemma~\ref{lem:lasso-noise-score} covers its obtained prefixes, including stopping. In FE, Lemma~\ref{lem:forced-re} applies directly to the combined initialization and later forced samples, and Lemma~\ref{lem:comp-width-forced} supplies the count event. In OP, the retained design is precisely the design in Condition~\ref{ass:on-policy-re}.
For every response channel, the data, $\rho_n$, and certified objective-gap tolerance are common to all candidates. Lemma~\ref{lem:lasso-corruption} is deterministic for any corruption vector on the noise event. Therefore Lemma~\ref{lem:regression-to-score} covers all valid candidates simultaneously, without a candidate-dependent noise theorem or a candidate-dependent fit. The FE width bound holds for arbitrary recommendations. The OP width bound uses Proposition~\ref{prop:lasso-widths-on-policy} and the additional candidate event. Since $C_Z\le T(1+Z)$, the grid contains a valid point and the same event permits the post hoc substitution of $\Gamma^\star$.
\end{proof}

Thus, on the appropriate common event,
\begin{equation}
\label{eq:shared-valid-width}
2\sum_{t=t_{\mathrm{init}}+1}^T I_t
\beta_{\widehat a_t^{\Gamma^\star},t}^{Z,\Gamma^\star}
\le\Rshared^0+\psi_{\mathrm{est}}(\Gamma^\star).
\end{equation}
Initialization is charged separately, so $\Rshared^0$ is a post-initialization width bound. The parameters of the OP event concern the algorithm's actual shared trajectory, not the trajectory of a candidate run alone.

\subsection{Master Comparison with a Uniform-Advice Expert}
\label{app:master-regret}

Let $\mathcal Q_t$ include the full analysis history, current context, fitted coefficients, prices, and all recommendations before fresh round-$t$ draws. Let $\mathcal P_t$ also include the potential observed payoff vector. Non-anticipation fixes this vector before the independent action seed, so, conditional on $\mathcal P_t$ and a non-forced branch, $a_t\sim p_t$.

\begin{lemma}[Unbiased expert-payoff estimate]
\label{lem:unbiased-grid-score}
On each post-initialization, non-forced active round, every expert $j\in\mathcal E_{\mathrm M}$ obeys
\[
\E[\widehat y_{t,j}\mid\mathcal P_t,e_t=0]
=\sum_{a\in\Asetzero}\xi_{t,j}(a)Y_t(a).
\]
For a grid candidate this equals $Y_t(\widehat a_t^\Gamma)$; for the uniform expert it is the uniform average.
\end{lemma}
\begin{proof}
All $p_t(a)\ge p_{\min}>0$. Conditional expectation of
$\widehat Y_t(a)=\1\{a_t=a\}Y_t(a_t)/p_t(a_t)$ is $Y_t(a)$; summing against the advice gives the result. The denominator is the actual non-forced sampling probability, not the mixture probability including forced overrides.
\end{proof}

\begin{lemma}[Master score regret with overrides]
\label{lem:master-regret}
For the parameters \eqref{eq:exp4p-parameters}, let $L_{\mathrm M}=\log(3M/\delta_{\mathrm{master}})$. With failure at most $2\delta_{\mathrm{master}}/3$, simultaneously for all $\Gamma\in\mathcal G$,
\begin{equation}
\label{eq:master-score-regret}
\sum_{t=t_{\mathrm{init}}+1}^T I_t
[F_t^{\mathrm{obs}}(\widehat a_t^\Gamma;\lambda_t)-F_t^{\mathrm{obs}}(a_t;\lambda_t)]
\le8h\sqrt{TK_0 L_{\mathrm M}}+h\sum_{t=1}^T\epsilon_t.
\end{equation}
\end{lemma}
\begin{proof}
Suppose first $T\ge4K_0\log M$ and $L_{\mathrm M}\le K_0T$. Then the cap on $p_{\min}$ is inactive and \eqref{eq:exp4p-parameters} matches EXP4.P with confidence $\delta_{\mathrm{master}}/3$. The expert class includes an expert whose advice is uniform on every round, as required by \citet[Theorem~2 and Algorithm~1]{beygelzimer2011contextual}. The explicit uniform-advice expert is not replaced by action smoothing. Payoffs are in $[0,1]$, advice and potential payoffs are determined before the fresh master draw, and Lemma~\ref{lem:unbiased-grid-score} verifies importance weighting.

Reindex post-initialization, non-forced active rounds. Their number is at most $T$; pad this sequence after its final round with zero-payoff rounds, retaining the uniform expert and continuing the same master update. Padding adds zero payoff difference for every expert and does not alter earlier decisions. This is a length-$T$ non-anticipating expert-advice problem, even though the original subsequence and stopping time depend on the preceding history. The cited theorem gives a normalized regret at most $6\sqrt{TK_0 L_{\mathrm M}}$ with failure at most $\delta_{\mathrm{master}}/3$.

Forced rounds leave the master weights unchanged and cost at most $h$ each. Since $I_t$ is predictable and the forced coin is fresh,
\[
\sum_{t>t_{\mathrm{init}}}I_t e_t
\le\sum_{t=1}^T\epsilon_t+\sqrt{2T\log(3/\delta_{\mathrm{master}})}
\]
except with probability $\delta_{\mathrm{master}}/3$, by the exponential bound for bounded martingale differences. Multiplication by $h$ and $K_0\ge2$ give \eqref{eq:master-score-regret}. If either horizon condition fails, the total observed-score regret is always at most $hT$. When $T<4K_0\log M$ this is at most $2h\sqrt{TK_0L_{\mathrm M}}$; when $L_{\mathrm M}>K_0T$ it is at most $h\sqrt{TK_0L_{\mathrm M}}$. Thus the same bound holds in the capped or trivial regime without invoking the nontrivial-horizon theorem.
\end{proof}

\begin{lemma}[Uniform clean-score transfer]
\label{lem:master-clean-transfer}
With failure at most $\delta_{\mathrm{master}}/3$, simultaneously for every fixed $\Gamma\in\mathcal G$,
\begin{equation}
\label{eq:observed-clean-transfer}
\begin{aligned}
&\sum_{t>t_{\mathrm{init}}}I_t[F^0(\widehat a_t^\Gamma,x_t;\lambda_t)-F^0(a_t,x_t;\lambda_t)]\\
&\quad\le\sum_{t>t_{\mathrm{init}}}I_t[F_t^{\mathrm{obs}}(\widehat a_t^\Gamma;\lambda_t)-F_t^{\mathrm{obs}}(a_t;\lambda_t)]
+2C_Z+2h\sqrt{2T\log(3|\mathcal G|/\delta_{\mathrm{master}})}.
\end{aligned}
\end{equation}
\end{lemma}
\begin{proof}
For each non-null action,
\[
F_t^{\mathrm{obs}}(a;\lambda_t)-F^0(a,x_t;\lambda_t)
=\eta_t^r(a)-Z\lambda_t^\top\eta_t^b(a)+c_t^r(a)-Z\lambda_t^\top c_t^b(a).
\]
Each of the two action sequences contributes at most $C_Z$ in absolute corruption. For a fixed grid index, the recommendation is pre-feedback measurable and the played action uses an independent fresh seed. Hence the difference of their clean-noise terms is conditionally mean-zero, with absolute value at most $2h$. This is a martingale-difference argument under the original history and current context, not under conditioning on the potential observed payoff vector or a successful confidence event. Azuma's inequality and a union bound over the fixed grid give the displayed remainder. Only after this uniform event is established do we substitute the path-dependent $\Gamma^\star$.
\end{proof}

The preceding two lemmas together imply a clean candidate-to-played comparison bounded by
$\Reg_{\mathrm{mst}}(T)+h\sum_t\epsilon_t+2C_Z$, using the stated choice
$\Reg_{\mathrm{mst}}(T)=20h\sqrt{TK_0\log(3M/\delta_{\mathrm{master}})}$.

\subsection{Resource Conversion, Stopping, and Assembly}
\label{app:proof-unknown}

\begin{proof}[Proof of Theorem~\ref{thm:unknown-general}]
Define
\[
\mathcal E_{\mathrm{opt}}^{\mathrm{sh}}
=\sum_{t=1}^T I_t\left[\sum_a y^\star(a\mid x_t)F^0(a,x_t;\lambda_t)-F^0(a_t,x_t;\lambda_t)\right].
\]
Split this sum into the initialization prefix, the subsequent population-LP-to-$\Gamma^\star$ recommendation comparison, and the subsequent $\Gamma^\star$-to-played-action comparison. Initialization costs at most $ht_{\mathrm{init}}$. On the joint confidence and coverage event, Lemma~\ref{lem:optimistic-comparison} and \eqref{eq:shared-valid-width} bound the second part by $\Rshared^0+\psi_{\mathrm{est}}(\Gamma^\star)$. The master and clean-transfer lemmas bound the third by $\Reg_{\mathrm{mst}}(T)+h\sum_t\epsilon_t+2C_Z$.

For FE the joint event comprises noise, forced design/counts, and the master/transfer events. For OP it instead uses the realized-design event together with candidate-recommendation coverage. Their failure probabilities sum to the corresponding $\delta_{\mathrm{tot}}$ in \eqref{eq:unknown-delta-total}. The per-round clean-score gap is always at most $h$, including on the complement. All terms on the good-event upper bound are nonnegative, so taking expectations gives
\begin{equation}
\label{eq:unknown-opt-error-bound}
\begin{aligned}
\E[\mathcal E_{\mathrm{opt}}^{\mathrm{sh}}]\le{}&
\Rshared^0+\Reg_{\mathrm{mst}}(T)+ht_{\mathrm{init}}+h\sum_t\epsilon_t
+\E[\psi_{\mathrm{est}}(\Gamma^\star)+2C_Z]+h\delta_{\mathrm{tot}}T.
\end{aligned}
\end{equation}
The method uses exactly the known algorithm's actual observed-consumption price update and pre-context stopping test. Lemmas~\ref{lem:known-population-conversion},~\ref{lem:known-resource-bridge}, and~\ref{lem:bwk-conversion} therefore apply to this played sequence without any estimator change:
\begin{equation}
\label{eq:unknown-bwk-conversion}
\Reg(T)\le\E[\mathcal E_{\mathrm{opt}}^{\mathrm{sh}}]+Z\E[R_D]+\E[C_Z]+Z.
\end{equation}
Here $Z\E[R_D]\le\DT$ by Lemma~\ref{lem:dual-regret}. The additional $C_Z$ is the clean-to-observed resource transfer, separate from the two master score transfers. Thus their total is $3\E[C_Z]\le3\E[\Gamma^\star]$. This proves the expected-regret claim. Proposition~\ref{prop:feasibility-levels} gives the pathwise observed and clean resource claims, independently of all confidence and coverage events.
\end{proof}

\begin{proof}[Proofs of Corollaries~\ref{cor:unknown-rates-main} and~\ref{cor:unknown-forced-concrete}]
For FE substitute \eqref{eq:forced-width-constant} into \eqref{eq:unknown-opt-error-bound} and then \eqref{eq:unknown-bwk-conversion}. The combined initialization and forced retained history has the RSC calibration proved in Lemma~\ref{lem:forced-re}; subsequent forced counts provide a lower bound for its sample size. This is a bound on counts and radii, not an assumed ordering of estimators fitted on different samples. Add $ht_{\mathrm{init}}$, $h\epsilon T$, the master, dual, and stopping costs. This yields \eqref{eq:unknown-fe-unoptimized-rate}. Applying the exploration calculation in Appendix~\ref{app:exploration-balance} gives \eqref{eq:unknown-fe-lasso-rate}; its chosen $\epsilon$ is independent of the unknown pathwise corruption.
For OP use \eqref{eq:shared-op-widths}. Its proof keeps recommendation and retained counts separate and explicitly charges $h\chi Kn_0$ before informative recommendation confidence. Substitution, with $\epsilon_t=0$, yields \eqref{eq:unknown-op-lasso-rate}. Retain the full $h\delta_{\mathrm{tot}}T$ contribution in both cases. If $C_Z=0$, $\Gamma^\star=0$ removes only the corruption terms; it does not remove initialization, regularization, or the master comparison.
\end{proof}

\subsection{Proof of Recommendation Coverage}
\label{app:candidate-coverage-proof}
\begin{proof}[Proof of Lemma~\ref{lem:candidate-coverage-sufficient}]
Fix $(\Gamma,a,t)$ and, for post-initialization rounds $\tau<t$, let
\[
Q_\tau:=I_\tau\1\{\widehat a_\tau^\Gamma=a\},
\qquad
X_\tau:=Q_\tau\1\{a_\tau=a\}.
\]
Conditional on $\mathcal Q_\tau$, $\E[X_\tau\mid\mathcal Q_\tau]\ge q_{\mathrm{cov}}Q_\tau$.
For every deterministic $u>0$, a multiplicative martingale lower-tail bound gives
\[
\Pr\!\left(\sum_{\tau<t}X_\tau<\frac{q_{\mathrm{cov}}}{2}\sum_{\tau<t}Q_\tau,
\quad
\sum_{\tau<t}Q_\tau\ge u\right)
\le
\exp\!\left(-\frac{q_{\mathrm{cov}}u}{8}\right).
\]
Set
\[
u=
\frac{8}{q_{\mathrm{cov}}}
\log\!\frac{2|\mathcal G|KT}{\delta_{\mathrm{cand}}}.
\]
A union bound over candidates, actions, and prefixes implies that, simultaneously for all $(\Gamma,a,t)$, either
$N_a^{\mathrm{rec},\Gamma}(t)<u$, or
\[
\sum_{\tau<t}X_\tau
\ge
\frac{q_{\mathrm{cov}}}{2}N_a^{\mathrm{rec},\Gamma}(t).
\]
Since $\sum_{\tau<t}X_\tau\le N_a(t)$, in the latter case $N_a^{\mathrm{rec},\Gamma}(t)\le 2N_a(t)/q_{\mathrm{cov}}$. Combining the two cases proves \eqref{eq:candidate-coverage-sufficient}.
\end{proof}

\end{APPENDICES}
\end{document}